\documentclass{article}

\usepackage{microtype}
\usepackage{graphicx}
\usepackage{subfigure}
\usepackage{booktabs} 
 
\usepackage{enumitem}
\usepackage{wrapfig}
\usepackage[colorlinks,
linkcolor=red,
anchorcolor=blue,
citecolor=blue
]{hyperref}
\usepackage{smile}
\usepackage{fullpage}
\usepackage{enumitem}
\usepackage{algpseudocode}

\title{Provably Efficient Generative Adversarial Imitation Learning for Online and Offline Setting with Linear Function Approximation}

\author
{\normalsize
Zhihan Liu\thanks{University of Science and Technology of China; 
\texttt{liuzhihan0627@mail.ustc.edu.cn}}
\qquad Yufeng Zhang\thanks{Northwestern University; 
\texttt{yufengzhang2023@u.northwestern.edu}}
\qquad Zuyue Fu\thanks{Northwestern University; 
\texttt{zuyue.fu@u.northwestern.edu}}
\qquad Zhuoran Yang\thanks{Princeton University; 
\texttt{zhuoranyang.work@gmail.com}}
\qquad Zhaoran Wang\thanks{Northwestern University; 
\texttt{zhaoranwang@gmail.com}}}

\usepackage[english]{babel}
\usepackage{breakcites}

\def\cP{{\mathcal{P}}}
\def\cS{{\mathcal{S}}}
\def\cA{{\mathcal{A}}}
\def\cR{{\mathcal{R}}}
\def\cV{{\mathcal{V}}}
\def\cN{{\mathcal{N}}}
\def\D{{\mathbb D}}
\def\rE{{\mathrm {E}}}
\def\rA{{\mathrm {A}}}

\renewcommand{\tilde}{\widetilde}
\def\cO{{\mathcal{O}}}
\def\tO{{\tilde{\cO}}}
\def\given{\,|\,}
\def\biggiven{\,\big{|}\,}

\def\##1\#{\begin{align}#1\end{align}}
\usepackage{microtype}
\usepackage{mathrsfs}

\ifx\BlackBox\undefined
\newcommand{\BlackBox}{\rule{1.5ex}{1.5ex}}  
\fi

\ifx\theorem\undefined
\newtheorem{theorem}{Theorem}
\fi
\ifx\example\undefined
\newtheorem{example}[theorem]{Example}
\fi
\ifx\property\undefined
\newtheorem{property}{Property}
\fi
\ifx\lemma\undefined
\newtheorem{lemma}[theorem]{Lemma}
\fi
\ifx\proposition\undefined
\newtheorem{proposition}[theorem]{Proposition}
\fi
\ifx\remark\undefined
\newtheorem{remark}[theorem]{Remark}
\fi
\ifx\corollary\undefined
\newtheorem{corollary}[theorem]{Corollary}
\fi
\ifx\definition\undefined
\newtheorem{definition}[theorem]{Definition}
\fi
\ifx\conjecture\undefined
\newtheorem{conjecture}[theorem]{Conjecture}
\fi
\ifx\fact\undefined
\newtheorem{fact}[theorem]{Fact}
\fi
\ifx\claim\undefined
\newtheorem{claim}[theorem]{Claim}
\fi
\ifx\assumption\undefined
\newtheorem{assumption}[theorem]{Assumption}
\fi
\ifx\cond\undefined
\newtheorem{cond}[theorem]{Condition}
\fi

\begin{document}

\maketitle

\begin{abstract}
    In generative adversarial imitation learning (GAIL), the agent aims to learn a policy from an expert demonstration so that its performance cannot be discriminated from the expert policy on a certain predefined reward set. In this paper, we study GAIL in both online and offline settings with linear function approximation, where both the transition and reward function are linear in the feature maps. Besides the expert demonstration, in the online setting the agent can interact with the environment, while in the offline setting the agent only accesses an additional dataset collected by a prior. For online GAIL, we propose an optimistic generative adversarial policy optimization algorithm (OGAP) and prove that OGAP achieves  $\widetilde{\mathcal{O}}(H^2 d^{3/2}K^{1/2}+KH^{3/2}dN_1^{-1/2})$ regret. Here $N_1$ represents the number of trajectories of the expert demonstration, $d$ is the feature dimension, and $K$ is the number of episodes.
   For offline GAIL, we propose a pessimistic generative adversarial policy optimization algorithm (PGAP). For an arbitrary additional dataset, we obtain the optimality gap of PGAP, achieving the minimax lower bound in the utilization of the additional dataset. Assuming sufficient coverage on the additional dataset, we show that PGAP achieves $\widetilde{\mathcal{O}}(H^{2}dK^{-1/2} +H^2d^{3/2}N_2^{-1/2}+H^{3/2}dN_1^{-1/2} \ )$ optimality gap. Here $N_2$ represents the number of trajectories of the additional dataset with sufficient coverage.
    \end{abstract}
    

\section{Introduction}
In imitation learning (IL, \cite{IL-survey}) (a.k.a apprenticeship learning), the agent remains unknown of the reward, but can learn from an expert demonstration so that the agent learns a policy as good as the expert one.  To solve IL problem, there exist mostly three types of methods: behavior cloning (BC, \cite{BC-Robot-Learning-from-demonstration}), inverse reinforcement learning (IRL, \cite{2004-AL-using-IRL}), and online generative adversarial imitation learning (online GAIL). BC  regards IL as a supervised learning problem of predicting actions based on states.  While appealingly simple, BC suffers from compounding error caused by covariate shift \citep{BC-corvirate-shift-paper-1,BC-corvirate-shift-paper-2}. IRL explicitly solves the true reward function and then accordingly fully solves an RL subproblem at every iteration \citep{2004-AL-using-IRL,IRL-theory-Algorithms-for-IRL}. Though it has succeeded in tasks involving continuous spaces \citep{IRL-Continuous-Inverse-Optimal-Control,IRL-Continuous-Inverse-control-deep}, IRL lacks computational efficiency and the desired true reward function may not be unique.
To address these issues, online GAIL \citep{GAIL} solves IL through minimax optimization with alternating updates to learn a policy whose performance cannot be discriminated from the expert policy on a certain predefined reward set. The alternating updates in online GAIL mirror the training of generative adversarial networks \citep{GANs,WGAN}. Specifically, during every iteration, online GAIL first minimizes the discrepancy in expected cumulative reward between the expert policy and the learned policy and then maximizes such a discrepancy over a given reward function class in adversary.
Online GAIL achieves tremendous empirical success in a variety of fields, such as autonomous driving \citep{application-gail-autodriving}, human behavior modeling \citep{application-gail-human}, natural language processing \citep{application-gail-nlp}, and robotics control \citep{application-gail-control}.

Despite the state-of-art empirical performance of online GAIL, the agent requires a huge amount of interactions with the environment during the training. For some practical problems, it is inconvenient, costly, or risky to get expert data or labeled data, especially when collecting clinical data or developing autonomous driving. Meanwhile, it is available to get other sources of offline data, which may be originated from historical experiments, non-labeled data, and published datasets, etc. Naturally, we desire to utilize these offline data to alleviate the shortage of expert demonstration and aid the agent to mimic the expert policy.
To this end, besides online GAIL, we consider the offline generative adversarial imitation learning (offline GAIL) setting. In offline GAIL, we assume that the agent is accessible to an additional dataset besides the expert demonstration, without further interaction with the environment.  
Some related works \citep{zOffline-Learning-from-Demonstrations-and-Unlabeled-Experience,offlineGAILAPP} study this setting and provide methods from the aspect of the experiment. 
   
Furthermore, previous theoretical analyses on GAIL either focus on the tabular case \citep{Online-Apprenticeship-Learning}, where the state and action spaces are discrete, or relies on strong assumptions, including access to a well-explored dataset \citep{nerual-gail}, linear-quadratic regulators \citep{cai2019global}, or kernelized nonlinear regulators \citep{KNR}. Theoretical analysis for GAIL with linear function approximation either in online or offline settings still remains an open problem, which is crucial for the application of GAIL in the continuous or high dimensional state and action spaces. The cruxes of such an analysis involve: (i) Different from RL, both online GAIL and offline GAIL are minimax optimization problems with respect to the policy and reward function, especially with linear reward set. (ii) For offline GAIL, without assuming the well-exploredness of the additional dataset, the agent may be misled by distribution shift in the additional dataset and shares the suffering with offline RL \citep{pessimismRl-Jinchi,limitOffline}; (iii) For offline GAIL, we are incapable to update the reward function based on the trajectory of present policy.

Hence in this paper, we aim at tackling these issues and answering the following question:

\begin{center}
{\textit{Can we design provably efficient algorithms for online and offline GAIL \\ with linear function approximation in a unified view?}}
\end{center}

 To answer the above question, we present a unified framework and specialize it as
optimistic generative adversarial policy optimization (OGAP) for online GAIL and pessimistic generative adversarial policy optimization (PGAP) for offline GAIL with linear function approximation. 
This framework is motivated by the  alternating update process of GANs and involves two main stages: (i) policy update stage and (ii) reward update stage. (i) In the policy update stage, we apply mirror descent \citep{mirror-descent-1,mirror-descent-2} to update the policy and evaluate policy online optimistically for OGAP and offline pessimistically for PGAP. 
(ii) In the reward update stage, we first estimate the gradient of GAIL objective function with respect to the reward parameter through the collected trajectory induced by the present policy for OGAP. While for PGAP, we build the estimate through estimated action-value functions during the stage of policy update.
Then we use projected gradient ascent to update reward parameters via such an estimate of gradient. 

\paragraph{Contribution}
Particularly, we conclude our contributions in the following three aspects.
\begin{itemize}
\item For online GAIL with linear function approximation, we propose a new algorithm OGAP and prove that OGAP achieves $\tO(H^2 d^{3/2}K^{1/2}+KH^{3/2}dN_1^{-1/2})$ \footnote{Here $\tilde{\cO}(\cdot)$ hides the log terms and constant terms.} regret when applying linear function approximation, demonstrating that OGAP is provably efficient. Here $N_1$ represents the number of trajectories of the expert demonstration, $d$ is the feature dimension, and $K$ is the episode.
\item For offline GAIL with linear function approximation, we design a new algorithm PGAP and obtain the optimality gap of the output policy under the minimal assumption on the additional dataset. Then we decompose the optimality gap into three sources: optimization error, Monte Carlo (MC) estimation error, and intrinsic error. We prove that optimization error and  MC estimation error respectively scale to $\tO(K^{-1/2})$ and $\tO(N_1^{-1/2})$, while intrinsic error depends on how well the additional dataset $\D^\rA$ covers the expert policy and attains the minimax optimality in the utilization of the additional dataset.
\item In addition, we demonstrate that if we further assume that the additional dataset  $\mathbb{D}^\rA$ has sufficient coverage on the expert policy, we prove PGAP achieves $\tO(H^{2}dK^{-1/2} +H^2d^{3/2}N_2^{-1/2}+H^{3/2}dN_1^{-1/2} \ )$ optimality gap, thus PGAP has global convergence. Here $N_2$ represents the number of trajectories of the additional dataset with sufficient coverage. Furthermore, we discuss the effect of the additional offline dataset $\mathbb{D}^\rA$. In particular, facilitated with an additional dataset $\mathbb{D}^\rA$ with sufficient coverage, we decrease the dependency for horizon $H$ and feature dimension $d$ in the optimality gap.
\end{itemize}
\paragraph{Related Works}

Our work adds to the body of analysis on GAIL \citep{cai2019global,chen-GAIL,nerual-gail,ErrorBoundINIMITATIONLEARNING,Online-Apprenticeship-Learning,KNR}. 
 \cite{Online-Apprenticeship-Learning} study online GAIL and obtain $\tilde{\cO}(H^2 |\cS| |\cA|^{1/2} K^{1/2}+H^{3/2} |\cS|^{1/2}|\cA|^{1/2}KN_1^{-1/2} )$ regret in the tabular case with bounded reward functions but we apply linear function approximation on the transition kernels without assuming the state space or the action space is discrete and we adopt linear reward set. \cite{chen-GAIL} only study the convergence of offline GAIL to a stationary point instead of global convergence (optimality gap) as in this paper. \cite{ErrorBoundINIMITATIONLEARNING,nerual-gail} analyze the global convergence of GAIL with neural networks respectively in the tabular case and the continuous case but assume that a well-explored dataset is available (concentrability coefficients are uniformly
upper bounded), while our analysis need not such a strict and impractical assumption. \cite{cai2019global} study the global convergence of offline GAIL in the setting of linear-quadratic regulators, which is unnecessary for this paper. As an independent work, \cite{KNR} study offline GAIL with bounded reward functions in the continuous kernelized nolinear regulator (KNR,  \cite{KNRSetting}) and  Gaussian process (GP, \cite{GP}) setting. We point out that the KNR (resp. GP) setting is different from linear kernel MDP as analyzed in this paper and each one does not imply the other, which leads to the difference in model estimation and later analysis. In addition, we study the linear reward set instead of bounded reward set and the former case is difficult to handle with \citep{Online-Apprenticeship-Learning}. 

Our work is related to IRL, \cite{2004-AL-using-IRL,AL+IRL+Grandient,AL-linear-program,Game-Theoretic-Approach-AL2007} study the convergence of IRL in the tabular case, while they require to solve an RL subproblem every iteration, inefficiently.
Our work is also related to BC \citep{BC-corvirate-shift-paper-1,jiao-BC-Toward-the-Fundamental-Limits-of-IL,Provably-Breaking-the-Quadratic-Error-Compounding-Barrier-in-Imitation-Learning, MACONG-2021new-Offline+IL}. BC does not solve a minimax problem as GAIL, but directly mimics the expert policy extracted from the expert demonstration. \cite{Provably-Breaking-the-Quadratic-Error-Compounding-Barrier-in-Imitation-Learning} propose a BC method which achieves $\tilde{\cO}(|\mathcal{S}|H^2/N_1)$ suboptimality, attaining $\Omega(|\mathcal{S}|H^2/N_1)$ the lower bound of BC \citep{jiao-BC-Toward-the-Fundamental-Limits-of-IL}, when the transition model is unknown.  To best of our knowledge, present analysis of BC only focus on the tabular case and would fail in the continuous state and action space with horizon $H\ge 2$, since BC is considered as a classification problem and always faces unseen states in the continuous state space. 

Besides, our work is
related to the vast body of existing literature on online RL cooperated with optimism \citep{MABexploration,UCRL,explorationRL1,ProvablyefficientQlearning,LinearFunctionApproximation,RLinfeaturespace}, offline RL  \citep{offlineRL2,CQL,BCQ,minimaxoffpolicy,uniform-covering,pessimismRl-Jinchi}, policy optimization \citep{mirror-descent-1,mirror-descent-2,OPPO,policy_optimization}, adversarial MDP \citep{advMDPbanditfeed,advMDP2,advMDP3JIN}, and linear function approximation  \citep{minimaxoffpolicy,LinearLSforTD,LinearsampleoptimalparaQlearning,LinearFunctionApproximation,linear-approximation-example-1,linear-approximation-example-2,Gu-probility-measure} while they study minimization or maximization problem with known reward through value-based or policy-based method, instead of minimax problem with respect to policy and reward function as GAIL.

Our work is related to a line of study on pessimism. Specifically, the uncertainty quantification for estimated model in PGAP is motivated by the pessimism in offline RL \citep{jiangnanOffline,jiangnanOffline2,jiangnanOffline3,CQL,pessimismRl-Jinchi,PessFQI,PessOfflinPOLICY,PessOfflinPOLICY2,importanceofpess,MACONG-2021new-Offline+IL,sunwenOffline}. \cite{PessFQI} propose a pessimistic variant of fitted Q-learning algorithm \citep{uniform3,uniform4} achieving the optimal policy within
a restricted class of policies without assuming the dataset to be well-explored. \cite{pessimismRl-Jinchi} propose a provably efficient algorithm with the spirit of pessimism to solve offline RL with linear function approximation, under no coverage assumption on the dataset. \cite{jiangnanOffline2} propose a refined pessimistic estimate and obtain a tighter suboptimality in $d$ compared with \cite{pessimismRl-Jinchi}.
\cite{MACONG-2021new-Offline+IL} study the offline RL in the tabular case through lower confidence bound (LCB), relining on the partial coverage assumption on the dataset. \cite{sunwenOffline} analyze the constrained pessimistic policy optimization with general function approximation and with the partial coverage assumption of the dataset, then they specialize the case in the KNR setting and give a refined upper bound.
The importance of pessimism in offline RL is characterized by \cite{importanceofpess,LowerboundOfflineRL} through discussing the lower bound of offline RL when the dataset has no restriction. 
\paragraph{Notations}
We denote by $[N]=\{1,\ldots,N\}$. We also denote by $\|\cdot\|_{2}$ the $\ell_{2}$-norm of a vector and denote by $\|\cdot\|_A$ the spectral norm of a matrix $A$. We denote by $\Delta(\mathcal{X})$ the set of probability distributions on a set $\mathcal{X}$ and correspondingly define
$\Delta(\mathcal{A} \mid \mathcal{S}, H)=\{\{\pi_{h}(\cdot \mid \cdot)\}_{h\in [H]}: \pi_{h}(\cdot \mid s) \in \Delta(\mathcal{A})$ for any $(s,h) \in \mathcal{S} \times [H]\}$
for all set $\mathcal{S}$ and $H \in \mathbb{ N}_{+}$.  For $p_{1}, p_{2} \in \Delta(\mathcal{A}),$ we denote by $D_{\mathrm{KL}}(p_{1} \| p_{2})$ the KL-divergence, that is,
$$
D_{\mathrm{KL}}(p_{1} \| p_{2})=\int_\cA p_{1}(a) \log \frac{p_{1}(a)}{p_{2}(a)} \mathrm{d}a.
$$ And $\langle\cdot, \cdot\rangle_{\mathcal{A}}$ is the inner product taken over the action space $\mathcal{A}$. We also denote by $\delta_x$ Dirac function centered at $x$. We denote by Vol$(\mathcal{X})$ by the measure of set $\mathcal{X}$.

\section{Preliminary}
In this section, we first introduce the notion of the episodic Markov decision process (MDP). Then we introduce generative adversarial imitation learning in the online and offline settings, respectively. Finally we introduce the definition of linear function approximation.
\subsection{Episodic Markov Decision Process}
We consider an episodic MDP $(\cS, \cA, H, \cP, r)$, where $\cS$ and $\cA$ are the state and action spaces, respectively, $H$ is the length of each episode, $\cP_h$ is the Markov transition kernel of the $h$-th step of each episode for any $h \in [H]$, and $r_h\colon \cS\times \cA\to [0,1]$ is the reward function at the $h$-th step of each episode for any $h \in [H]$. We assume without loss of generality that the reward function $r_h$ is deterministic. 

In the episodic MDP, the agent interacts with the environment as follows.  At the beginning of each episode, the agent determines a policy $\pi = \{\pi_h\}_{h\in[H]} \in \Delta(\cA \given \cS, H)$. Then the agent takes the action $a_h \sim \pi_h(\cdot\given s_h)$ at the $h$-th step of the $k$-th episode, observes the reward $r_h(s_h, a_h)$, and transits to the next state $s_{h+1} \sim \cP_h(\cdot\given s_h, a_h)$. The episode terminates when the agent reaches the state $s_{H+1}$. Without loss of generality, we assume that the initial state $s_1 = x$ is fixed across different episodes. We remark that our algorithms and corresponding analyses readily generalize to the
setting where the initial state $s_1$ is sampled from a fixed distribution. 

We now define the value functions in the episodic MDP.   For any policy $\pi = \{\pi_h\}_{h\in[H]}$ and reward function $r = \{r_h\}_{h\in[H]}$, the state- and action-value functions are defined for any $(s,a,h) \in \cS\times\cA \times[H]$ as follows, 
\begin{equation}
V_{h,\pi}^{r}(s)=\mathbb{E}_{\pi}\Big[\sum_{i=h}^{H} r_{i}(s_{i}, a_{i}) \Big| s_{h}=s\Big], \quad 
Q_{h,\pi}^{r}(s, a)=\mathbb{E}_{\pi}\Big[\sum_{i=h}^{H} r_{i}(s_{i}, a_{i})  \Big| s_{h}=s, a_{h}=a\Big], 
\label{eq:def:Q-function}
\end{equation}
where the expectation $\EE_\pi[\cdot]$ is taken with respect to the action $a_i \sim \pi_i(\cdot \given s_i)$ and the state $s_{i+1}\sim \cP_i(\cdot\given s_i, a_i)$ for any $i \in \{h, h+1, \ldots, H\}$. With slight abuse of notations, we also denote by $\cP_h$ the operator form of the transition kernel such that $(\cP_h f)(s,a) = \EE_{s'\sim \cP_h(\cdot\given s,a)}[f(s')]$ for any $f\colon \cS\to \RR$. By the definitions of the value functions in \eqref{eq:def:Q-function}, for any $(s,a,h)\in \cS\times\cA\times [H]$, any policy $\pi$, and any reward function $r$, we have
\#
V_{h,\pi}^r(s) = \langle Q_{h,\pi}^r(s,\cdot),\pi_h(\cdot,s)\rangle _{\mathcal{A}}, \quad 
Q_{h,\pi}^r (s,a) = r_h(s,a) + \mathcal{P}_hV_{h+1,\pi}^r(s,a), \quad V_{H+1,\pi}^r(s) = 0,  \label{eq:Bellman}
\#
where $\la \cdot, \cdot \ra$ denotes the inner product over the action space $\cA$. 
We further define the expected cumulative reward as follows, 
\#\label{eq:def-J}
J(\pi, r)=V_{1,\pi}^{r}(x).
\#
In this paper, we characterize the performance of the agent via the expected cumulative reward $J(\pi, r)$ defined in \eqref{eq:def-J}. 


\subsection{Generative Adversarial Imitation Learning}
Given an expert demonstration $\mathbb{D}^{\mathrm{E}}=\{(s_{h,\tau}^{\mathrm{E}}, a_{h,\tau}^{\mathrm{E}})\}_{h \in[H ],\tau \in [N_1]}$ with $N_1$ trajectories of state-action pairs generated following the underlying MDP and the expert policy $\pi^{\mathrm{E}}$, the goal of GAIL is to find a policy whose performance is close to that of the expert policy $\pi^\rE$ for any reward function in a given set $\cR$ \citep{GAIL}. Here the set $\cR$ is specified later in \S\ref{sec:linear}. 
We assume that the trajectories in the expert demonstration $\mathbb{D}^{\mathrm{E}}$ are independent, which is a standard assumption in the literature \citep{2004-AL-using-IRL,Online-Apprenticeship-Learning}. 
In GAIL, we consider the following minimax optimization problem, 
\#\label{eq:min-max}
\min _{\pi \in \Delta(\cS|\cA,H)} \max _{r \in \cR}  J(\pi^{\mathrm{E}},r)-J(\pi, r), 
\#
where $J(\pi,r)$ is defined in \eqref{eq:def-J}. 

\vskip5pt
\noindent\textbf{Online GAIL.} 
In online GAIL, the agent interacts with the environment to collect state-action pairs following the underlying MDP and the current policy. 
For online GAIL, we are interested in the performance of the algorithm during learning. To this end, we compare the expected cumulative reward corresponding to the algorithm during learning with the expected cumulative reward corresponding to the expert policy under the worst-case scenario, which is defined as follows \citep{Online-Apprenticeship-Learning}, 
\#\label{def:eq:reg}
\mathrm{Regret} (K)= \max_{r \in\mathcal{R} } \sum_{k=1}^K \big[J(\pi^\mathrm{E},r)- J(\pi^k,r)\big],
\#
where $\pi^k$ is the policy of the agent at the $k$-th episode.


\vskip5pt
\noindent\textbf{Offline GAIL.} \label{sec:ORIL}
To simultaneously utilize non-expert data without further interaction with the environment, we consider offline GAIL, which involves an additional dataset to benefit the policy learning.
Specifically, except for the expert demonstration  $\mathbb{D}^{\mathrm{E}} = \{(s_{h,\tau}^\mathrm{E},a_{h,\tau}^\mathrm{E})\}_{h\in[H],\tau \in [N_1]}$ collected by the expert policy $\pi^\mathrm{E}$ in the underlying MDP, the agent has access to an \text{additional} dataset $\mathbb{D}^\rA = \{(s_{h}^{\tau}, a_{h}^{\tau})\}_{h \in[H ],\tau \in [N_2]}$, which is collected a priori by an experimenter in the underlying MDP. In particular, at each step $h \in[H]$ of each trajectory $\tau \in[N_2]$, the experimenter takes the action $a_{h}^{\tau}$ at the state $s_{h}^{\tau}$ and observes the next state $s_{h+1}^{\tau} \sim \mathcal{P}_{h}(\cdot \given s_{h}^{\tau}, a_{h}^{\tau})$.  Here
$a_{h}^{\tau}$ is arbitrarily chosen by the experimenter given the filtration 
\$
\mathcal{F}_{h, \tau}=\sigma\big(\{(s_{i}^{n}, a_{i}^{n}) \colon (n-1) H+i \leq(\tau-1) H+h\}\big),
\$
In other words, in the $\tau$-th trajectory, the action the experiment takes is only determined by the historical information with randomness.
For offline GAIL, we measure the performance of a policy $\pi$ by the $\cR$-distance \citep{chen-GAIL} between the expert policy $\pi^\mathrm{E}$ and $\pi$, which is defined as follows, 
\#\label{eq:def-gap}
\mathbf{D_{\cR}(\pi^\mathrm{E},{\pi}}) = \max_{r\in\mathcal{R}}[J(\pi^\mathrm{E},r)- J(\pi,r)\big].
\#
Here $\cR$ is the reward set, which is specified later in \S\ref{sec:linear}. Optimality gap defined in \eqref{eq:def-gap} can be considered as one episode regret defined in \eqref{def:eq:reg}. When optimality gap of policy $\pi$ approaches zero, it implies that the performance difference between the policy $\pi$ and the expert policy $\pi^\mathrm{E}$ tends to be undistinguishable by the reward set $\mathcal{R}$, which implies that the performance of $\pi$ is measured by both the optimality gap $\mathbf{D}_{\mathcal{R}}(\pi^\mathrm{E},\pi)$ and the richness of the reward set $\mathcal{R}$.

\subsection{Linear Function Approximation}\label{sec:linear}

We consider the linear setting where the transition kernel is linear in a feature map, which is formalized in the following assumption.

\begin{assumption}[Linear Kernel Episodic MDP]\label{def:linear}
Given measurable sets $\cS$ and $\cA$ with finite measure, the episodic MDP $(\cS, \cA, H, \cP, r^\mu)$ is a linear MDP with a feature map $\phi: \cS \times \cA \times \cS \to \RR^d$, that is, for any $h\in [H]$, there exists $\theta_{h}\in \mathbb{R}^{d}$ with $\|\theta_h\|_2\le \sqrt{d}$ such that $\mathcal{P}_{h}  (s' \given s, a ) = \phi  (s, a, s' )^{\top} \theta_{h}$ for any $(s,a,s') \in \cS \times \cA \times \cS$. 
Also, there exists an absolute constant $R>0$ such that
\$
R^{-2} \cdot \sup _{s^{\prime} \in \mathcal{S}}  |\phi  (s, a, s^{\prime} )^{\top} y |^{2} \leq \int _{\cS} |\phi  (s, a, s^{\prime} )^{\top} y |^{2} \mathrm{~d} s^{\prime} \leq d,
\$
for any $(s, a) \in \mathcal{S} \times \mathcal{A}$ and $y\in\mathbb{R}^d$ with $\|y\|_{2} \leq 1$.
\end{assumption}

Under Assumption \ref{def:linear}, we further assume that there exists a feature map $\psi\colon \cS \times \cA \to \RR^d$ such that the reward set $\cR$ in \eqref{eq:min-max} takes the following form, 
\#\label{eq:B-mu}
\mathcal{R}  = \bigl \{r^{\mu} \colon r_h^\mu(\cdot,\cdot) = \psi(\cdot,\cdot)^\top \mu_h  \text{ for any }  (h,\mu)\in[H] \times S  \bigr\},
\#
where $r^{\mu}=\{r_h^\mu(\cdot,\cdot)\}_{h\in [H]}$ is the reward function and $\mu = \{\mu_h\}_{h\in [H]}$ is the reward parameter. Here $S$ is the reward parameter domain, which is defined as follows, 
\#\label{eq:def:parameter domain}
S = \{\mu  \colon \mu_h\in B \text{ for any } h\in[H] \}, \quad \text{ where } B  = \{ u \in \mathbb{R}^d \colon \|u\|_2\le \sqrt{d} \}.
\#
We assume that $\|\psi(s,a)\|_2\le 1$ for any $(s,a)\in\cS\times\cA$, which ensures that $r_h^\mu (s,a)\in[0,\sqrt{d}\ ]$ for any $(s,a,h,\mu)\in\cS\times\cA\times [H]\times S$. 
For notational convenience, for any reward function $r^\mu$,  we denote by the GAIL objective function $L(\pi, \mu)$ as follow, 
\#\label{def:L}
L(\pi, \mu) = J(\pi^{\rm E}, r^\mu) - J(\pi, r^\mu), 
\#
where $J(\pi, r^\mu)$ is defined in \eqref{eq:def-J}. 

Assumption \ref{def:linear} corresponds to the linear kernel MDP model in RL. See \cite{linear-approximation-example-1,Gu-probility-measure,OPPO} for various examples of linear kernel MDPs. 
We remark that the existence of $R$ in Assumption \ref{def:linear} can be guaranteed if for any $(s, a) \in \mathcal{S} \times \mathcal{A}$, the feature map $\phi(s, a, \cdot)$ is upper bounded and Lipschitz continuous.  In particular, a tabular MDP where the state space $\mathcal{S}$ and the action space $\mathcal{A}$ are both finite, is a special case of the linear kernel MDP in Assumption \ref{def:linear} with $d=|\mathcal{S}|^{2}|\mathcal{A}|$ and the feature map $\phi  (s, a, s' )$ being the canonical basis $e_{  (s, a, s' )}$ of $\mathbb{R}^{|\mathcal{S}|^{2}|\mathcal{A}|}$.  It implies that our analysis for GAIL with linear function approximation also covers the tabular case. We also remark that the range of the reward function is $[0,\sqrt{d}]$ instead of $[0,1]$. It means that, with the increasing $d$, we can  enrich the reward set $\mathcal{R}$ and then capture the performance of the policy more meticulously using the optimality gap defined in \eqref{eq:def-gap}. Analysis in GAIL  with linear reward set is more challenging than the case with bounded reward set, which is studied in the previous literature \citep{Online-Apprenticeship-Learning,KNR}.


\section{Algorithm}

We first propose a unified framework in Algorithm \ref{framework} to solve GAIL in both online and offline settings. 
Then we specify the framework for online and offline settings in {\bf O}ptimistic {\bf G}enerative {\bf A}dversarial {\bf P}olicy optimization (OGAP in Algorithm \ref{alg:gail-online} of \S\ref{section:online}) and {\bf P}essimistic {\bf G}enerative {\bf A}dversarial {\bf P}olicy optimization (PGAP in Algorithm \ref{alg:gail-offline} of \S\ref{section:offline}), respectively. 

\begin{algorithm}[h]
\caption{A Unified Framework for OGAP and PGAP} \label{framework}
\begin{algorithmic}[1]
\renewcommand{\algorithmicrequire}{\textbf{Input:}}
\renewcommand{\algorithmicensure}{\textbf{Output:}} 
\State Initialize $\{Q^0_h\}_{h \in [H]}$ as zero functions over $\cS \times \cA$ and $\{\pi^0_h\}_{h \in [H]}$ as uniform distributions over $\cA$.
\State (PGAP) Construct estimated transition kernels $\{\widehat \cP_h\}_{h \in [H]}$ and uncertainty qualifiers $\{ \Gamma_h\}_{h \in [H]}$ based on~$\D^\rA$. 
\For{$k=1,\ldots,K$}
\State Update policy $\pi^k = \{\pi^k_h\}_{h \in [H]}$ by mirror descent with estimated action-value function $\{\widehat Q^{k-1}_h\}_{h \in [H]}$.
\State (OGAP) Rollout a trajectory following $\pi^k$, and construct empirically estimated transition kernels $\{\widehat \cP_h^k\}_{h \in [H]}$ and bonus functions $\{\Gamma_h^k\}_{h \in [H]}$.
\State (OGAP/PGAP) Optimistically/Pessimistically  estimate action-value function $\{\widehat Q^{k}_h\}_{h \in [H]}$.  
\State (OGAP/PGAP) Estimate $\nabla_\mu L(\pi^k,\mu^k)$ via \eqref{eq:MC_estimation-G}/\eqref{eq:offline-estimator}. 
\State Update reward parameter $\mu^{k+1}$ by projected gradient ascent with estimated $\nabla_\mu L(\pi^k,\mu^k)$.
\EndFor
\State (PGAP) Output the mixed policy $\widehat{\pi}$ of $\{\pi^k\}_{k\in[K]}$.
\end{algorithmic}
\end{algorithm}

This framework in Algorithm \ref{framework} involves two stages: (i) policy update stage and (ii) reward update stage. (i) In the policy update stage, we use mirror descent  to update the policy based on the estimated action-value function constructed in the previous iteration. For OGAP, we sample a trajectory following the updated policy. Then we construct estimated action-value functions with optimism based on the finite historical data for OGAP or pessimism based on the additional dataset for PGAP. (ii) In the reward update stage, we first construct an estimate of the gradient based on the collected trajectory induced by the present policy and the finite historical data for OGAP or estimated action-value functions for PGAP , and then we use projected gradient ascent to update reward parameters via such an estimate of gradient. We further detail OGAP and PGAP in \S\ref{section:online} and \S\ref{section:offline}, respectively.


\subsection{Optimistic Generative Adversarial Policy Optimization}
\label{section:online}\label{alg:online:detail}

To specialize Algorithm \ref{framework} to solve online GAIL, we propose OGAP in Algorithm \ref{alg:gail-online}. We detail the policy update stage and reward update stage as follows. 



\begin{algorithm}[h]
\caption{Optimistic Generative Adversarial Policy Optimization (OGAP)}\label{alg:gail-online}
\begin{algorithmic}[1]
\State{\textbf{Input:} Expert demonstration $\D^\rE$, scaling factor $\kappa$, and step size $\eta$ and $\alpha$.}
\State{Initialize  $\{\widehat{Q}_{h}^{0}\}_{h\in [H]}$ as zero functions over $\cS\times \cA$,  $\{\pi_{h}^{0}\}_{h\in [H]}$ as uniform distributions over $\cA$, $\mu^1 = \{\mu_h^1\}_{h\in [H]}$ as zero vectors, and $\{\widehat{V}_{H+1}^{k}\}_{k\in \{0,1,\ldots, K\}}$ as zero functions over $\cS$.}
\For{$k=1,\ldots,K$}
\For{$h=1,\ldots,H$}\algorithmiccomment{Policy Improvement}\label{line:p1}
\State{$\pi_{h}^{k}(\cdot \given \cdot) \propto \pi_{h}^{k-1}(\cdot \given \cdot) \cdot \exp \{\alpha \cdot \widehat{Q}_{h}^{k-1}(\cdot, \cdot)\}$. } \label{line:policy}
\EndFor\label{line:p2}
\State{Rollout a trajectory $\{(s_h^k,a_h^k)\}_{h\in [H]}$ following $\pi^k$.}
\For{$h=H,\ldots,1$} \algorithmiccomment{Policy Evaluation} \label{line:p3}
\State{Set $\{\widehat{\cP}_h^k\}_{h=1}^H$ and $\{\Gamma_h^k\}_{h=1}^H$ via \eqref{eq:def-online-estimation-P} and \eqref{eq:bonus}, respectively.} \label{line:r1}
\State{$\widehat{Q}_{h}^{k}(\cdot, \cdot)\gets \min\{({r}_h^k+\widehat{\mathcal{P}}^k_{h} \widehat{V}_{h+1}^{k}+\Gamma^k_{h})(\cdot, \cdot), (H-h+1)\sqrt{d}\}_+$.}\label{line:r2}
\State{$\widehat{V}_{h}^{k}(\cdot) \gets\langle\widehat{Q}_{h}^{k}(\cdot, \cdot) ,\pi_{h}^k(\cdot \given \cdot)\rangle_{\mathcal{A}}$.} \label{line:1}
\EndFor\label{line:p4}
\State{Set $\{\nabla_{\mu_h}\Tilde{J}(\pi^{\mathrm{E}},r^\mu)\}_{h\in[H]}$ via \eqref{eq:MC_estimation-G}. \algorithmiccomment{Reward Update}}\label{line:p5}
   \For{$h=1,\ldots,H$}
\State{$\widehat{\nabla}_{\mu_h}L(\pi^k,\mu^k)\gets \nabla_{\mu_h} \tilde{J}(\pi^\rE,r^\mu)\given_{\mu = \mu^k}-\psi(s_{h}^k,a_h^k)$.}
\State{$\mu_h^{k+1}\gets \text{Proj}_{B}(\mu_h^k+\eta\widehat{\nabla}_{\mu_h}L(\pi^k,\mu^k))$.} \label{line:on-proj}
\EndFor\label{line:p6}
\EndFor
\end{algorithmic}
\end{algorithm}


\subsubsection{Policy Update Stage}
The policy update stage (Lines \ref{line:p1}--\ref{line:p4} of Algorithm \ref{alg:gail-online}) consists of two steps: (i) policy improvement (Lines \ref{line:p1}--\ref{line:p2}) and (ii) policy evaluation (Lines \ref{line:p3}--\ref{line:p4}).  In  policy improvement, we apply mirror descent in its proximal form to update the current policy via estimated action-value functions, which is specified in policy evaluation stage. In policy evaluation, we employ the optimism principle to construct the estimated action-value functions, which further utilize estimated  transition kernels and bonus functions. 

\vskip5pt
\noindent\textbf{Policy Improvement.}  To generate a policy whose performance is close to the expert policy $\pi^{\rm E}$, we update the policy $\pi^{k}$ to minimize the GAIL objective function $L(\pi, \mu^{k-1}) = J(\pi^{\rm E}, r^{k-1}) - J(\pi, r^{k-1})$ in \eqref{def:L} under the current reward function $r^{k-1} = r^{\mu^{k-1}}$. Note that $\pi^\rE$ is fixed, then we only need to maximize $J(\pi,r^{k-1})$. 
Applying online mirror descent \citep{mirror-descent-1,mirror-descent-2}, a standard algorithm to solve online learning problem, we update $\pi$ as follows,
\#\label{eq:policy-update-target}
\pi^{k}=\underset{\pi \in \Delta(\mathcal{A} \given \mathcal{S}, H)}{\operatorname{argmax}}\{\mathcal{L}_{k-1}(\pi)-\alpha^{-1} D  (\pi, \pi^{k-1} )\},
\#
where $\alpha$ is the step size, Bregman divergence regularizer $D (\pi, \pi^{k-1} )$ is chosen as the expected KL divergence 
$D (\pi, \pi^{k-1} ) = \mathbb{E}_{\pi^{k-1}}[\sum_{h=1}^{H} D_{\mathrm{KL}}  (\pi_{h}  (\cdot \given s_{h} ) \| \pi_{h}^{k-1}  (\cdot \given s_{h} ) )\given s_1 =x]$, and $\mathcal{L}_{k-1}(\pi)$ takes the form as
\# \label{eq:def:mathcal-L}
\mathcal{L}_{k-1}(\pi)  =   J(\pi^{k-1},r^{k-1}) + \mathbb{E}_{\pi^{k-1}}\Big[\sum_{h=1}^{H}\big\langle \widehat Q_h^{k-1} (s_{h}, \cdot),   \pi_{h}(\cdot \given s_{h})-\pi_{h}^{k-1}(\cdot \given s_{h})\big\rangle_{\mathcal{A}} \biggiven s_{1}=x\Big]. 
\#
Here expectation $\mathbb{E}_{\pi}[\cdot]$ is taken with respect to the trajectory induced by $\pi$ and $\widehat Q_h^{k-1}$ is an estimator of $Q^{r^{k-1}}_{h,\pi^{ k-1}}$, which is specified later in \eqref{eq:est-q}. Such policy update formulation defined in \eqref{eq:policy-update-target} also corresponds to the policy optimization in online RL \citep{NPG,TRPO,ppo,regularizedMDP,TRPO2,OPPO}.


By solving \eqref{eq:policy-update-target}, we obtain the following closed-form update as, 
\#\label{eq:policy-solution}
\pi_{h}^{k}(\cdot \given s) \propto \pi_{h}^{k-1}(\cdot \given s) \cdot \exp \{\alpha \cdot \widehat Q_h^{k-1}(s, \cdot)\} \text{ for any $(s,h) \in \cS\times [H]$}, 
\#
which gives Line \ref{line:policy} of Algorithm \ref{alg:gail-online}.

\vskip5pt
\noindent\textbf{Policy Evaluation.} 
To evaluate the policy $\pi^k$ under the reward function $r^k$, we first construct estimated transition kernels $\widehat{\cP}^k = \{\widehat{\cP}_h^k\}_{h\in[H]}$, through value-target regression \citep{linear-approximation-example-1} on finite historical data in Line \ref{line:r1}, and then construct an estimator of the action-value functions by the Bellman equation in \eqref{eq:Bellman} with an extra bonus term to incorporate exploration in Line \ref{line:r2}. 

Specifically, in the $k$-th episode, we construct our estimated transition kernels  $\widehat{\cP}^k = \{\widehat{\mathcal{P}}^k_{h}\}_{h\in[H]}$ as
\#\label{eq:def-online-estimation-P}
\widehat{\mathcal{P}}_{h}^k(s^\prime\given s,a) = \phi  (s, a, s^\prime )^{\top} \widehat{\theta}_{h}^k \text{ for any $(h,s,a,s^\prime)\in[H]\times \cS\times\cA\times\cS$},
\#
where $\widehat{\theta}^k_{h}$ is the minimizer of the regularized empirical mean-squared Bellman error as follows, 
\#\label{eq:online-regid-regress}
&\widehat{\theta}_{h}^k=\underset{\theta \in \mathbb{R}^{d}}{\operatorname{argmin}} \sum_{\tau=1}^{k-1}  \big|  \varphi_h^\tau(s_h^\tau,a_h^\tau)^\top \theta-\widehat{V}_h^\tau(s_{h}^{\tau})\big|^{2}+\lambda\|\theta\|_{2}^{2}, \quad \text {where } \varphi_{h}^{\tau}(\cdot, \cdot)=\int_{\mathcal{S}} \phi\left(\cdot, \cdot, s^{\prime}\right) \widehat{V}_{h+1}^{\tau}\left(s^{\prime}\right) \mathrm{d} s^{\prime}.
\#
Here $\widehat{V}_h^\tau$ is constructed in Line \ref{line:1} of Algorithm \ref{alg:gail-online} and $\lambda > 0$ is the regularization parameter, which is specified later in Theorem \ref{thm:online}. 
By solving \eqref{eq:online-regid-regress}, we obtain the closed-form update of $\widehat \theta_h^k$ as follows, 
\#\label{eq:online-result-L}
\widehat{\theta}_{h}^k =(\Lambda^k_{h})^{-1} \sum_{\tau=1}^{k-1} \varphi _h^\tau (s_{h}^{\tau}, a_{h}^{\tau} )\widehat{V}_{h+1}^\tau (s_{h+1}^\tau), 
\text {\quad where } \Lambda_{h}^k =\sum_{\tau=1}^{k-1} \varphi_h^\tau  (s_{h}^{\tau}, a_{h}^{\tau}) \varphi_h^\tau  (s_{h}^{\tau}, a_{h}^{\tau})^{\top} \mathrm{d} s^{\prime} + \lambda I.
\#
We use $(r_h^k + \widehat \cP^k_h \widehat V_{h+1}^{k})(s,a) $  as an estimator  of  $Q^{r^{k}}_{h,\pi^{ k}}(s,a)$. 
To further handle the uncertainty incurred by finite historical data and balance between exploration and exploitation, we employ optimism to incentivize exploration as many no-regret online RL algorithms do \citep{MABexploration,UCRL,explorationRL1,ProvablyefficientQlearning,LinearFunctionApproximation,RLinfeaturespace}. Specifically, we incorporate a bonus term into the estimator $(r_h^k + \widehat \cP^k_h \widehat V_{h+1}^{k})(s,a)$ of $Q^{r^{k}}_{h,\pi^{ k}}(s,a)$, i.e., 
\#
\widehat{Q}_{h}^{k}(\cdot, \cdot)&= \min\big\{ ({r}_h^k+\widehat{\mathcal{P}}^k_{h} \widehat{V}_{h+1}^{k}+\Gamma_{h}^k)(\cdot, \cdot) ,(H-h+1)\sqrt{d}\big\}^+, \label{eq:est-q} \\
 \text{where }  \Gamma^k_h (s,a)&= H\sqrt{d}\cdot \min \big \{\kappa \cdot  \|\varphi_h^k  (s, a ) \|_{(\Lambda_{h}^k)^{-1}}, 1 \big\}. \label{eq:bonus}
\#
Here $\widehat{Q}_{h}^{k}$ is truncated into $[0,(H-h+1)\sqrt{d}]$, and $\kappa>0$ is a scaling parameter. 

We highlight that the policy update stage of OGAP (Lines \ref{line:p1}--\ref{line:p4} of Algorithm \ref{alg:gail-online}) corresponds to the no-regret policy optimization in adversarial MDP with full information feedback \citep{advMDPbanditfeed,advMDP2,OPPO,advMDP3JIN}. Such tolerance of arbitrarily chosen reward function every episode paves the way for the alternate update between the policy and reward function.

\subsubsection{Reward Update Stage}
To discriminate the discrepancy between the expert policy $\pi^\rE$ and the current policy $\pi^k$, we update the reward parameter $\mu^{k+1}$ by maximizing GAIL objective function $L(\pi^k,\mu)$ defined in \eqref{def:L}. By projected gradient ascent, we obtain the update of the reward parameter as follows,
\#\label{eq:update-r}
\mu_{h}^{k+1} = \operatorname{Proj}_{B}  \{\mu_{h}^{k} + \eta \widehat {\nabla}_{\mu_{h}} L (\pi^{k}, \mu^{k} ) \}, 
\#
where  $\eta$ is the stepsize, $\widehat {\nabla}_{\mu_{h}} L (\pi^{k}, \mu^{k} )$ is an estimator of ${\nabla}_{\mu_{h}} L (\pi^{k}, \mu^{k} )$, and ${\rm Proj}\colon \mathbb{R}^{d} \to B$ is the projection operator to restrict the updated reward parameter $\mu_h^{k+1}$ within the ball $B$ for any $h \in [H]$. Here $B$ is defined in \eqref{eq:def:parameter domain}.
Without accessing to the true transition kernels of the underlying MDP and the expert policy $\pi^\rE$, we need to obtain an estimator $\widehat {\nabla}_{\mu_{h}} L (\pi^{k}, \mu^{k} )$ in \eqref{eq:update-r}.

Specifically, to construct an estimator of ${\nabla}_{\mu_{h}} L (\pi^{k}, \mu^{k} )$, we first construct a Monte Carlo (MC) estimator $\widehat L (\pi^{k}, \mu^{k} )$ of $L (\pi^{k}, \mu^{k} )$ as follows, 
\#\label{eq:hat-L}
\widehat L (\pi^{k}, \mu^{k} ) = \tilde J(\pi^{\rm E}, r^k) - \tilde J(\pi^k, r^k).
\#
Here $\tilde J(\pi^{\rm E}, r^k)$ and $\tilde J(\pi^k, r^k)$ are MC estimators of $ J(\pi^{\rm E}, r^k)$ and $J(\pi^k, r^k)$, respectively, which are defined as follows, 
\#\label{MC:estimate}
\tilde J(\pi^{\rm E}, r^k)  = \frac{1}{N_1}\sum_{\tau=1}^{N_1}\sum_{h=1}^H \psi (s_{h,\tau}^\rE,a_{h,\tau}^\rE)^\top \mu_h, \quad 
\tilde J(\pi^k, r^k)  = \sum_{h=1}^H \psi (s_{h}^k,a_{h}^k)^\top \mu_h, 
\#
where we use $N_1$ trajectories in $\tilde J(\pi^{\rm E}, r^k)$ and one trajectory in $\tilde J(\pi^k, r^k)$. Combining \eqref{eq:hat-L} and \eqref{MC:estimate}, we obtain that
\#\label{eq:MC_estimation-G}
\widehat{\nabla}_{\mu_{h}} L (\pi^{k}, \mu^{k} ) = \frac{1}{N_1}\sum_{\tau=1}^{N_1} \psi (s_{h,\tau}^\rE,a_{h,\tau}^\rE) - \psi(s_h^{k},a_h^k). 
\#
We use $\widehat{\nabla}_{\mu_{h}} L (\pi^{k}, \mu^{k} )$ as an estimator of $\nabla_{\mu_{h}} L (\pi^{k}, \mu^{k} )$, which gives Lines \ref{line:p5}--\ref{line:p6} of Algorithm \ref{alg:gail-online}.

\subsection{Pessimistic Generative Adversarial Policy Optimization}\label{section:offline}
To specialize Algorithm \ref{framework} to solve offline GAIL, we propose PGAP in Algorithm \ref{alg:gail-offline}. Besides the policy update stage and reward update stage, PGAP further contains an initial construction stage, which constructs estimated transition kernels and bonus functions at the beginning of the algorithm. We detail the initial construction, the policy update, and the reward update stage as follows. 


\begin{algorithm}[h]
\caption{Pessimistic Generative Adversarial Policy Optimization (PGAP) }\label{alg:gail-offline}
\begin{algorithmic}[1]
\State \textbf{Input:} Expert demonstration $\mathbb{D}^\rE$, the additional 
dataset $\mathbb{D}^\rA$, step size $\eta,\alpha$
\State Initialize  $\{\widehat{Q}_{h}^{0}\}_{h\in[H]}$ as zero functions,  $\{\pi_{h}^{0}\}_{h\in[H]}$ as
uniform distribution, $\mu^1 = \{\mu_h^1\}_{h\in [H]}$ as zero vectors, and $\{\widehat{V}_{H+1}^{k}\}_{k\in \{0,1,\ldots, K\}}$ as zero functions over $\cS$.
\State Construct $\{\widehat{\cP}_h\}_{h\in [H]}$ and $\{\Gamma_{h} \}_{h\in [H]}$  from $\D^\rA$ via \eqref{eq:def:estimation-kernel} and \eqref{eq:xi-uncertainty}, respectively. \algorithmiccomment{Initial Construction}\label{line:pp6}
\For{$k=1,\ldots,K$}
\For{$h=1,\ldots,H$} \algorithmiccomment{Policy Improvement}\label{line:pp1}
\State $\pi_{h}^{k}(\cdot \given \cdot) \propto \pi_{h}^{k-1}(\cdot\given\cdot) \exp \{\alpha \cdot \widehat{Q}_{h}^{k-1}(\cdot, \cdot)\}$.  
\EndFor \label{line:pp3}
\For{$h=H,\ldots,1$} \algorithmiccomment{Policy Evaluation}\label{line:pp4}
\State $\widehat{Q}_{h}^{k}(\cdot, \cdot)\leftarrow \max \{({r}_h^k+\widehat{\mathcal{P}}_{h} \widehat{V}_{h+1}^{k}-\Gamma_{h})(\cdot, \cdot),0\}$. \label{line:pp5}
\State $\widehat{V}_{h}^{k}(\cdot) \leftarrow\langle\widehat{Q}_{h}^{k}(\cdot, \cdot) ,\pi_{h}^k(\cdot \given \cdot)\rangle_{\mathcal{A}}.$
 \EndFor \label{line:pp2}
 \State Construct $\{\nabla_{\mu_h}\Tilde{J}(\pi^{\mathrm{E}},r^\mu)\}_{h\in[H]}$  via \eqref{eq:MC_estimation-G}. \label{line:p-r}
 \algorithmiccomment{Reward Update}
 \State \text{Construct} $\{{\nabla}_{\mu_h}\widehat{J}(\pi^k,r^\mu)\}_{h\in[H]}$ via Proposition \ref{prop:grad}.
 \label{eq:grad-estimate}   
 \For{$h=1,\ldots,H$}
 \State 
 $ {\nabla}_{\mu_h}\widehat{L}(\pi^k,\mu^k)\leftarrow\nabla_{\mu_h}\tilde{J}(\pi^{\mathrm{E}},r^{\mu})\given_{\mu = \mu^k}-{\nabla}_{\mu_h}\widehat{J}(\pi^k,r^\mu)\given_{\mu=\mu^k}.$
 \State 
  $\mu_{h}^{k+1} \leftarrow \operatorname{Proj}_{B}[\mu_{h}^{k}+\eta {\nabla}_{\mu_{h}} \widehat{L}(\pi^{k}, \mu^{k})].$ \label{line:pp-proj}
 \EndFor\label{line:p-r2}
\EndFor
\State \textbf{Output:} $\widehat{\pi} = {\rm Unif}(\{\pi^k\}_{k\in[K]})$.
\end{algorithmic}
\end{algorithm}

\subsubsection{Initial Construction Stage}\label{offline:construction}
In Line \ref{line:pp6} of PGAP, we construct estimated transition kernels $\{\widehat{\cP}_h\}_{h\in[H]}$ and uncertainty quantifiers $\{\Gamma_{h} \}_{h\in[H]}$ via the additional dataset $\D^{\rm A}$. 
Before we detail such construct, we first introduce the following definition of uncertainty quantifiers \citep{pessimismRl-Jinchi} with the confidence parameter $\xi\in(0,1)$, which quantifies the uncertainty.

\begin{definition}[$\xi$-Uncertainty Quantifier]\label{def:xi-uncertainty}
We say $\{\Gamma_{h}\}_{h\in[H]}$ with $\Gamma_{h}: \mathcal{S} \times \mathcal{A} \to \mathbb{R}$ are $\xi$-uncertainty quantifiers for estimated kernels $\widehat{\cP} = \{\widehat{\cP}_h\}_{h\in [H]}$ with respect to $\mathbb{P}_{\D}$ if the event
\$
\mathcal{E}=&\big\{ |\widehat{\cP}_h\widehat{V}(s,a)-{\cP}_h\widehat{V}(s,a)|\le \Gamma_h(s,a) 
\text{ for any } (s,a,h)\in\cS\times\cA\times[H] \text{ and any } \widehat{V}: \cS\rightarrow [0,H\sqrt{d}\ ]\big\}
\$
satisfies $\mathbb{P}_{\D}(\mathcal{E}) \geq 1-\xi/2$. Here $\mathbb{P}_{\D} $ is with respect to the joint distribution of $\ \mathbb{D}^\rA\cup\mathbb{D}^\rE$.
\end{definition}

We remark that the $\xi$-uncertainty quantifiers in Definition \ref{def:xi-uncertainty} is a counterpart of the bonus functions in OGAP. 
Recalling that $|r_h(s,a)|\le\sqrt{d}$ for any $(s,a,h)\in\cS\times\cA\times[H]$, Definition \ref{def:xi-uncertainty} implies that with probability at least $1-\xi/2$, 
the deviation between the true Bellman equation in \eqref{eq:Bellman} with $\cP$ and the estimated Bellman equation in \eqref{eq:Bellman} with $\widehat{\cP}$ is upper bounded by the $\xi$-uncertainty quantifier $\{\Gamma_h\}_{h \in [H]}$.

Now, we construct the estimated transition kernels and uncertainty quantifiers as follows. 
Specifically, we first construct the initial estimated transition kernels $\tilde{\cP} = \{\tilde{\mathcal{P}}_{h}\}_{h\in[H]}$ as
\#\label{eq:def-estimation-kernel-initial}
\tilde{\mathcal{P}}_{h}(s^\prime\given s,a)=\phi  (s, a, s^\prime )^{\top} \tilde{\theta}_{h} \quad \text{for any } (s,a,s^\prime,h)\in\cS\times\cA\times\cS\times [H],
\#
where $\widetilde{\theta}_{h}$ is defined as follows, 
\#\label{eq:regid-regress}
\tilde{\theta}_{h}=\underset{\theta \in \mathbb{R}^{d}}{\operatorname{argmin}} \sum_{\tau=1}^{N_2} \int_{\mathcal{S}}  \big|\phi  (s_{h}^{\tau}, a_{h}^{\tau}, s^{\prime} )^{\top} \theta-\delta_{s_{h+1}^{\tau}}  (s^{\prime} )\big|^{2} \mathrm{~d} s^{\prime}+\lambda\|\theta\|_{2}^{2}.
\#
Here $\lambda > 0$ is the regularization parameter and $\delta_x(y)$ is Dirac function. By solving \eqref{eq:regid-regress}, we obtain the closed-form solution of $\tilde{\theta}_{h}$ as follows, 
\#\label{eq:result-regression  of thetaand Lambda}
\tilde{\theta}_{h} =\Lambda_{h}^{-1} \sum_{\tau=1}^{N_2} \phi  (s_{h}^{\tau}, a_{h}^{\tau}, s_{h+1}^{\tau} ), 
\text {\quad where } \Lambda_{h} =\sum_{\tau=1}^{N_2} \int_{\mathcal{S}} \phi  (s_{h}^{\tau}, a_{h}^{\tau}, s^{\prime} ) \phi  (s_{h}^{\tau}, a_{h}^{\tau}, s^{\prime} )^{\top} \mathrm{d} s^{\prime}+\lambda I.
\#
Given \eqref{eq:def-estimation-kernel-initial} and \eqref{eq:result-regression  of thetaand Lambda}, we further construct $\xi$-uncertainty quantifiers $\{\Gamma_h\}_{h\in[H]}$ as follows, 
\#\label{eq:xi-uncertainty}
\Gamma_h (s,a)= H\sqrt{d}\int_\cS \Gamma_h^\cP(s,a,s^\prime)ds^\prime ,
\text{\quad where }\Gamma_{h}^{\mathcal{P}}  (s, a, s^{\prime} ) =\min  \big \{\kappa \cdot  \|\phi  (s, a, s^{\prime} ) \|_{\Lambda_{h}^{-1}}, 1 \big\} .
\#
Here $\kappa>0$ is a scaling parameter. 
In the following lemma, we show that $\{\Gamma_h\}_{h\in [H]}$ in \eqref{eq:xi-uncertainty} are $\xi$-uncertainty qualifiers for $\tilde{\cP}$ in \eqref{eq:def-estimation-kernel-initial} if $\kappa$ is properly chosen.

\begin{lemma}\label{lem:xi-uncertainty}
In (\ref{eq:xi-uncertainty}), we set $\lambda = 1$ and $\kappa = c \cdot R \sqrt{d \log (d H N_2/\xi)}$, where $c>0$ is an absolute constant and $\xi \in(0,1)$ is the confidence parameter. 
Under Assumption \ref{def:linear}, $\{\Gamma_{h}\}_{h\in[H]}$ in \eqref{eq:xi-uncertainty} are $\xi$-uncertainty qualifiers for $\tilde{\cP}$, defined in Definition \ref{def:xi-uncertainty}.
\end{lemma}
\begin{proof}
See Appendix \ref{pf:lem:xi-uncertainty} for a detailed proof.
\end{proof}
However, given $(s,a)\in\cS\times\cA$,  the initial estimated transition kernels $\tilde{\cP}_h(\cdot\given s,a)$ in \eqref{eq:def-estimation-kernel-initial} is not guaranteed to lie within $\Delta(\cS)$. 
 Different from OGAP, we are incapable to update reward functions based on the newly sampled trajectory in offline GAIL. This difference is crucial to the analysis of PGAP (Theorem \ref{thm:gap}), which relies on the fact that estimated GAIL objective function $\widehat{L}(\pi,\mu)$ defined in \eqref{eq:def-Lhat} is concave for $\mu$ (we prove it in Lemma \ref{lem:concave}).
To address this issue, we define a feasible estimation parameter domain $\Theta$ and choose the estimated transition kernel parameter $\widehat{\theta}$ from the feasible domain $\Theta$ \citep{Gu-probility-measure}. Formally, we take $\Theta = \Theta_1\cap \Theta_2$ with $\Theta_1$ and $\Theta_2$ defined as follows,
\#\label{eq:def-theta-1}
& \Theta_1 = \big\{\widehat{\theta} \colon \text{if $\cE$ holds, then it satisfies that }  |\widehat{\cP}_h\widehat{V}(s,a)-\tilde{\cP}_h\widehat{V}(s,a)|\le \Gamma_h(s,a) \notag \\
&\qquad\qquad \text{ for any } (s,a,h)\in\cS\times\cA\times[H] \text{ and any } \widehat{V}: \cS\rightarrow [0,H\sqrt{d}]\big\},\\
& \Theta_2 = \big\{\widehat{\theta} \colon  \tilde{\cP}_h (\cdot\given s,a)\in \Delta(\cS) \text{ for any } (s,a,h)\in\cS\times\cA\times[H] \big\}, \notag 
\#
where $\mathcal{E}$ and $\{\widetilde \cP_h\}_{h\in [H]}$ are defined in Definition \ref{def:xi-uncertainty} and \eqref{eq:def-estimation-kernel-initial}, respectively. 
We remark that under Assumption \ref{def:linear}, the true transition kernel parameter $\theta = \{\theta_h\}_{h\in [H]}$ lies within the feasible estimation parameter domain $\Theta$, which implies that $\Theta$ is not empty. Thus, similar to \eqref{eq:regid-regress} but enforcing the estimated transition kernel parameter to lie within $\Theta$, we define $\widehat \theta = \{\widehat \theta_h\}_{h\in [H]}$ as follows, 
\#\label{eq:regid-regress-hat}
\widehat{\theta}_{h}=\underset{\theta \in \Theta}{\operatorname{argmin}} \sum_{\tau=1}^{N_2} \int_{\mathcal{S}}  \big|\phi  (s_{h}^{\tau}, a_{h}^{\tau}, s^{\prime} )^{\top} \theta-\delta_{s_{h+1}^{\tau}}  (s^{\prime} )\big|^{2} \mathrm{~d} s^{\prime}+\lambda\|\theta\|_{2}^{2}, 
\#
where the minimization is taken over $\Theta$. Similarly, we construct the estimated transition kernel  $\widehat{\cP}=\{\widehat{\cP}_h\}_{h\in[H]}$ as follows,
\#\label{eq:def:estimation-kernel}
\widehat{\cP}_h  (s^\prime\given s,a)=\phi  (s, a, s^\prime )^{\top} \widehat{\theta}_{h} \quad \text{for any } (s,a,s^\prime,h)\in\cS\times\cA\times\cS\times [H],
\#
where $\widehat{\theta}_{h}$ is defined in \eqref{eq:regid-regress-hat}.


\subsubsection{Policy Update Stage}
As a pessimistic variant of OGAP in Algorithm \ref{alg:gail-online}, the policy update stage of PGAP (Lines \ref{line:pp1}--\ref{line:pp2} of Algorithm \ref{alg:gail-offline}) consists of two steps: (i) policy improvement (Lines \ref{line:pp1}--\ref{line:pp3}) and (ii) policy evaluation (Lines \ref{line:pp4}--\ref{line:pp2}). In the stage of policy improvement, we adopt the same idea as in OGAP, which employs mirror descent to update the current policy. In the stage of policy evaluation, instead of the optimism principle, we employ the pessimism principle to construct the estimated action-value functions via the additional dataset $\D^\rA$, which is not assumed to be well-explored as specified later in \S\ref{section:main:1}. The principle pessimism-in-face-of-uncertainty guides the agent to be conservative to visit the states and
actions that are less covered by the additional dataset $\mathbb{D}^\mathrm{A}$ \citep{CQL,pessimismRl-Jinchi,PessFQI,PessOfflinPOLICY,PessOfflinPOLICY2,importanceofpess}. Specifically, we construct the estimated action-value functions as follow, 
\$
\widehat{Q}_{h}^{k}(\cdot, \cdot) = \max \bigl\{({r}_h^k+\widehat{\mathcal{P}}_{h} \widehat{V}_{h+1}^{k}-\Gamma_{h})(\cdot, \cdot),0 \bigr\}, 
\$
where  $\{\widehat{\cP}_h\}_{h\in[H]}$ are the estimated transition kernels and $\{\Gamma_{h} \}_{h\in[H]}$ are the uncertainty quantifiers defined in Line \ref{line:pp6} of PGAP. 


\subsubsection{Reward Update Stage}
Similar to the  reward update stage of OGAP, in the reward update stage of PGAP, we update the reward parameter as follows,
\#\label{eq:update-r-p}
\mu_{h}^{k+1} = \operatorname{Proj}_{B}  \{\mu_{h}^{k} + \eta \widehat {\nabla}_{\mu_{h}} L (\pi^{k}, \mu^{k} ) \}.
\#
Here $\eta$ is the stepsize, $\widehat {\nabla}_{\mu_{h}} L (\pi^{k}, \mu^{k} )$ is an estimator of ${\nabla}_{\mu_{h}} L (\pi^{k}, \mu^{k} )$,  and ${\rm Proj}\colon \mathbb{R}^{d} \to B$ is the projection operator to restrict the updated reward parameter $\mu_h^{k+1}$ within the ball $B$ for any $h \in [H]$. Here $B$ is defined in \eqref{eq:def:parameter domain}. To achieve \eqref{eq:update-r-p}, we also need to obtain the estimated gradient $\widehat \nabla_{\mu_h} L(\pi^k,\mu^k)$ in \eqref{eq:update-r-p}.
However, since the agent in offline GAIL cannot interact with the environment to collect the state-action pairs following current policy $\pi^k$, the estimator in \eqref{eq:MC_estimation-G} for OGAP is not applicable to PGAP. Instead, we construct an estimator $\widehat{L}(\pi^k,\mu^k)$ for $ L (\pi^{k}, \mu^{k})$ and use its gradient ${\nabla}_{\mu_{h}} \widehat L (\pi^{k}, \mu^{k} )$ to estimate ${\nabla}_{\mu_{h}} L (\pi^{k}, \mu^{k} )$. Specifically, we define the estimator $\widehat{L}(\pi^k,\mu^k)$ as
\#\label{eq:def-Lhat}
\widehat{L}(\pi^k,\mu^k) = \tilde{J}(\pi^\rE,r^k) - \widehat{J}(\pi^k,r^k),
\#
where $\tilde{J}(\pi^\rE,r^k)$ is a MC estimator of ${J}(\pi^\rE,r^k)$. Here we estimate $J(\pi^k,r^{k})$ with $\widehat{J}(\pi^k,r^k)=\widehat{V}_{1}^{k}(x)$, which is constructed in Line \ref{line:pp5} of PGAP.
Based on \eqref{eq:def-Lhat}, we construct an estimator $\widehat {\nabla}_{\mu_{h}} L (\pi^{k}, \mu^{k} )$ of $ {\nabla}_{\mu_{h}} L (\pi^{k}, \mu^{k} )$ as follows, 
\#\label{eq:offline-estimator}
\widehat \nabla_{\mu_h}{L}(\pi^k,\mu^k) =  \nabla_{\mu_h}\widehat{L}(\pi^k,\mu^k) = \nabla_{\mu_h}\tilde{J}(\pi^{\mathrm{E}},r^k)-\nabla_{\mu_h}\widehat{J}(\pi^k,r^k),
\#
where $\nabla_{\mu_h} \widehat{J}(\pi^k,r^k)$ is characterized in the following proposition.


\begin{proposition}
   If we define $\{\widehat{Q}_{h}^{k,r^\mu}\}_{h\in[H]}$ and $\{\widehat{V}_{h+1}^{k,r^\mu}\}_{h\in[H]}$ as 
   \begin{equation}
   \begin{aligned}
   \widehat{V}_{H+1}^{k,r^\mu}(\cdot) &= 0,\\
   \widehat{Q}_{h}^{k,r^\mu}(\cdot, \cdot)
   &=\max\big\{({r}^\mu_{h}+\widehat{\mathcal{P}}_{h} \widehat{V}_{h+1}^{k,r^\mu}-\Gamma_{h})(\cdot, \cdot), 0\big\},\\
   \widehat{V}_{h}^{k,r^\mu}(\cdot, \cdot)    
   &=\big\langle\widehat{Q}_{h}^{k,r^\mu}(\cdot, \cdot) ,\pi_{h}^{k}(\cdot \mid \cdot)\big\rangle_{\mathcal{A}},
   \label{eq:def:iota-on-1}
   \end{aligned}
   \end{equation}
   for all $h\in[H]$ and $\mu\in S$.
   It suggests that $\widehat{Q}_h^k = \widehat{Q}_h^{k,r^k}$ and $\widehat{V}_h^k=\widehat{V}_h^{k,r^k}$, where $\widehat{Q}_h^k$ and $\widehat{V}_h^k$ are construted in the policy evaluation stage in PGAP (Algorithm \ref{alg:gail-offline} Lines \ref{line:pp4}--\ref{line:pp2}).
   We can solve $\nabla_{\mu_h} \widehat{V}_{1} ^{k,r^\mu} (x)$ recursively as follows,
   $$
   \nabla_{\mu_h} \widehat{V}_{t} ^{k,r^\mu} (s_{t}) =
   \begin{cases}
       \Big\la \pi^k_h(\cdot\mid s_{t})g_t^k(s_{t},\cdot), \big[\widehat{\cP}_{t}(\nabla_{\mu_h }\widehat{V}_{t+1}^{k,r^\mu})\big](s_{t},\cdot)  \Big\ra_{\cA} &\text{ if }1\le t\le h-1,\\
       \big\la \pi^k_h(\cdot\mid s_{h})g_h^k(s_{h},\cdot), \nabla_{\mu_h }r_h^\mu  (s_h,\cdot)\big\ra_{\cA}&\text{ if }t=h,
   \end{cases}
   $$
where $s_1=x$, $[\widehat{\cP}_h f](s,a)$ is a shorthand of $\int_\cS  f(s^\prime)\widehat{\cP}_h(s^\prime \mid s, a) \mathrm{d}s^\prime$ and  $g_h^k: \cS\times\cA\rightarrow\mathbb{R}$ is defined as
\begin{equation}
   g_h^k(s,a) = \mathbf{1}\big\{\widehat{Q}_{h}^{k,r^\mu}(s,a)>0\big\},
\end{equation}
for all $(s,a)\in\cS\times\cA$. Here $\mathbf{1}\{\cdot\}$ is the indicator function.
\label{prop:grad}
\end{proposition}

\begin{proof}
   Take gradient toward $\mu_h$ and apply chain rule on \eqref{eq:def:iota-on-1}, then we obtain Proposition \ref{prop:grad}.
\end{proof}

    \section{Main Results}
    In this section, we present the theoretical analysis for OGAP and PGAP in \S\ref{section:online-thm} and \S\ref{section:main:1}, respectively. Specifically, in \S\ref{section:online-thm} we upper bound the regret of OGAP. In \S\ref{section:main:1}, we upper bound the optimality gap of PGAP under no coverage assumption and propose a lower bound to show that PGAP achieves minimax optimality in the utilization of the additional dataset $\D^\rA$. Moreover,  under the assumption that the additional dataset $\D^\rA$ has sufficient coverage, we establish the global convergence guarantee for PGAP and discuss how the additional dataset $\D^\rA$ facilitates our policy learning.
    \subsection{Analysis of OGAP }\label{section:online-thm}
    We derive an upper bound of the regret of OGAP in the following theorem.
    \begin{theorem}[Regret of OGAP]\label{thm:online}
    In Algorithm \ref{alg:gail-online}, we set
    \$
    \alpha=\sqrt{2 \log |\mathcal{A}| / (H^2\sqrt{d}K)},\quad \lambda=1 ,\quad \kappa=C \sqrt{d  \log (HdK / \xi)} ,\quad \eta=1 / \sqrt{HK},
    \$
    where $C>0$ is a constant. Under Assumption \ref{def:linear}, it holds with probability at least $1-\xi$ that
    \#\label{eq:reg}
    \textrm{Regret}(K) \leq \cO \big (H^2 d^{3/2}K^{1/2}\log(HdK/\xi)\big)+K \delta_{N_1}
    \#
    where $\delta_{N_1} =\cO(H^{3/2} d N_1^{-1/2}\log(N_1/\xi))$.
    \end{theorem}
    \begin{proof}
    See \S \ref{pf.sketch-online} for a proof sketch.
    \end{proof}
    The first term on the right-hand side of \eqref{eq:reg} scales with ${K}^{1/2}$, which attains the optimal dependency on $K$ for online RL. The second term on the right-hand side of \eqref{eq:reg} is linear in $K$ and depends on the MC estimation error $\delta_{N_1}$. As the statistical error from the MC estimation, the error term $\delta_{N_1}$ is inevitable and independent of GAIL algorithm, since we cannot access the expert policy but expert demonstration with $N_1$ trajectories. 
    When the number of trajectories $N_1$ in the expert demonstration is sufficiently large such that $N_1 = \Omega(K)$, the first term on the right-hand side of \eqref{eq:reg} dominates the regret upper bound so that the regret of OGAP scales with  $K^{1/2}$. The dependency of $H,K,$ and $N_1$ correspond to $\tilde{\cO}(H^2 |\cS| |\cA|^{1/2} K^{1/2}+H^{3/2} |\cS|^{1/2}|\cA|^{1/2}KN_1^{-1/2} )$ regret in the tabular case, established by \cite{Online-Apprenticeship-Learning}. If we consider the case $K=\Omega(N_1^{3/2})$, then the average regret decays at a rate of $N_1^{-1/2}$ and the dependency for $H$ turns from $H^2$ into $H^{3/2}$. As $K$ and $N_1$ both tend to infinity, the average regret also shrinks to zero, meaning that the output policy has the same performance on average with the expert policy with respect to the linear reward set $\mathcal{R}$.   
    
    \vskip5pt
    \noindent\textbf{Relationship with Online RL}
    According to Assumption \ref{def:linear}, if we constrain the reward set $\mathcal{R}$ to a fixed reward function $r=\{r_h\}_{h\in[H]}$, then GAIL \eqref{eq:min-max} is reduced to RL, with respect to an episodic MDP $(\mathcal{S},\mathcal{A},H,\mathcal{P},r)$, where $\mathcal{S},\mathcal{A},H,\mathcal{P}$ are the same as the ones in Assumption \ref{def:linear}. Hence OGAP can also be considered as RL algorithm for episodic MDP with linear function approximation. From the aspect for information-theory, the lower bound of the regret of any online RL algorithm is $\tilde{\cO}(K)$ even in tabular case \citep{ProvablyefficientQlearning}. Since the reward set is singleton, OGAP needs not MC estimation and has regret $\tilde{\cO}(K)$, achieving such lower bound and revealing minimax optimality.

    \subsection{Analysis of PGAP}\label{section:main:1}
    
    We upper bound the optimality gap of PGAP in the following theorem.
    \begin{theorem}
    \label{thm:gap}
    (Optimalty Gap of PGAP). In Algorithm \ref{alg:gail-offline}, we set 
    $$
    \lambda=1, \quad \kappa=c R \sqrt{d \log (HdK)/\xi},\quad
    \alpha=\sqrt{2\log ( \operatorname{vol}(\mathcal{A}))/(H^2\sqrt{d}K)},\quad \eta = 1/\sqrt{HK},
    $$
    where $c>0$ is a constant. Under Assumption \ref{def:linear}, 
    $\{\Gamma_h\}_{h=1}^H$ constructed in \S\ref{offline:construction} are $\xi$-uncertainty qualifiers defined in Definition \ref{def:xi-uncertainty}. 
    it holds with probability at least $1-\xi$ that
    \#\label{eq:gap}
    \mathbf{D_{\cR}(\pi^\mathrm{E},\widehat{\pi}}) \leq \cO\big(H^2dK^{-1/2}) +\delta_{N_1} + \text { IntUncert}_{\mathbb{D}^\mathrm{A}}^{\pi^{\mathrm{E}}},
    \#
    where $\widehat{\pi}$ is the output policy of Algorithm \ref{alg:gail-offline},  $\text {IntUncert}_{\mathbb{D}^\mathrm{A}}^{\pi^{\mathrm{E}}}=2 \sum_{h=1}^{H}\mathbb{E}_{\pi^{\mathrm{E}}}  [\Gamma_{h}  (s_{h}, a_{h} ) \mid s_1 =x ]$, 
    and $\delta_{N_1} =\cO(H^{3/2} d N_1^{-1/2}\log(N_1/\xi))$.
    \end{theorem}

    \begin{proof}
    See \S \ref{pf.sketch-1} for a proof sketch.
    \end{proof}

    In Theorem \ref{thm:gap}, the first term on the right-hand side of \eqref{eq:gap} is an optimization error term, which is independent of both expert demonstration $\mathbb{D}^{\mathrm{E}}$ and the additional dataset $\mathbb{D}^\rA$. The optimization error term decays at a rate of $K^{-1/2}$.
    The second term on the right-hand side of \eqref{eq:gap}  is related to the MC estimation error and also occurs in the upper bound of the regret of OGAP as in Theorem \ref{thm:online}. 
    The third term on the right-hand side of \eqref{eq:gap} is an intrinsic error $\text {IntUncert}_{\mathbb{D}^\mathrm{A}}^{\pi^{\mathrm{E}}}$, which arises from the uncetainty of estimating Bellman equation \eqref{eq:Bellman} based on the additional dataset $\mathbb{D}^\rA$.
    In the structure of the intrinsic error $\text {IntUncert}_{\mathbb{D}^\mathrm{A}}^{\pi^{\mathrm{E}}}$, we note that the expectation is taken with respect to the trajectory induced by the expert policy $\pi^{\mathrm{E}}$, which measures the quality of the additional dataset $\mathbb{D}^\rA$ and is irrelevant to the training process. We clarify that the occurrence of $\text {IntUncert}_{\mathbb{D}^\mathrm{A}}^{\pi^{\mathrm{E}}}$ implies that the optimality gap only depends on how well the additional dataset $\mathbb{D}^\rA$ {covers} the trajectories of the expert policy $\pi^{\mathrm{E}}$ and it is not necessary to assume that the additional dataset $\mathbb{D}^\rA$ is well-explored. Hence, Theorem \ref{thm:gap} relies on no assumption on the coverage of the additional dataset $\mathbb{D}^\rA$, such as uniformly lower bound of densities of visitation measures, the behavior policy to be upper bounded uniformly over the state-action space, the concentrability coefficients are uniformly
    upper bounded, or even the partial coverage assumption \citep{uniform3,uniform4,upper-bound-a,upper-bound-b,uniform-covering,offlinepartialcoverage,partial2,partial3,nerual-gail,ErrorBoundINIMITATIONLEARNING}.

    We also highlight that Theorem \ref{thm:gap} can be generalized to the case when the transition kernel is non-linear, only if we explicitly find proper uncertainty quantifiers $\{\Gamma_h\}_{h=1}^H$ satisfying Definition \ref{def:xi-uncertainty} for the estimated transition kernel. 
    
    \vskip5pt
    \noindent\textbf{Pessimism Guarantees Minimax Utilization.} 
    With a well-explored and large enough dataset, the full information about the transition kernel can be extracted by the agent and supports the agent to make correct decision. But when we assert no restriction on the dataset, it is challenging to do the same  because of the distribution shift and extrapolation error on the states
    and actions that are less covered by the dataset. This problem has been studied widely in offline RL \citep{BCQ,CQL,offlineRL2,uniform-covering,pessimismRl-Jinchi} and \cite{limitOffline} even propose that the lower bound of offline RL can grow exponentially with the horizon under linear approximation and no assumption on the dataset. Hence how to cooperate the additional dataset $\D^\mathrm{A}$ to aid the agent to improve the performance in offline GAIL is also difficult since the additional dataset $\D^\mathrm{A}$ is not assumed to be well-explored.
    Inspired by the spirit of being conservative in offline RL \citep{offlineRL2,CQL,pessimismRl-Jinchi}, we propose a pessimistic variant of policy optimization in the policy update stage of PGAP (Lines \ref{line:pp1}--\ref{line:pp2} of Algorithm \ref{alg:gail-offline}), which ensures that PGAP utilize the information of the additional dataset $\mathbb{D}^\mathrm{A}$ in the sense of minimax optimality.
    To illustrate it, we present the following proposition, which is adapted from Theorem 4.7 in \cite{pessimismRl-Jinchi}.
     
    \begin{proposition}[Minimax Optimality in Utilizing Additional Dataset]
        For the output policy $\texttt{Algo}(\mathbb{\overline{D}})$ of any offline algorithm only based on the dataset $\mathbb{\overline{D}}$, there exists a linear kernel MDP $\mathcal{M}\  (\cS,\cA,H,\cP,r)$ with an initial state $x \in\mathcal{S}$, a dataset $\mathbb{\overline{D}}$ compliant with $\mathcal{M}$, and a reward set $\mathcal{R}$  
        , such that
        $$
        \max_{\pi^\mathrm{E}\in\Delta(\cA\mid\cS,H)}\mathbb{E}_{\mathbb{\overline{D}}}\bigg[\frac{\mathbf{D}_\cR(\pi^\mathrm{E},\texttt{Algo}(\mathbb{\overline{D}}))}{\text{Information}_{\mathbb{\overline{D}} }^{\pi^\mathrm{E}}  }\bigg]\ge c,
        $$
        where  $c>0$ is an absolute constant, $\mathbb{E}_{\mathbb{\overline{D}}}[\cdot]$ is taken expectation with respect to randomness of the dataset $\mathbb{\overline{D}}$, and $\text{Information}_{\mathbb{\overline{D}}}^{\pi^\mathrm{E}}$ 
        is defined as 
        $$\text{Information}_{\mathbb{\overline{D}}}^{\pi^\mathrm{E}} =(\text{Vol}(\cS))^{-1}\cdot \mathbb{E}_{\pi^\mathrm{E}}\Big[\sum_{h=1}^H \int_{\cS}\|\phi(s_h,a_h,s^\prime)\|_{\Lambda_h^{-1}}\mathrm{d}s^\prime \Big|\, s_1=x\Big],$$
        \label{prop:minimax}
        where $\Lambda_h$ is only determined by the dataset $\mathbb{\overline{D}}$ and takes the same form as in \eqref{eq:result-regression  of thetaand Lambda}.
    \end{proposition}
    \begin{proof}
        See \S \ref{sec:minimax} for a detailed proof.
    \end{proof}
    According to Proposition \ref{prop:minimax}, if we consider the additional dataset $\mathbb{D}^\mathrm{A}$ as $\overline{\D}$, it reveals that $\text {IntUncert}_{\mathbb{D}^\mathrm{A}}^{\pi^{\mathrm{E}}}$  in the upper bound of optimality gap of PGAP (Theorem \ref{thm:gap}) matches the lower bound up to $H,\sqrt{d},\text{Vol}(\cS)$, and the scaling parameter $\kappa$ defined in \eqref{eq:xi-uncertainty}.
    Though we do not assume any restriction on the additional dataset $\mathbb{D}^\mathrm{A}$, owing to the pessimism principle, PGAP ensures the good utilization of the additional dataset $\mathbb{D}^\mathrm{A}$ even in the worst case.
    
    In the sequel, we show that PGAP is provably efficient and attains global convergence under the assumption that the additional dataset $\mathbb{D}^\rA$ has sufficient coverage.  We first impose such an assumption on the additional dataset $\mathbb{D}^\rA$ as follows.

    \begin{assumption}[Sufficient Coverage]   \label{assumption:data}
    The additional dataset $\D^\rA$ has sufficient coverage with the expert policy $\pi^{\mathrm{E}}$, that is, there exists an absolute constant $c^{\dagger} > 0$ such that the event
    \$
    &  \mathcal{E}^{\dagger} = \Big\{\ \frac{1}{N_2}\sum_{\tau=1}^{N_2} \int_{\mathcal{S}} \phi  (s_{h}^{\tau}, a_{h}^{\tau}, s^{\prime} ) \phi  (s_{h}^{\tau}, a_{h}^{\tau}, s^{\prime} )^{\top} \mathrm{d} s^{\prime}\\
    &\qquad \qquad \ge c^{\dagger} \cdot \mathbb{E}_{\pi^{\mathrm{E}}} \Big [\int_{\mathcal{S}}\phi  (s_{h}, a_{h},s' ) \phi  (s_{h}, a_{h},s' )^{\top}\mathrm{d}s' \Big]  \text { for any } (s_1,h) \in \mathcal{S} \times [H]\Big\},
    \$
    satisfies $\mathbb{P}_{\mathbb{D}}  (\mathcal{E}^{\dagger} ) \geq 1-\xi/2.$ 
    Here the expectation $\mathbb{E}_{\pi^{\mathrm{E}}}[\cdot]$ is taken with respect to the trajectory induced by $\pi^{\mathrm{E}}$ and $\xi\in(0,1)$ is the confidence level. 
    \end{assumption}

    Assumption \ref{assumption:data} implies that the additional dataset $\mathbb{D}^\rA$ with sufficient coverage {covers} the trajectories of the expert policy $\pi^{\mathrm{E}}$ averagely in the sense of the feature map outer product $\int_\cS \phi(\cdot,\cdot,s^\prime)\phi(\cdot,\cdot,s^\prime)^\top \mathrm{d} s^\prime$. We highlight that sufficient coverage does not assume that the additional dataset  $\mathbb{D}^\mathrm{A}$ to be well-explored dataset \citep{uniform3,uniform4,nerual-gail,ErrorBoundINIMITATIONLEARNING,upper-bound-a,upper-bound-b,uniform-covering}, e.g. restricting that the densities of visitation measures of the behavior policy generating the dataset are uniformly lower bounded,  
    i.e. 
    $$\inf_{(s,a,h)\in\cS\times\cA\times[H]}\big[{\rho_h^{\pi^\mathrm{A}}(s,a)}\big]=c>0,$$
    where $\rho_h^\pi$ is the density of visitation measure on  $\cS\times \cA$ induced by policy $\pi$ at the $h$-step and $\pi^\mathrm{A}$ is the policy of the experimenter who collected the additional dataset $\mathbb{D}^\mathrm{A}$. 
    We also remark that sufficient coverage is a weaker restriction than the partial coverage assumption in offline RL \citep{offlinepartialcoverage,partial2,partial3}, which assumes that $$\sup_{(s,a,h)\in\cS\times\cA\times [H]}\big[{\rho_h^{\pi^\mathrm{E}}(s,a)}/{\rho_h^{\pi^\mathrm{A}}(s,a)}\big]=C^{\pi_\mathrm{E}}< \infty.$$
    Under Assumption \ref{assumption:data}, we present the following corollary.

    \begin{corollary}\label{cor:bound-assumption}
    Under Assumption \ref{assumption:data} and the same assumptions as in Theorem \ref{thm:gap}, it holds with probability at least $1-\xi$ that
    \$
    \mathbf{D_{\cR}(\pi^\mathrm{E},\widehat{\pi}}) \leq \tO\big(H^{2}dK^{-1/2} +H^2d^{3/2}N_2^{-1/2}+H^{3/2}dN_1^{-1/2}\ \big),
    \$
    where $\widehat{\pi}$ is the output of Algorithm \ref{alg:gail-offline}. 
    \end{corollary}
    \begin{proof}
    See \S \ref{pf:cor} for a detailed proof.
    \end{proof}

    Corollary \ref{cor:bound-assumption} proves that under Assumption \ref{assumption:data} the intrinsic error $\text {IntUncert}_{\mathbb{D}^\mathrm{A}}^{\pi^{\mathrm{E}}}$ in Theorem \ref{thm:gap} decays at a rate of $N_2^{-1/2}$, showing that PGAP attains global convergence. 
    We remark that this result does not require the additional dataset $\mathbb{D}^\mathrm{A}$ to be well-explored and only relies on a much weaker assumption as sufficient coverage in Assumption \ref{assumption:data}. 
    This improved result also responds to the information-theoretical lower bound $\Omega(H^2N_2^{-1/2})$ for offline policy evaluation \citep{minimaxoffpolicy}. When $K$, $N_1$, and $N_2$ all tend to infinity, the optimality gap shrinks to zero as a negative square-root rate, meaning that the output policy has the same performance with the expert policy with respect to the reward set $\mathcal{R}$.
    Our subsequent discussion illustrates how the additional dataset $\mathbb{D}^\rA $ facilitates our policy learning in PGAP.
    
    \vskip5pt
    \noindent\textbf{The  Additional dataset Contributes.} 
    We illustrate the contribution of the additional dataset $\D^\rA$ by considering the following two cases.
    \begin{enumerate}
    \item Without accessing an additional dataset, PGAP is also applicable by simply treating the given expert demonstration $\mathbb{D}^{\mathrm{E}}$ as the additional dataset $\mathbb{D}^\rA$ in PGAP (Algorithm \ref{alg:gail-offline}), that is, $\D^{\rm A} = \D^{\rm E}$, which satisfies Assumption \ref{assumption:data}. If we set $K=\Omega (N_1)$, then by Corollary \ref{cor:bound-assumption}, we have 
    \#\label{eq:bound on only expert deomonstration}
    \mathbf{D_{\cR}(\pi^\mathrm{E},\widehat{\pi}})= \tO(H^2 d^{3/2} N_1^{-1/2}).
    \#
    \item If we have access to a large enough additional dataset with sufficient coverage, taking $N_2 = \Omega(d^2 H N_1)$ for instance, then by setting $K=\Omega(N_2)$, we upper bound the optimality gap of PGAP as follows, 
    \#\label{eq:bound on proptioanal-N}
    \mathbf{D_{\cR}(\pi^\mathrm{E},\widehat{\pi}})= \tO(H^{3/2}dN_1^{-1/2}).
    \#
    \end{enumerate} 
    By comparing \eqref{eq:bound on only expert deomonstration} and \eqref{eq:bound on proptioanal-N}, we observe that the additional dataset $\mathbb{D}^\rA$ helps decrease the dependency for $H$ and $d$ in the optimality gap by $H^{1/2}$ and $d^{1/2}$. 
    It implies that we can use a much smaller expert demonstration $\mathbb{D}^{\mathrm{E}}$ to learn a policy as good as the expert policy $\pi^{\mathrm{E}}$, especially when horizon $H$ and feature space dimension $d$ are sufficiently large. 
    This improvement is meaningful in the imitation learning tasks, such as autonomous driving and robotics \citep{imitation+robotics,IL-survey,imitation+atuonomous-driving,autoD}. 

\section{Proof Sketch: Analysis of PGAP}
\subsection{Proof of Theorem \ref{thm:online}}
\label{pf.sketch-online}
\begin{proof}
Recall the definition of regret in \eqref{def:eq:reg} and GAIL objective function $L(\pi,\mu) $ in \eqref{eq:min-max}, we decompose the regret as follows, 
\begin{equation}\label{4.3}
\begin{aligned}
\textrm{Regret}(K)&=  \sup_{\mu\in S}\sum_{k=1}^K L(\pi^k,\mu)\\
&\le\underbrace{\sum_{k=1}^K[J(\pi^{\mathrm{E}},r^k)-J(\pi^k,r^k)]}_{(\text{A})}
+\underbrace{\sup_{\mu\in S}
\sum_{k=1}^K\big[L(\pi^k,r^{\mu})-L(\pi^k,r^k)\big]}_{(\text{B})}.
\end{aligned}
\end{equation}
The intuition of decomposition in \eqref{4.3} is to respectively deal with regret ocurring in the stage of policy update and reward update, which are denoted by term (A) and term (B).

\vskip5pt
\noindent{\bf Upper bound of term (A) in \eqref{4.3}.} 
In what follows, we upper bound term (A) in \eqref{4.3}. 
For the simplicity of later discussion, we define the model prediction error for estimating Bellman equation \eqref{eq:Bellman} in the $h$-th step of $k$-th episode in Algorithm \ref{alg:gail-offline} with reward function $r^\mu$ as follows,
\begin{equation} 
\iota_h^{k}(s,a) := r_h^k(s,a) + [\mathcal{P}_h\widehat{V}_{h+1}^{k}](s,a)-\widehat{Q}_{h}^{k}(s,a),
\label{eq:def-iota}
\end{equation}
for any $(s,a)\in\cS\times \cA,h\in[H],\mu\in S$.

First, we introduce a regret decomposition lemma to decompose term (A) in \eqref{4.3}.

\begin{lemma} [Regret Decomposition for Policy Update] \label{lem4.2} 
It holds for any initial state $x \in \cS$ that
\begin{equation}\label{4.4}
\begin{aligned}
\sum_{k=1}^K\big(V^{r^k}_{1,\pi^\rE}(x)-V^{r^k}_{1,\pi^k}(x)\big)=& \sum_{k=1}^{K}\sum_{h=1}^{H}\mathbb{E}_{\pi^{\mathrm{E}}}\big[\langle \widehat{Q}^{k}_{h}(s_{h},\cdot),\pi_{h}^{k}(\cdot\given s_{h})-\pi_{h}^k(\cdot\given s_{h})\rangle\biggiven s_{1}=x \big]\\
& \qquad + \mathcal{M}_{K,H,2}+\sum_{k=1}^K\sum_{h=1}^H\Big(\mathbb{E}_{\pi^{\mathrm{E}}}[\iota_{h}^{k}(s_{h}^{k},a_{h}^{k})|s_1=x]-\iota_{h}^{k}(s_{h}^{k},a_h^k)\Big).
\end{aligned}
\end{equation}  
Here $\iota_h^k$
is the model prediction error defined in \eqref{eq:def-iota}, and $\{\mathcal{M}_{K,H,m}\}_{(k,h,m)\in [K]\times [H]\times [2]}$ is a martingale adapted to the filtration $\{\mathcal{F}_{k,h,m}\}_{(k,h,m)\in [K]\times [H]\times [2]}$, with respect to the timestep index $t(k,h,m)=(k-1)\cdot2H+(h-1)\cdot2+m$.
\end{lemma} 
\begin{proof} See Appendix \ref{A.1} for a detailed proof.\end{proof}

Lemma \ref{lem4.2} shows that term (A) in \eqref{4.3} can be
decomposed into three terms as follows,
\begin{equation}\label{4.6}
\begin{aligned}
\text{(A)}&=\sum_{k=1}^K \big(V^{r^k}_{1,\pi^\rE}(x)-V^{r^k}_{1,\pi^k}(x)\big)\\
&=\underbrace{\sum_{k=1}^K\sum_{h=1}^H\mathbb{E}_{\pi^{\mathrm{E}}}\big[\langle \widehat{Q}^k_h (s_h,\cdot),\pi_{h}^{\mathrm{E}}(\cdot\given s_h)-\pi_{h}^k(\cdot\given s_{h})\rangle \big|\, s_1=x   \big]}_{\text{(A1)}}+\underbrace{\mathcal{M}_{K,H,2}
}_{\text{(A2)}}\\
&\qquad+\underbrace{
\sum_{k=1}^K\sum_{h=1}^H \big(\mathbb{E}_{\pi^{\mathrm{E}}}[\iota_h^k(s_h,a_h)|s_1=x]-\iota^k_h(s^k_h,a^k_h)\big)}_{\text{(A3)}}.
\end{aligned}
\end{equation}

To upper bound term (A1) and term (A2) in \eqref{4.6}, we introduce the following two lemmas, respectively.
\begin{lemma}[Performance Improvement] \label{lem:performance_improve}
If we set $\alpha=
\sqrt{2\log ( \operatorname{vol}(\mathcal{A}))/(H^2K\sqrt{d}\ )}$ in the policy update stage of OGAP (Line \ref{line:pp6} of Algorithm \ref{alg:gail-online}), then under Assumption \ref{def:linear}, for any inital state $x\in\cS$, it holds that 
\$
\sum_{k=1}^{K} \sum_{h=1}^{H} \mathbb{E}_{\pi^{\mathrm{E}}} \big [  \langle\widehat{Q}_{h}^{k-1}, \pi_{h}^{\mathrm{E}}-\pi_{h}^{k-1} \rangle_{\mathcal{A}} \biggiven s_1 = x\big] \leq \sqrt{2 H^4 \sqrt{d}K\log (\operatorname{vol}(\mathcal{A}))}.
\$
\end{lemma}
\begin{proof}
See  Appendix \ref{pf:lem:performance_improve} for a detailed proof.
\end{proof}

\begin{lemma}\label{lem4.4}
It holds that
\$
|\mathcal{M}_{K,H,2}|\leq 4\sqrt{H^3dK\log(8/\xi)}.
\$
with probability at least $1-\xi/4$, where $\mathcal{M}_{K,H,2}$ is the martingale defined in \eqref{4.4}.
\end{lemma}

\begin{proof}
See Appendix \ref{A.3} for a detailed proof.
\end{proof}
To upper bound the term (A3) in $\eqref{4.6}$, we introduce the following two lemmas.

\begin{lemma}[Optimism]\label{lem4.5}
 Under Assumption \ref{def:linear}, it holds with probability at least $1-\xi/4$ that
\$
-2\Gamma_{h}^k(s,a)\leq \iota^k_h(s,a)\leq 0 \quad \text{for any $(h, k, s, a)\in [H]\times[K]\times \cS\times \cA$}, 
\$
where $\iota^k_h$ is the model
prediction error defined in \eqref{4.4}.
\end{lemma}
\begin{proof}
See Appendix \ref{A.4} for a detailed proof.
\end{proof}
\begin{lemma}\label{lem4.6} Under Assumption \ref{def:linear}, it holds that
\$
\sum_{k=1}^K\sum_{h=1}^H\Gamma_{h}^k(s^k_h,a^k_h)\leq C ^\prime \sqrt{H^4d^3K}\cdot\log(HdK/\xi),
\$
where $C^\prime >0$ is an absolute constant.
\end{lemma}
\begin{proof}
See Appendix \ref{A.5} for a detailed proof.
\end{proof}
Lemma \ref{lem4.5} implies that $\mathbb{E}_{\pi^{\mathrm{E}}}[\iota^k_h(s_h,a_h)\given s_1=x] \geq 0$ with high probability. Combining Lemmas \ref{lem4.5} and \ref{lem4.6}, it holds with probability at least $1 - \xi/4$ that
\begin{equation}\label{4.14}
\text{(A3)}\leq \sum_{k=1}^K\sum_{h=1}^H\iota_h^k(s^k_h,a^k_h)\leq 2\sum_{k=1}^{K}\sum_{h=1}^H \Gamma_{h}^k\left(s_{h}^{k}, a_{h}^{k}\right) \le 2C ^\prime \sqrt{H^4d^3K}\cdot\log(HdK/\xi),
\end{equation}
with probability at least $1-\xi/4$.

Now, by plugging Lemma \ref{lem:performance_improve}, Lemma \ref{lem4.4}, and
\eqref{4.14} into the formulation of term (A) in \eqref{4.3}, we obtain that
\begin{equation}\label{4.15}
\begin{aligned}
\text{(A)}&\leq\sqrt{2 H^4\sqrt{d}K\log (\operatorname{vol}(\mathcal{A}))} +4\sqrt{H^3dK\log(8/\xi)}+ 2C ^\prime \sqrt{H^4d^3K}\cdot\log(HdK/\xi)\\
&\leq C_1 \sqrt{H^4d^3 K}\log(HdK/\xi),
\end{aligned}
\end{equation}
with probability at least $1-\xi/2$, where $C_1$ is an absolute constant.  

\vskip5pt
\noindent{\bf Upper bound of term (B) in \eqref{4.3}.} 
We decompose term (B) in \eqref{4.3} by the following lemma, which characterizes the regret occuring in the reward update.

\begin{lemma}\label{lem4.7}
For any $\mu=\{\mu_h\}^H_{h=1}\in S$, it holds that
\[
\begin{aligned}
\sum_{k=1}^K\big[L(\pi^k,\mu)-L(\pi^k,\mu^k)\big]&\leq \sum_{k=1}^K\sum_{h=1}^H\frac{1}{2\eta}\big(
\|\mu_h^k-\mu_h\|^2_2
-\|\mu_{h}^{k+1}-\mu_{h}\|^2_2
-\|\mu_{h}^{k+1}-\mu_{h}^k\|^2_2\big)\\
&\qquad+\sum_{k=1}^K\sum_{h=1}^H\big[(\mu_{h}^{k+1}-\mu_{h}^k)^{\top}\widehat{\nabla}_{\mu_{h}}L(\pi^k,\mu^k)\big]+ K \big[\tilde{J}(\pi^\rE,r^\mu)-J(\pi^\rE,r^\mu)\big] \\
&\qquad+\sum_{k=1}^K\sum_{h=1}^H\big[(\mu_{h}^k-\mu_h)^{\top}(\psi(s_{h}^k,a_h^k)-\nabla_{\mu_{h}} J(\pi^k,\mu^k))\big],
\end{aligned}
\]
for any $(k,h)\in [K]\times[H]$. Here $\widehat{\nabla}_{\mu_{h}}L(\pi,\mu^k)$ and $\tilde{J}(\pi^\rE,r^\mu)$ are defined in \eqref{eq:MC_estimation-G} and \eqref{MC:estimate}, respectively.
\end{lemma}
\begin{proof}
See Appendix \ref{A.6} for a detailed proof.
\end{proof}

Applying Lemma \ref{lem4.7}, we decompose term (B) in \eqref{4.3} into four terms as follows,
\begin{equation}
\begin{aligned}\label{4.16}
\text{(B)}&\leq \underbrace{\sup_{\mu\in S}\sum_{k=1}^K\sum_{h=1}^H\frac{1}{2\eta}\big(
\|\mu_h^k-\mu_h\|^2_2
-\|\mu_{h}^{k+1}-\mu_{h}\|^2_2
-\|\mu_{h}^{k+1}-\mu_{h}^k\|^2_2\big)}_{\text{(B1)}}\\
&\qquad +\underbrace{\sum_{k=1}^K\sum_{h=1}^H\big[(\mu_{h}^{k+1}-\mu_{h}^k)^{\top}\widehat{\nabla}_{\mu_{h}}L(\pi^k,\mu^k)}_{\text{(B2)}}\\
&\qquad + \underbrace{K\cdot\sup_{\mu\in S } \big[\tilde{J}(\pi^\rE,r^\mu)-J(\pi^\rE,r^\mu)\big] }_{\text{(B3)}}\\
&\qquad +\underbrace{\sup_{\mu\in S}\sum_{k=1}^K\sum_{h=1}^H\big[(\mu_{h}^k-\mu_h)^{\top}(\psi(s_{h}^k,a_h^k)-\nabla_{\mu_{h}} J(\pi^k,\mu^k))}_{\text{(B4)}}\big].
\end{aligned}
\end{equation}
We upper bound terms (B1), (B2), (B3), and (B4) as follows. 

By telescoping the summands in term (B1) of \eqref{4.16} with respect to $k\in[K]$, we have
\begin{equation}\label{4.18}
\begin{aligned}
\text{(B1)}
&=\sup_{\mu\in S}\frac{1}{2\eta}\Big[\sum_{h=1}^H
\big((\|\mu_{h}^1-\mu_{h}\|_2^2-\|\mu_{h}^{K+1}-\mu_{h}\|^2_2-\sum_{k=1}^K\|\mu_{h}^{K+1}-\mu_{h}^k\|^2_2\big)\Big]\\
&\leq \sup_{\mu\in S}\frac{1}{2\eta}\sum_{h=1}^{H}\|\mu_{h}^1-\mu_{h}\|^2_2\leq\frac{2}{\eta}Hd,
\end{aligned}
\end{equation}
where the last inequality relies on the fact that $\|\mu_h\|_2\le \sqrt{d}$ for all $\mu = \{\mu_h\}_{h=1}^H\in S$. This upper bounds term (B1). 

By the update process $\mu_h^{k+1}=\text{Proj}_{B} \{\mu_h^k+\eta\widehat{\nabla}_{\mu_{h}}L(\pi^k,\mu^k)\}$ in OGAP (Line \ref{line:on-proj} of Algorithm \ref{alg:gail-online}), we have
\begin{equation}\label{EA.55}
\|\mu_h^{k+1}-\mu_h^k\|_2\leq \|\eta \widehat{\nabla}_{\mu_{h}}L(\pi^k,\mu^k)\|_2.
\end{equation}
Then we upper bound (B2) in \eqref{4.16} as follows,
\begin{equation}\label{EA.56}
\begin{aligned}
(\text{B2})&=\sum_{k=1}^K\sum_{h=1}^H\big[(\mu_{h}^{k+1}-\mu_{h}^k)^{\top}\widehat{\nabla}_{\mu_{h}}L(\pi^k,\mu^k)\big] \\ &\leq \sum_{k=1}^K\sum_{h=1}^H
\|\mu_h^{k+1}-\mu_h^k\|_2 \cdot
\|\widehat{\nabla}_{\mu_{h} }L(\pi^k,\mu^k)\|_2 \\
&\leq \sum_{k=1}^K\sum_{h=1}^H \eta\cdot
\|\widehat{\nabla}_{\mu_{h} }L(\pi^k,\mu^k)\|^2_2\leq 4\eta HK,
\end{aligned}
\end{equation}
where the first inequality follows from Cauchy-Schwarz inequality, the second inequality
follows from \eqref{EA.55} and the last inequality follows from the fact that $\|\widehat{\nabla}_{\mu_{h}}L(\pi^k,\mu^k)\|\leq 2\|\psi(\cdot,\cdot)\|_2 \le 2$. This upper bounds term (B2). 

We upper bound term (B3) via the following lemma. 
\begin{lemma}
[Monte Carlo Estimation] Since reward function class $\cR$ is linear as defined in \eqref{eq:B-mu} and the estimator $\tilde{J}(\pi^{\mathrm{E}},r^\mu)$ constructed in \eqref{eq:MC_estimation-G}, it holds that
\[
\sup_{\mu\in S}\big|\tilde{J}(\pi^{\mathrm{E}},r^\mu)-J(\pi^{\mathrm{E}},r^\mu) \big|\le
4 \sqrt{H^3d^2/N_1}\log(6N_1/\xi),\]
with probability at least $1-\xi$.\label{lem:MC estimate}
\end{lemma}
\begin{proof}
See Appendix \ref{pf:lem:MC estimate} for detailed proof. 
\end{proof}

By Lemma \ref{lem:MC estimate},  it holds with probability at least $1-\xi/4$ that
\begin{equation}
\text{(B3)}\le 4 K \sqrt{H^3d^2/N_1}\log(24N_1/\xi),
\label{4.18}
\end{equation}
which upper bounds term (B3). 

To upper bound term (B4) in \eqref{4.16}, we introduce the following lemma. 
\begin{lemma}[Unbiased Estimation]\label{lem4.8}
 It holds that   
\begin{equation*}
\begin{aligned}
\sup_{\mu\in S}\sum_{k=1}^K\sum_{h=1}^H\big[(\mu_{h}^k-\mu_h)^{\top}(\psi(s_{h}^k,a_h^k)-\nabla_{\mu_{h}} J(\pi^k,\mu^k))\big]\leq 32\sqrt{H^3d^2K \log (9/\xi)},
\end{aligned}
\end{equation*}
with probability at least $1-\xi$.
\end{lemma}
\begin{proof}
See Appendix \ref{A.7} for a detailed proof.
\end{proof}

By Lemma \ref{lem4.8}, it holds with probability at least $1-\xi/4$ that
\begin{equation}\label{4.17}
\text{(B4)}\le\sup_{\mu\in S}\sum_{k=1}^K\sum_{h=1}^H\big[(\mu_{h}^k-\mu_h)^{\top}(\psi(s_{h}^k,a_h^k)-\nabla_{\mu_{h}} J(\pi^k,\mu^k))\leq 32\sqrt{H^3d^2K \log (36/\xi)},
\end{equation}
which upper bounds term (B4).

Plugging \eqref{4.18}, \eqref{EA.56}, \eqref{4.17}, and \eqref{4.18} into \eqref{4.16}, it holds with probability at least $1-\xi/2$ that
\begin{equation}\label{4.19}
\begin{aligned}
\text{(B)}&\leq 2\sqrt{H^3d^2 K} + 4 K\sqrt{HK} +32\sqrt{H^3d^2K\log(36/\xi)} + 4\sqrt{H^3d^2/N_1}\log(24N_1/\xi)\\
& \le 32\sqrt{H^3d^2K\log(36/\xi)}+ 4K\sqrt{H^3d^2/N_1}\log(24N_1/\xi),
\end{aligned}
\end{equation} 
where we recall that $\eta = 1/\sqrt{HK}$. 
Combining \eqref{4.3}, \eqref{4.15}, and \eqref{4.19}, we obtain that
\[
\begin{aligned}
\text{Regret}(K) \le& C_1 \sqrt{H^4 d^3K}\log(HdK/\xi)+32\sqrt{H^3d^2K\log(24/\xi)}+ 4\sqrt{H^3d^2/N_1}\log(36N_1/\xi)\\
\le& C_2 \big(H^2d^{3/2}K^{1/2} \log(HdK/\xi) + KH^{3/2}dN_1^{-1/2}\log(N_1/\xi)\big),
\end{aligned}
\]
with probability at least $1-\xi$, where $C_2$ is an absolute constant. 
This concludes the proof of Theorem \ref{thm:online}.
\end{proof}

\section{Proof Sketch: Analysis of PGAP}
\subsection{Proof of Theorem \ref{thm:gap}}
\label{pf.sketch-1}
\begin{proof}
By the property of mixed policy, we can rewrite the optimality gap as 
\begin{equation}
\mathbf{D_{\cR}(\pi^\mathrm{E},\widehat{\pi}})=\sup_{\mu \in S} \big[J(\pi^{\mathrm{E}},r^\mu)-J(\hat{\pi},r^\mu)\big]=\frac{1}{K}
\sup _{\mu \in S} \sum_{k=1}^K L(\pi^{k}, \mu),
\label{eq:000}
\end{equation}
where $L(\pi^k,\mu)=J  (\pi^{\mathrm{E}}, r^\mu )-J  (\pi^{k}, r^\mu )$.
Recall that $J(\pi^k,r^\mu)=V_{1,\pi^k}^{r^\mu}(x)$ and $\widehat{J}(\pi^k,r^k)=\widehat{V}_{1}^{k}(x)$, where $x$ is the initial state, we upper bound $
\mathbf{D_{\cR}(\pi^\mathrm{E},\widehat{\pi}})$ as follows,
\begin{equation}
\begin{aligned}
\mathbf{D_{\cR}(\pi^\mathrm{E},\widehat{\pi}})&\le
 \frac{1}{K}\Big\{\underbrace{\sum_{k=1}^{K}\big[ J(\pi^{\mathrm{E}},r^k)
-\widehat{J}(\pi^k,r^k) \big]}_{(\text{A})}\\
&\qquad +\underbrace{\sup_{\mu\in S}\sum_{k=1}^K\big[ J(\pi^{\mathrm{E}},r^\mu)- J(\pi^k,r^\mu)- J(\pi^{\mathrm{E}},r^k)+  \widehat{J}(\pi^k,r^k) \big]}_{(\text{B})}\Big\}.
\end{aligned}
\label{eq:gap-compos}
\end{equation}
We upper bound terms (A) and (B) in \eqref{eq:gap-compos}as follows, respectively. 

\vskip5pt
\noindent{\bf Upper Bound of Term (A) in \eqref{eq:gap-compos}.}
To upper bound term (A) in \eqref{eq:gap-compos}, we introduce the following lemma. 

\begin{lemma}[Extended Value Difference \citep{OPPO}] \label{lem:extend-V}
 Let $\pi=\{\pi_{h}\}_{h=1}^{H}$ and $\pi^{\prime}=\{\pi_{h}^{\prime}\}_{h=1}^{H}$ be any two policies and let $\{\widehat{Q}_{h}\}_{h=1}^{H}$ be any estimated Q-functions. For any $h \in[H]$, we define the estimated  V-function $\widehat{V}_{h}: \mathcal{S} \rightarrow \mathbb{R}$ by setting $\widehat{V}_{h}(x)=\big\langle\widehat{Q}_{h}(x, \cdot), \pi_{h}(\cdot \mid x)\big\rangle_{\mathcal{A}}$ for any $x \in \mathcal{S}$. For any initial state $x \in \mathcal{S}$, we have
$$
\begin{aligned}
\widehat{V}_{1}(x)-V_{1}^{\pi^{\prime}}(x)=& \sum_{h=1}^{H} \mathbb{E}_{\pi^{\prime}}\big[\langle\widehat{Q}_{h}(s_{h}, \cdot), \pi_{h}(\cdot \mid s_{h})-\pi_{h}^{\prime}(\cdot \mid s_{h})\rangle_{\mathcal{A}} \biggiven s_{1}=x\big] \\
&\qquad+\sum_{h=1}^{H} \mathbb{E}_{\pi^{\prime}}\big[\widehat{Q}_{h}(s_{h}, a_{h})-r(s_{h}, a_{h})-\mathcal{P}_{h} \widehat{V}_{h+1}(s_{h}, a_{h}) \biggiven s_{1}=x\big],
\end{aligned}
$$
where $\mathbb{E}_{\pi^{\prime}}$ is taken with respect to the trajectory generated by $\pi^{\prime}$ and $r$ is the reward function.
\end{lemma}
\begin{proof}
See Appendix B.1 in \cite{OPPO} for a detailed proof.
\end{proof}

For the simplicity of later discussion, at the $h$-th step of $k$-th episode,  we define the error for estimating Bellman equation in \eqref{eq:Bellman} in the policy evaluation stage of PGAP (Lines \ref{line:pp4}-\ref{line:pp2} of Algorithm \ref{alg:gail-offline}) with any reward function $r^\mu$ as follows,
$$ 
\iota_h^{k,r^\mu}(s,a) := r_h^\mu(s,a) + [\mathcal{P}_h\widehat{V}_{h+1}^{k,r^\mu}](s,a)-\widehat{Q}_{h}^{k,r^\mu}(s,a),
$$
for any $(s,a,h,\mu)\in\cS\times \cA \times [H] \times S$, where $\widehat{Q}_{h}^{k,r^\mu}$ and $\widehat{V}_{h+1}^{k,r^\mu}$ are defined as 
\begin{equation}
\begin{aligned}
\widehat{V}_{H+1}^{k,r^\mu}(\cdot) &= 0,\\
\widehat{Q}_{h}^{k,r^\mu}(\cdot, \cdot)
&=\max\big\{({r}^\mu_{h}+\widehat{\mathcal{P}}_{h} \widehat{V}_{h+1}^{k,r^\mu}-\Gamma_{h})(\cdot, \cdot), 0\big\},\\
\widehat{V}_{h}^{k,r^\mu}(\cdot, \cdot)
&=\big\langle\widehat{Q}_{h}^{k,r^\mu}(\cdot, \cdot) ,\pi_{h}^{k}(\cdot \mid \cdot)\big\rangle_{\mathcal{A}},\quad \text{for } h\in[H].
\label{eq:def:iota-on}
\end{aligned}
\end{equation}
It implies that $\widehat{Q}_h^k = \widehat{Q}_h^{k,r^k}$ and $\widehat{V}_h^k=\widehat{V}_h^{k,r^k}$, where $\widehat{Q}_h^k$ and $\widehat{V}_h^k$ are construted in the policy evaluation stage in PGAP (Lines \ref{line:pp4}-\ref{line:pp2} of Algorithm \ref{alg:gail-offline}).

By applying Lemma \ref{lem:extend-V}, we decompose term (A) in \eqref{eq:gap-compos} as follows,
\begin{equation}
\begin{aligned}
(\mathrm{A})&=
\underbrace{\sum_{k=1}^{K}\sum_{h=1}^{H} \mathbb{E}_{\pi^{\mathrm{E}}}  \big[  \langle \widehat{Q}_{h}^{k} (s_{h}, \cdot ), \pi_{h}^{\mathrm{E}}  (\cdot \mid s_{h} )-\pi_{h}^{k}  (\cdot \mid s_{h} ) \rangle_{\mathcal{A}} \big| s_{1}=x \big] }_{\text{(A1)}} \\ 
&\qquad +\underbrace{  \sum_{k=1}^{K}\sum_{h=1}^{H}\mathbb{E}_{\pi^{\mathrm{E}}}  [\iota_h^{k,r^k}(s_h,a_h)|s_1=x \big]}_{\text{(A2)}} .
\end{aligned}
\label{sketch:eq:---1}
\end{equation}
We upper bound terms (A1) and (A2) in \eqref{sketch:eq:---1} as follows. 

By Lemma \ref{lem:performance_improve}, we have
\#\label{eq:8888}
\text{(A1)} \le  \sqrt{2 H^4d\sqrt{d}K\log (\operatorname{vol}(\mathcal{A}))}.
\#
We upper bound term (A2) in \eqref{sketch:eq:---1} using the following lemma. 
\begin{lemma}[Pessimism]
\label{lem:pess}
If $\{\Gamma_h\}_{h\in[H]}$ are $\xi$-uncertainty qualifiers defined in Definition \ref{def:xi-uncertainty}, when conditioned on $\mathcal{E}$ defined in Definition \ref{def:xi-uncertainty}, it holds for any $(s, a,h) \in \mathcal{S} \times \mathcal{A}\times[H]$ and any reward function $r^\mu$ with $\mu\in S$ that
$$
0 \leq \iota_{h}^{k,r^\mu}(s, a) \leq 2 \Gamma_{h}(s, a).
$$
\end{lemma} 
\begin{proof}
See  Appendix \ref{pf:lem:pess} for a detailed proof.
\end{proof}
By Lemma \ref{lem:pess}, conditioned on $\mathcal{E}$, we upper bound term (A2) in \eqref{sketch:eq:---1} as follows,
\#\label{eq:a-2}
\mathrm{(A2)} \le  \sum_{k=1}^{K}\sum_{h=1}^{H}\mathbb{E}_{\pi^{\mathrm{E}}}  \big[2\Gamma_h(s_h,a_h) \biggiven s_1 =x \big]=K\cdot\text{IntUncert}_{\mathbb{D}^{\mathrm{A}}}^{\pi^\mathrm{E}} ,
\#
where we denote by $\text{IntUncert}_{\mathbb{D}^{\mathrm{A}}}^{\pi^\mathrm{E}}  = \sum_{h=1}^{H}\mathbb{E}_{\pi^{\mathrm{E}}}  [2\Gamma_h(s_h,a_h) \biggiven s_1 =x \big]$ for notational convenience.
Combining \eqref{eq:8888} and \eqref{eq:a-2}, we derive an upper bound on term (A) in \eqref{eq:gap-compos} conditioned on $\mathcal{E}$ as 
\#\label{eq:result-A}
\mathrm{(A)}\le  \sqrt{2 H^4\sqrt{d}\log (\operatorname{vol}(\mathcal{A}))}+ K\cdot\text {IntUncer}_{\mathbb{D}}^{\pi^{\mathrm{E}}}.
\#

\vskip5pt
\noindent{\bf Upper Bound of Term (B) in \eqref{eq:gap-compos}.}
We $\widehat{L}(\pi,\mu)$ as follows,
\[
\widehat{L}(\pi,\mu)=\tilde{J}(\pi^{\mathrm{E}},r^{\mu})-\widehat{J}(\pi^k,r^{\mu}). 
\]
Here $\tilde{J}(\pi,r^\mu)$ is the MC estimation defined in \eqref{MC:estimate} and $\widehat{J}(\pi^k,r^\mu)$ is the estimated cumulative reward defined as $
\widehat{J}(\pi^k,r^\mu) = \widehat{V}_1^{k,r^\mu}(x)
$, where $\widehat{V}_1^{k,r^\mu}$ is defined in \eqref{eq:def:iota-on} and $x$ is the initial state.
By this, we upper bound term (B) in \eqref{eq:gap-compos} as
\begin{equation}
\begin{aligned}
\text{(B)}=&\sup_{\mu\in S}\sum_{k=1}^K  \big[ J(\pi^{\mathrm{E}},r^\mu)- J(\pi^k,r^\mu)- J(\pi^{\mathrm{E}},r^k)+ \widehat{J}(\pi^k,r^k) \big]\\
\leq & \underbrace{\sup_{\mu\in S}\sum_{k=1}^K \big[J(\pi^{\mathrm{E}},r^\mu)-\Tilde{J}(\pi^{\mathrm{E}},r^\mu)\big]}_{\text{(B1)}} + \underbrace{\sup_{\mu\in S}\sum_{k=1}^K \big[-J(\pi^{\mathrm{E}},r^k)+\Tilde{J}(\pi^{\mathrm{E}},r^k)\big]}_{\text{(B2)}}\\
& \qquad + \underbrace{\sup_{\mu\in S}\sum_{k=1}^K \big[\widehat{J}(\pi^k,r^\mu)-J(\pi^k,r^\mu)\big]}_{\text{(B3)}} 
+ \underbrace{\sup_{\mu\in S}\sum_{k=1}^K \big[\widehat{L}(\pi^k,\mu)-\widehat{L}(\pi^k,\mu^k)\big]}_{\text{(B4)}}. \label{eq:decomp-B}
\end{aligned}
\end{equation}
We upper bounds terms (B1)--(B4)  in \eqref{eq:decomp-B} as follows. 

By applying Lemma \ref{lem:MC estimate} on term (B1) and term (B2) in \eqref{eq:decomp-B}, it holds with probability at least $1-\xi/2$ that  
\begin{equation}
\begin{aligned}
\text{(B1) + (B2)} \le 8 K\sqrt{H^3d^2/N_1}\log(24N_1/\xi).
\end{aligned}
\label{eq:mc-estiamte-1}
\end{equation}
To upper bound term (B3) in \eqref{eq:decomp-B}, we invoke Lemmas \ref{lem:extend-V} and \ref{lem:pess}, which imply 
\begin{equation}
\text{(B3)}=  \sup_{\mu\in S}\Big[\sum_{k=1}^{K} \sum_{h=1}^{H}\mathbb{E}_{\pi^{k}}  \big[-\iota_{h}^{k,r^\mu}  (s_{h}^{k}, a_{h}^{k} ) \biggiven s_{1}=x \big] \Big]\le 0.\label{eq:pess-1}
\end{equation}
To upper bound term (B4) in \eqref{eq:decomp-B}, we first introduce the following lemma. 
\begin{lemma}\label{lem:concave}
The function $\widehat{L}(\pi^k,\mu)$ defined in \eqref{eq:def-Lhat} is concave in $\mu_h$ for any $h\in[H]$, where $\mu = \{\mu_h\}_{h=1}^H\in S$ and $\widehat{L}(\pi^k,\mu)$.
\end{lemma}
\begin{proof}
See Appendix \ref{pf:lem:concave} for a detailed proof.
\end{proof}

By Lemma \ref{lem:concave}, we establish the following lemma to upper bound term (B4) in \eqref{eq:decomp-B}, which corresponds to the reward update stage in PGAP (Lines \ref{line:p-r}--\ref{line:p-r2} of Algorithm \ref{alg:gail-offline}).
\begin{lemma}
\label{lem:telescope-reward-update]}
For any $\mu\in S$, it holds that 
$$
\begin{aligned}
\sum_{k=1}^K  \big[\widehat{L}(\pi^k,\mu)-\widehat{L}(\pi^k,\mu^k)\big] &\le\sum_{k=1}^K \sum_{h=1}^H \big[\frac{1}{2\eta}\|\mu_h^{k+1}-\mu_h\|_2^2+\frac{1}{2\eta}\|\mu_h^{k+1}-\mu_h\|_2^2 \\
&\qquad - \frac{1}{2\eta}\|\mu_h^{k+1}-\mu^{k}_h\|_2^2+\eta \|\nabla_{\mu_h}\widehat{L}(\pi^k,\mu^k)\|_2^2\big],
\end{aligned}
$$
\end{lemma}
\begin{proof}
See Appendix \ref{pf:telescope-reward-update]} for a detailed proof.
\end{proof}
By Lemma \ref{lem:telescope-reward-update]}, we have 
\$
\text{(B4)}&\le 
\sup_{\mu\in S}\Big[\sum_{h=1}^H\big( \frac{1}{2\eta }\|\mu_h^{1} - \mu_h\|_2^2- \frac{1}{2\eta }\|\mu_h^{K+1}-\mu_h\|_2^2 -\frac{1}{2\eta} \sum_{k=1}^K \|\mu_h^{k} - \mu_h^{k+1}\|_2^2\\
&\qquad + \sum_{k=1}^K\eta\|{\nabla_{\mu_h}}\widehat{L}(\pi^k,\mu^k)\|_2^2\big)\Big],
\$
which implies that
\#
\begin{aligned}
\text{(B4)}&\le 
\sup_{\mu\in S}\sum_{h=1}^H\big[ \frac{1}{2\eta }\|\mu_h^{1} - \mu_h\|_2^2\big]+ \sup_{\mu\in S}\sum_{h=1}^H\sum_{k=1}^K\big[\eta\|{\nabla_{\mu_h}}\widehat{L}(\pi^k,\mu^k)\|_2^2\big].
\end{aligned}
\label{eq:tele-2.5}
 \#
Based on \eqref{eq:tele-2.5}, we upper bound $\|{\nabla_{\mu_h}}\widehat{L}(\pi^k,r^\mu)\|_2^2$ for any $\mu\in S$ and $h\in[H]$ as follows, 
\#
\begin{aligned}
\|{\nabla_{\mu_h}}\widehat{L}(\pi^k,r^\mu)\|_2^2&\le \|\nabla_{\mu_h}\Tilde{J}(\pi^\rE,r^\mu)-\nabla_{\mu_h}\widehat{J}(\pi^k,r^\mu)\|_2^2 \\
&\le 2\|\nabla_{\mu_h}\Tilde{J}(\pi^\rE,r^\mu)\|_2^2+2\|\nabla_{\mu_h}\widehat{J}(\pi^k,r^\mu)\|_2^2. 
\end{aligned}
\label{eq:tele-3.1}
\#
Recall that $\nabla_{\mu_h}J(\pi^\rE,r^\mu) = {N_1}^{-1}\sum_{\tau=1}^{N_1} \psi(s_{h,\tau}^\rE,a_{h,\tau}^\rE)$ and $\nabla_{\mu_h}\widehat{J}(\pi^k,r^\mu)$ is characterized in Proposition \ref{prop:grad}, then we have that 
\#
\|\nabla_{\mu_h}\Tilde{J}(\pi^\rE,r^\mu)\|_2^2+\|\nabla_{\mu_h}\widehat{J}(\pi^k,r^\mu)\|_2^2\le 2\|\psi(\cdot,\cdot)\|_2^2. 
\label{eq:tele-5}
\#
 Since it holds that $\|\mu_h\|_2\le \sqrt{d}$ for any $\mu \in S$ and $\|\psi(\cdot,\cdot)\|_2\le 1$, it also yields that 
\#
\|\mu_h^\prime - \mu_h\|_2^2\le (\|\mu_h^\prime \|_2+\| \mu_h\|_2)^2\le 4d 
\label{eq:mu-b}
\#
for any $\mu,\mu^\prime \in S$. By setting $\eta = 1/\sqrt{KH}$ and combining  \eqref{eq:tele-2.5}, \eqref{eq:tele-3.1}, \eqref{eq:tele-5}, and \eqref{eq:mu-b}, we attain that 
\#
\begin{aligned}
\text{(B4)}\le &
H\cdot \frac{\sqrt{HK}}{2} \cdot 4d+HK\cdot \frac{1}{\sqrt{HK}}\cdot 2\cdot 4d\\
\le & 2H^{3/2}d K^{1/2} + 8H^{1/2}dK^{1/2}\le 8H^{3/2}d K^{1/2}.
\end{aligned}
\label{eq:b4}
\#
Plugging \eqref{eq:mc-estiamte-1}, \eqref{eq:pess-1}, and \eqref{eq:b4} into \eqref{eq:decomp-B}, conditioned on $\mathcal{E}$, it holds with probability at least $1-\xi/2$ that
\begin{equation}
\begin{aligned}
\text{(B)}&\le 
8 K\sqrt{H^3d^2/N_1}\log(24N_1/\xi) +8dH^{3/2} K^{1/2}.
\end{aligned}\label{eq:result-B}
\end{equation}
Recall that Definition \ref{def:xi-uncertainty} implies $\mathbb{P}_\D(\mathcal{E})>1-\xi/2$.
Combining \eqref{eq:000}, \eqref{eq:result-A}, and \eqref{eq:result-B}, with probability at least $1 - \xi$, we have
\$
\mathbf{D}_{\cR}(\pi^\rE,\widehat{\pi})
&\le 
\sqrt{2 H^4\sqrt{d}\log (\operatorname{vol}(\mathcal{A}))/K}+ \text { IntUncert}_{\mathbb{D}}^{\pi^{\mathrm{E}}}
\\
&\qquad +8 \sqrt{H^3d^2/N_1}\log(24N_1/\xi) +8H^{3/2}d K^{-1/2}\\
&\le 8 H^2dK^{-1/2}+\text { IntUncert}_{\mathbb{D}}^{\pi^{\mathrm{E}}} + 8 H^{3/2}dN_1^{-1/2}\log(24N_1/\xi). 
\$
Thus, we conclude the proof of Theorem \ref{thm:gap}.
\end{proof}

\subsection{Proof of Proposition \ref{prop:minimax}}
\label{sec:minimax}
\begin{proof}
  Our proof is based on the following theorem that presents the information-theoretic lower bound.
  \begin{theorem}[Information-Theoretic Lower Bound, Theorem $4.6$ in \cite{pessimismRl-Jinchi}]
    For the output $\texttt{Algo}(\mathbb{\overline{D}})$ of any offline RL algorithm based on the dataset $\overline{\D}$, there exists a tabular MDP $\mathcal{M}\ (\cS,\cA,H,\cP,r)$ with initial state $x \in \mathcal{S}$ and a dataset $\mathbb{\overline{D}}$ which is compliant with $\mathcal{M}$, such that
  $$
  \mathbb{E}_{\mathbb{\overline{D}}}\left[\frac{J(\pi^\star) - J(\texttt{Algo}(\mathbb{\overline{D}}))}{\sum_{h=1}^{H} \mathbb{E}_{\pi^{\star}}\big[1 / \sqrt{1+n_{h}\left(s_{h}, a_{h}\right)} \biggiven s_{1}=x\big]}\right] \geq c
  $$
  where $c$ is an absolute constant,$\pi^\star$ is the optimal policy statisfying that $\pi^\star:= \operatorname{argmax}\limits_{\pi\in\Delta{(\cA\mid \cS,H)}} J(\pi)$, and $n_{h}\left(s_{h}, a_{h}\right)=\sum_{\tau=1}^{N} \mathbf{1}\left\{s_{h}^{\tau}=s_{h}, a_{h}^{\tau}=a_{h}\right\}$ for $\left(s_{h}, a_{h}\right) \in \mathcal{S} \times \mathcal{A}$. \label{thm:jin:minimax}
  \end{theorem}

First we show the linear kernel MDP defined in Assumption \ref{def:linear} can be reduced to tabular MDP. If we set $d = |\mathcal{S} |^2|\cA|$ and take the feature map as the canonical basis 
$$
\phi(s,a,s^\prime) = \mathbf{e}_{(s,a,s^\prime)},\quad \psi(s,a) = \mathbf{e}_{(s,a)}, \quad \text{for all }(s,a,s^\prime)\in\cS\times\cA\times\cS,
$$
then Assumption \ref{def:linear} is satisfied with $R=1$. 

Applying Theorem \ref{thm:jin:minimax}, we derive a hard instance in tabular MDP with known reward $r$. When we choose the reward set $\cR$ as the singleton reward $r$, it yields that 
\#
    \max_{\pi^\mathrm{E}\in\Delta(\cA\mid\cS,H)}\mathbb{E}_{\mathbb{\overline{D}}}\bigg[\frac{\mathbf{D}_{\cR}(\pi^\mathrm{E},\texttt{Algo}(\mathbb{\overline{D}}))}{ \text{Information}_{\mathbb{\overline{D}} }^{\pi^\mathrm{E}}  }\bigg]\ge \mathbb{E}_{\mathbb{\overline{D}}}\bigg[\frac{\mathbf{D}_{\cR}(\pi^\star,\texttt{Algo}(\mathbb{\overline{D}}))}{ \text{Information}_{\mathbb{\overline{D}} }^{\pi^\star}  }\bigg]=\mathbb{E}_{\mathbb{\overline{D}}}\bigg[\frac{J(\pi^\star) - J(\texttt{Algo}(\mathbb{\overline{D}}))}{ \text{Information}_{\mathbb{\overline{D}} }^{\pi^\star}  }\bigg],\label{prop:eq:minimax-1}
\#
where the last equality is originated from the definition of optimality gap in \eqref{eq:def-gap}.

Next we handle $\text{IntUncert}_{\mathbb{\overline{D}} }^{\pi^\star}$, which takes the form as
\#
\begin{aligned}
  \text{Information}_{\mathbb{\overline{D}}}^{\pi^\star} &=(\text{Vol}(\cS))^{-1}\cdot\mathbb{E}_{\pi^\star}\Big[\sum_{h=1}^H \int_{\cS}\|\phi(s_h,a_h,s^\prime)\|_{\Lambda_h^{-1}}\mathrm{d}s^\prime \biggiven s_1=x\Big]\\
  \text {\quad where } \Lambda_{h} &=\lambda I +\sum_{\tau=1}^{N} \int_{\mathcal{S}} \phi  (s_{h}^{\tau}, a_{h}^{\tau}, s^{\prime} ) \phi  (s_{h}^{\tau}, a_{h}^{\tau}, s^{\prime} )^{\top} \mathrm{d} s^{\prime},
\end{aligned} \label{prop:eq:minimax-2}
\#
where $N$ is the number of the trajectories in the dataset $\overline{\D}$. In the tabular setting, we obtain that
$$
\sum_{\tau=1}^{N}  \phi  (s_{h}^{\tau}, a_{h}^{\tau}, s^{\prime} ) \phi  (s_{h}^{\tau}, a_{h}^{\tau}, s^{\prime} )^{\top}  = \sum_{(s,a)\in\cS\times\cA}n_h(s,a)W_{(s,a,s^\prime)},
$$
where $W_{(s,a,s^\prime)}$ is a symmetric matrix whose non-zero entry is at $((s,a,s^\prime)(s,a,s^\prime))$ and equals to $1$. Summing $s^\prime$ over $\cS$, in the tabular case we have that 
$$
\Lambda_h = \lambda I +\sum_{\tau=1}^{N} \sum_{s^\prime \in \cS}(s_{h}^{\tau}, a_{h}^{\tau}, s^{\prime} ) \phi  (s_{h}^{\tau}, a_{h}^{\tau}, s^{\prime} )^{\top} \mathrm{d} s^{\prime} = \lambda I + \sum_{(s,a,s^\prime)\in\cS\times\cA\times \cS}n_h(s,a)W_{(s,a,s^\prime)}.
$$
Choosing $\lambda = 1$, we obtain that 
$$
\phi\left(s, a, s^{\prime}\right)^{\top} \Lambda_{h}^{-1} \phi\left(s, a, s^{\prime}\right)=\frac{1}{1+ n_{h}(s, a)},
$$
for all $(s,a,s^\prime)\in\cS\times \cA\times \cS$. Hence we have that 
\#
\sum_{h=1}^H \sum_{s^\prime \in \cS}\|\phi(s,a,s^\prime)\|_{\Lambda_h^{-1}} = \sum_{h=1}^H\frac{|\cS|}{\sqrt{1+n_h(s,a)}},
\label{prop:eq:minimax-3}
\#
for all $(s,a)\in\cS\times \cA$.
Taking expectation on \eqref{prop:eq:minimax-3} with respect to the optimal policy $\pi^\star$ and according to \eqref{prop:eq:minimax-2}, it holds that
\#
\text{Information}_{\mathbb{\overline{D}}}^{\pi^\star} =|\cS|^{-1}\cdot \sum_{h=1}^{H} \mathbb{E}_{\pi^{\star}}\big[1 / \sqrt{1+n_{h}\left(s_{h}, a_{h}\right)} \biggiven s_{1}=x\big]\label{prop:minimax-4}.
\#
Plugging \eqref{prop:minimax-4} into \eqref{prop:eq:minimax-1}, under the hard instance in Theorem \ref{thm:jin:minimax}, we obtain that 
$$
\max_{\pi^\mathrm{E}\in\Delta(\cA\mid\cS,H)}\mathbb{E}_{\mathbb{\overline{D}}}\bigg[\frac{\mathbf{D}_{\cR}(\pi^\mathrm{E},\texttt{Algo}(\mathbb{\overline{D}}))}{ \text{Information}_{\mathbb{\overline{D}} }^{\pi^\mathrm{E}}  }\bigg] \ge c ,
$$
where $c>0$ is a positive constant. Then we conclude the proof of Proposition \ref{prop:minimax}

\end{proof}

\subsection{Proof of Corollary \ref{cor:bound-assumption}}
\label{pf:cor}
\begin{proof}
By the property of trace, we have that
\begin{equation}
\begin{aligned}
 &\mathbb{E}_{\pi^{\mathrm{E}}}\Big[\int_\cS \sqrt{ (\phi  (s_{h}, a_{h},s^\prime )^{\top}  \Lambda_{h}^{-1} \phi  (s_{h}, a_{h} ,s^\prime) } \mathrm{d}s^\prime\Big|\, s_{1}=x \Big] \\
&\qquad=\mathbb{E}_{\pi^{\mathrm{E}}}  \Big[\int_{\cS} \sqrt{\operatorname{Tr}  (\phi  (s_{h}, a_{h},s' )^{\top} \Lambda_{h}^{-1} \phi  (s_{h}, a_{h},s' ) )} \mathrm{d}s^\prime \Big|\, s_{1}=x \Big] \\
&\qquad=\mathbb{E}_{\pi^{\mathrm{E}}} \Big [\int_{\cS}\sqrt{\operatorname{Tr}  (\phi  (s_{h}, a_{h},s' ) \phi  (s_{h}, a_{h},s' )^{\top} \Lambda_{h}^{-1} )}  \mathrm{d}s^\prime\Big|\, s_{1}=x \Big].
\end{aligned}
\label{eq:cor:pf----0.1}
\end{equation}
Applying Cauchy-Schwarz inequality on \eqref{eq:cor:pf----0.1}, we derive that
\begin{equation}
\begin{aligned}
&\mathbb{E}_{\pi^{\mathrm{E}}} \Big [\int_{\cS}\sqrt{\operatorname{Tr}  (\phi  (s_{h}, a_{h},s' ) \phi  (s_{h}, a_{h},s' )^{\top} \Lambda_{h}^{-1} )}  \mathrm{d}s^\prime\Big|\, s_{1}=x \Big] \\
& \qquad\leq \big(\text{Vol}(\cS)\big)^{1/2}\cdot\Big(\mathbb{E}_{\pi^{\mathrm{E}}} \Big [\int_{\cS} \operatorname{Tr}  (\phi  (s_{h}, a_{h},s' ) \phi  (s_{h}, a_{h},s' )^{\top} \Lambda_{h}^{-1} ) \mathrm{d}s^\prime \Big|\, s_{1}=x \Big]\Big)^{1/2}
\end{aligned}
\label{eq:cor:pf----1}
\end{equation}
for any $x,s^\prime \in \mathcal{S}$ and all $h \in[H] $. 

To utilize Theorem \ref{thm:gap}, we define the event $\mathcal{E}^{\sharp}$ as follows, 
\begin{equation*}
\begin{aligned}
\mathcal{E}^{\sharp} = \Big \{\mathbf{D_{\cR}(\pi^\mathrm{E},\widehat{\pi}}) \leq \cO\big(H^2dK^{-1/2}) +\delta_{N_1}+\text { IntUncert}_{\mathbb{D}}^{\pi^{\mathrm{E}}}\Big\}, 
\end{aligned}
\label{eq:cor:pf===def on E}
\end{equation*}
$\text{where}\text { IntUncert}_{\mathbb{D}}^{\pi^{\mathrm{E}}} = 2 \sum_{h=1}^{H}\mathbb{E}_{\pi^{\mathrm{E}}}  [\Gamma_{h}  (s_{h}, a_{h} ) \mid s_1 =x ]$ and $\Gamma_h$ is defined in \eqref{eq:xi-uncertainty}.
Conditioned on the event $\mathcal{E}^{\sharp} \cap \mathcal{E}^{\dagger},$ where $\mathcal{E}^{\dagger}$ is defined in Assumption \ref{assumption:data}, we obtain that
\begin{equation}
 \begin{aligned}
\text{IntUncert}_{\mathbb{D}^{\mathrm{A}}}^{\pi^\mathrm{E}}  
& \leq 2 \kappa H\sqrt{d}\sum_{h=1}^{H} \mathbb{E}_{\pi^{\mathrm{E}}} \Big[\int_{\cS} \|\phi(s_h,a_h,s')\|_{\Lambda_h^{-1}} \ \mathrm{d}s^\prime\biggiven s_1 = x\Big] .
\label{eq:cor:--1}
\end{aligned}
\end{equation}
By plugging \eqref{eq:cor:pf----0.1} and \eqref{eq:cor:pf----1} into \eqref{eq:cor:--1}, we have
\#
\begin{aligned}
\text{IntUncert}_{\mathbb{D}^{\mathrm{A}}}^{\pi^\mathrm{E}} 
& \le 2\kappa H\sqrt{d}\sqrt{\text{Vol}(\cS)}\cdot \sum_{h=1}^H  \sqrt{\mathrm{Tr}  \Big(\mathbb{E}_{\pi^{\mathrm{E}}}  \Big[\int_{\mathcal{S}}\phi(s_h,a_h,s')\phi(s_h,a_h,s')^\top \mathrm{d}s^\prime \mid s_1=x\Big ]\ \Lambda_h^{-1} \Big) },
\label{eq:int-err-2}
\end{aligned}
\#
where $\text{Vol}(\mathcal{S})$ is the finite measure of the state space $\mathcal{S}$.
For notational simplicity, we define
\begin{equation}
\Sigma_{h}(x)=\mathbb{E}_{\pi^{\mathrm{E}}}  \Big[\int_{\cS}\phi  (s_{h}, a_{h},s' ) \phi  (s_{h}, a_{h},s' )^{\top} \mathrm{d}s^\prime\mid s_{1}=x \Big],
\label{def:Sigma_h}
\end{equation}
for all $x \in \mathcal{S}$ and all $h \in[H] .$

By Assumption \ref{assumption:data} and the definition of $\Sigma_h(x)$ in \eqref{def:Sigma_h}, we know that the matrix $( I +c^\dagger N_2 \Sigma_h(x)\big)^{-1}-\Lambda_h^{-1}$ is positive definite conditioned on $\mathcal{E}^{\dagger}$. Combining \eqref{eq:int-err-2} and \eqref{def:Sigma_h}, we obtain that
\#
\text{IntUncert}_{\mathbb{D}^{\mathrm{A}}}^{\pi^\mathrm{E}}  & \leq 2 \big(\text{Vol}(\cS)\big)^{1/2} \cdot\kappa H\sqrt{d}\sum_{h=1}^{H} \sqrt{\operatorname{Tr}  \big(\Sigma_{h}(x) \cdot  (I+c^{\dagger} \cdot N_2 \cdot \Sigma_{h}(x) )^{-1} \big)} \notag \\
&=2\big(\text{Vol}(\cS)\big)^{1/2} \cdot\kappa H\sqrt{d}\sum_{h=1}^{H} \sqrt{\sum_{j=1}^{d} \frac{\lambda_{h, j}(x)}{1+c^{\dagger} \cdot N_2 \cdot \lambda_{h, j}(x)}}. \label{eq:878700}
\#
Here $  \{\lambda_{h, j}(x) \}_{j=1}^{d}$ are the eigenvalues of $\Sigma_{h}(x)$ for any $x \in \mathcal{S}$ and $h \in[H].$  Meanwhile, under Assumption \ref{def:linear}, we have $\|\phi(\cdot, \cdot,\cdot)\|_2 \leq d R$, which is shown in \eqref{phi-2}. By applying Cauchy-Schwarz inequality to \eqref{def:Sigma_h}, it holds that
$$
  \|\Sigma_{h}(x) \|_{\text {op }} \leq \mathbb{E}_{\pi^{\mathrm{E}}}  \Big[  \Big\|\int_\cS \phi  (s_{h}, a_{h},s' ) \phi  (s_{h}, a_{h},s' )^{\top} \mathrm{d}s^\prime\Big\|_{\text {op }} \Big|\, s_{1}=x \Big] \leq d^{3/2}R^2\cdot(\text{Vol}(\cS))^{1/2},
$$
for any $x \in \mathcal{S}$ and $h \in[H]$, where $\|\cdot\|_{\mathrm{op}}$ is the operator norm. As $\Sigma_{h}(x)$ is positive semidefinite, we have $\lambda_{h, j}(x) \in[0,\|\Sigma_{h}(x) \|_{\text {op }}]$ for any $x \in \mathcal{S}$ and $(h,j) \in[H]\times[d]$.  Hence, conditioned on $\mathcal{E}^{\dagger} \cap \mathcal{E}^{\dagger}$, combining \eqref{eq:878700}, it holds for any $x \in \mathcal{S}$ that 
\begin{equation}
\begin{aligned}
\text{IntUncert}_{\mathbb{D}^{\mathrm{A}}}^{\pi^\mathrm{E}}   & \leq 2 \big(\text{Vol}(\cS)\big)^{1/2}\cdot \kappa H\sqrt{d}\sum_{h=1}^{H} \sqrt{\sum_{j=1}^{d} \frac{\lambda_{h, j}(x)}{1+c^{\dagger} \cdot N_2 \cdot \lambda_{h, j}(x)}} \\
& \leq 2 \big(\text{Vol}(\cS)\big)^{1/2}\cdot \kappa H\sqrt{d}\sum_{h=1}^{H} \sqrt{\sum_{j=1}^{d} \frac{1}{ c^{\dagger} \cdot N_2}} \\
&\leq c^{\prime} \big(\text{Vol}(\cS)\big)^{1/2} d^{3/2} H^{2} N_2^{-1 / 2} \log(HdN_2/\xi),
\end{aligned}
\label{eq:cor:final result}
\end{equation}
where the second inequality follows from the fact that $\lambda_{h, j}(x) \in[0,\|\Sigma_{h}(x) \|_{\text {op }}]$ for any $(x,h,j) \in \mathcal{S} \times [H]\times [d]$, while the third inequality follows from the choice of the scaling parameter $\kappa>0$ stated in Theorem \ref{thm:gap}. Here $c^\prime$ is an absolute constant dependent on $c$ and $c^{\dagger}$. By the condition in Corollary \ref{cor:bound-assumption}, we have $\mathbb{P}_{\D}  (\mathcal{E}^{\dagger} ) \geq 1-\xi / 2 .$ Also, by Theorem \ref{thm:gap}, we have $\mathbb{P}_{\D} (\mathcal{E}^{\sharp} ) \geq 1-\xi / 2$. Hence, by the union bound, we derive that $\mathbb{P}_{\D}  (\mathcal{E}^{\dagger} \cap \mathcal{E}^{\sharp} ) \geq 1-\xi$. Combining 
\eqref{eq:cor:final result},  we finish the proof of Corollary \ref{cor:bound-assumption}.
\end{proof}

\section{Conclusion}	
In this paper, we study provably efficient algorithms for GAIL in the online and offline setting with linear function approximation, where both the transition kernel and reward functions are linear. We present a unified framework and specialize it as optimistic generative adversarial policy optimization (OGAP) for online GAIL and pessimistic generative adversarial policy optimization for offline GAIL. With linear function approximation, we derive the upper bound of the regret of OGAP as $\tO(H^2 d^{3/2}K^{1/2}+KH^{3/2}dN_1^{-1/2})$ and the decomposition of optimality gap of PGAP, without any assumption on the additional dataset. When facilitated with additional dataset with sufficient coverage, we demonstrate that PGAP also has global convergence, achieving $\tO(H^{2}dK^{-1/2} +H^2d^{3/2}N_2^{-1/2}+H^{3/2}dN_1^{-1/2} \ )$ optimality gap. However, provably efficient GAIL with general function approximation (both the transition kernel and reward function) still remains an open problem, which is a challenging but important future direction.

\section*{Acknowledgments}
The authors would like to thank Miao Lu for helpful and valuable discussions.

\bibliographystyle{ims}
\bibliography{reference}

\begin{thebibliography}{82}
\expandafter\ifx\csname natexlab\endcsname\relax\def\natexlab#1{#1}\fi
\expandafter\ifx\csname url\endcsname\relax
  \def\url#1{\texttt{#1}}\fi
\expandafter\ifx\csname urlprefix\endcsname\relax\def\urlprefix{}\fi

\bibitem[{Abbasi-Yadkori et~al.(2011)Abbasi-Yadkori, P{\'a}l and
  Szepesv{\'a}ri}]{elliptical-potential}
\text{Abbasi-Yadkori, Y.}, \text{P{\'a}l, D.} and \text{Szepesv{\'a}ri, C.}
  (2011).
\newblock Improved algorithms for linear stochastic bandits.
\newblock In \textit{International Conference on Machine Learning}.

\bibitem[{Abbeel and Ng(2004)}]{2004-AL-using-IRL}
\text{Abbeel, P.} and \text{Ng, A.} (2004).
\newblock Apprenticeship learning via inverse reinforcement learning.
\newblock \textit{International Conference on Machine Learning}.

\bibitem[{Antos et~al.(2007)Antos, Munos and Szepesv{\'a}ri}]{uniform3}
\text{Antos, A.}, \text{Munos, R.} and \text{Szepesv{\'a}ri, C.} (2007).
\newblock Fitted q-iteration in continuous action-space mdps.

\bibitem[{Argall et~al.(2009)Argall, Chernova, Veloso and
  Browning}]{BC-Robot-Learning-from-demonstration}
\text{Argall, B.~D.}, \text{Chernova, S.}, \text{Veloso, M.} and
  \text{Browning, B.} (2009).
\newblock A survey of robot learning from demonstration.
\newblock \textit{Robotics and autonomous systems}, \textbf{57} 469--483.

\bibitem[{Arjovsky et~al.(2017)Arjovsky, Chintala and Bottou}]{WGAN}
\text{Arjovsky, M.}, \text{Chintala, S.} and \text{Bottou, L.} (2017).
\newblock Wasserstein generative adversarial networks.
\newblock In \textit{International conference on machine learning}. PMLR.

\bibitem[{Auer et~al.(2002)Auer, Cesa-Bianchi and Fischer}]{MABexploration}
\text{Auer, P.}, \text{Cesa-Bianchi, N.} and \text{Fischer, P.} (2002).
\newblock Finite-time analysis of the multiarmed bandit problem.
\newblock \textit{Machine learning}, \textbf{47} 235--256.

\bibitem[{Auer et~al.(2009)Auer, Jaksch and Ortner}]{UCRL}
\text{Auer, P.}, \text{Jaksch, T.} and \text{Ortner, R.} (2009).
\newblock Near-optimal regret bounds for reinforcement learning.
\newblock In \textit{Advances in Neural Information Processing Systems}.

\bibitem[{Ayoub et~al.(2020)Ayoub, Jia, Szepesvari, Wang and
  Yang}]{linear-approximation-example-1}
\text{Ayoub, A.}, \text{Jia, Z.}, \text{Szepesvari, C.}, \text{Wang, M.} and
  \text{Yang, L.} (2020).
\newblock Model-based reinforcement learning with value-targeted regression.
\newblock In \textit{International Conference on Machine Learning}.

\bibitem[{Azar et~al.(2017)Azar, Osband and Munos}]{explorationRL1}
\text{Azar, M.~G.}, \text{Osband, I.} and \text{Munos, R.} (2017).
\newblock Minimax regret bounds for reinforcement learning.
\newblock In \textit{International Conference on Machine Learning}.

\bibitem[{Beck and Teboulle(2003)}]{mirror-descent-1}
\text{Beck, A.} and \text{Teboulle, M.} (2003).
\newblock Mirror descent and nonlinear projected subgradient methods for convex
  optimization.
\newblock \textit{Operations Research Letters}, \textbf{31} 167--175.

\bibitem[{Bradtke and Barto(1996)}]{LinearLSforTD}
\text{Bradtke, S.~J.} and \text{Barto, A.~G.} (1996).
\newblock Linear least-squares algorithms for temporal difference learning.
\newblock \textit{Machine learning}, \textbf{22} 33--57.

\bibitem[{Buckman et~al.(2020)Buckman, Gelada and Bellemare}]{importanceofpess}
\text{Buckman, J.}, \text{Gelada, C.} and \text{Bellemare, M.~G.} (2020).
\newblock The importance of pessimism in fixed-dataset policy optimization.
\newblock \textit{arXiv preprint arXiv:2009.06799}.

\bibitem[{Cai et~al.(2019)Cai, Hong, Chen and Wang}]{cai2019global}
\text{Cai, Q.}, \text{Hong, M.}, \text{Chen, Y.} and \text{Wang, Z.} (2019).
\newblock On the global convergence of imitation learning: A case for linear
  quadratic regulator.
\newblock \textit{arXiv preprint arXiv:1901.03674}.

\bibitem[{Cai et~al.(2020)Cai, Yang, Jin and Wang}]{OPPO}
\text{Cai, Q.}, \text{Yang, Z.}, \text{Jin, C.} and \text{Wang, Z.} (2020).
\newblock Provably efficient exploration in policy optimization.
\newblock In \textit{International Conference on Machine Learning}.

\bibitem[{Chang et~al.(2021)Chang, Uehara, Sreenivas, Kidambi and Sun}]{KNR}
\text{Chang, J.~D.}, \text{Uehara, M.}, \text{Sreenivas, D.}, \text{Kidambi,
  R.} and \text{Sun, W.} (2021).
\newblock Mitigating covariate shift in imitation learning via offline data
  without great coverage.
\newblock \textit{arXiv preprint arXiv:2106.03207}.

\bibitem[{Chen and Jiang(2019)}]{jiangnanOffline}
\text{Chen, J.} and \text{Jiang, N.} (2019).
\newblock Information-theoretic considerations in batch reinforcement learning.
\newblock In \textit{International Conference on Machine Learning}.

\bibitem[{Chen et~al.(2020)Chen, Wang, Liu, Yang, Li, Wang and
  Zhao}]{chen-GAIL}
\text{Chen, M.}, \text{Wang, Y.}, \text{Liu, T.}, \text{Yang, Z.}, \text{Li,
  X.}, \text{Wang, Z.} and \text{Zhao, T.} (2020).
\newblock On computation and generalization of generative adversarial imitation
  learning.
\newblock \textit{ArXiv 2001.02792}.

\bibitem[{Chen et~al.(2017)Chen, Zhang, Boedihardjo, Dai and
  Lu}]{application-gail-nlp}
\text{Chen, Z.}, \text{Zhang, X.}, \text{Boedihardjo, A.~P.}, \text{Dai, J.}
  and \text{Lu, C.-T.} (2017).
\newblock Multimodal storytelling via generative adversarial imitation
  learning.
\newblock In \textit{International Joint Conference on Artificial
  Intelligence}.

\bibitem[{Demiris* and Johnson(2003)}]{imitation+robotics}
\text{Demiris*, Y.} and \text{Johnson, M.} (2003).
\newblock Distributed, predictive perception of actions: {A} biologically
  inspired robotics architecture for imitation and learning.
\newblock \textit{Connection Science}, \textbf{15} 231--243.

\bibitem[{Duan et~al.(2020)Duan, Jia and Wang}]{minimaxoffpolicy}
\text{Duan, Y.}, \text{Jia, Z.} and \text{Wang, M.} (2020).
\newblock Minimax-optimal off-policy evaluation with linear function
  approximation.
\newblock In \textit{International Conference on Machine Learning}.

\bibitem[{Finn et~al.(2016)Finn, Levine and
  Abbeel}]{IRL-Continuous-Inverse-control-deep}
\text{Finn, C.}, \text{Levine, S.} and \text{Abbeel, P.} (2016).
\newblock Guided cost learning: Deep inverse optimal control via policy
  optimization.
\newblock In \textit{International Conference on Machine Learning}.

\bibitem[{Fisac et~al.(2018)Fisac, Akametalu, Zeilinger, Kaynama, Gillula and
  Tomlin}]{GP}
\text{Fisac, J.~F.}, \text{Akametalu, A.~K.}, \text{Zeilinger, M.~N.},
  \text{Kaynama, S.}, \text{Gillula, J.} and \text{Tomlin, C.~J.} (2018).
\newblock A general safety framework for learning-based control in uncertain
  robotic systems.
\newblock \textit{IEEE Transactions on Automatic Control}, \textbf{64}
  2737--2752.

\bibitem[{Fujimoto et~al.(2019{\natexlab{a}})Fujimoto, Conti, Ghavamzadeh and
  Pineau}]{BCQ}
\text{Fujimoto, S.}, \text{Conti, E.}, \text{Ghavamzadeh, M.} and \text{Pineau,
  J.} (2019{\natexlab{a}}).
\newblock Benchmarking batch deep reinforcement learning algorithms.
\newblock \textit{arXiv preprint arXiv:1910.01708}.

\bibitem[{Fujimoto et~al.(2019{\natexlab{b}})Fujimoto, Meger and
  Precup}]{offlineRL2}
\text{Fujimoto, S.}, \text{Meger, D.} and \text{Precup, D.}
  (2019{\natexlab{b}}).
\newblock Off-policy deep reinforcement learning without exploration.
\newblock In \textit{International Conference on Machine Learning}.

\bibitem[{Geist et~al.(2019)Geist, Scherrer and Pietquin}]{regularizedMDP}
\text{Geist, M.}, \text{Scherrer, B.} and \text{Pietquin, O.} (2019).
\newblock A theory of regularized markov decision processes.
\newblock In \textit{International Conference on Machine Learning}.

\bibitem[{Goodfellow et~al.(2014)Goodfellow, Pouget-Abadie, Mirza, Xu,
  Warde-Farley, Ozair, Courville and Bengio}]{GANs}
\text{Goodfellow, I.}, \text{Pouget-Abadie, J.}, \text{Mirza, M.}, \text{Xu,
  B.}, \text{Warde-Farley, D.}, \text{Ozair, S.}, \text{Courville, A.} and
  \text{Bengio, Y.} (2014).
\newblock Generative adversarial nets.
\newblock In \textit{Advances in Neural Information Processing Systems}.

\bibitem[{Hazan(2019)}]{mirror-descent-2}
\text{Hazan, E.} (2019).
\newblock Introduction to online convex optimization.
\newblock \textit{arXiv preprint arXiv:1909.05207}.

\bibitem[{Ho and Ermon(2016)}]{GAIL}
\text{Ho, J.} and \text{Ermon, S.} (2016).
\newblock Generative adversarial imitation learning.
\newblock \textit{arXiv preprint arXiv:1606.03476}.

\bibitem[{Hussein et~al.(2017)Hussein, Gaber, Elyan and Jayne}]{IL-survey}
\text{Hussein, A.}, \text{Gaber, M.~M.}, \text{Elyan, E.} and \text{Jayne, C.}
  (2017).
\newblock Imitation learning: A survey of learning methods.
\newblock \textit{ACM Computing Surveys (CSUR)}, \textbf{50} 1--35.

\bibitem[{Jalali et~al.(2019)Jalali, Kebria, Khosravi, Saleh, Nahavandi and
  Nahavandi}]{autoD}
\text{Jalali, S. M.~J.}, \text{Kebria, P.~M.}, \text{Khosravi, A.},
  \text{Saleh, K.}, \text{Nahavandi, D.} and \text{Nahavandi, S.} (2019).
\newblock Optimal autonomous driving through deep imitation learning and
  neuroevolution.
\newblock In \textit{2019 IEEE International Conference on Systems, Man and
  Cybernetics (SMC)}. IEEE.

\bibitem[{Jin et~al.(2018)Jin, Allen-Zhu, Bubeck and
  Jordan}]{ProvablyefficientQlearning}
\text{Jin, C.}, \text{Allen-Zhu, Z.}, \text{Bubeck, S.} and \text{Jordan,
  M.~I.} (2018).
\newblock Is {Q}-learning provably efficient?
\newblock In \textit{Advances in Neural Information Processing Systems}.

\bibitem[{Jin et~al.(2020)Jin, Jin, Luo, Sra and Yu}]{advMDP3JIN}
\text{Jin, C.}, \text{Jin, T.}, \text{Luo, H.}, \text{Sra, S.} and \text{Yu,
  T.} (2020).
\newblock Learning adversarial markov decision processes with bandit feedback
  and unknown transition.
\newblock In \textit{International Conference on Machine Learning}.

\bibitem[{Jin et~al.(2019)Jin, Yang, Wang and
  Jordan}]{LinearFunctionApproximation}
\text{Jin, C.}, \text{Yang, Z.}, \text{Wang, Z.} and \text{Jordan, M.~I.}
  (2019).
\newblock Provably efficient reinforcement learning with linear function
  approximation.
\newblock In \textit{Annual Conference on Learning Theory}.

\bibitem[{Jin et~al.(2021)Jin, Yang and Wang}]{pessimismRl-Jinchi}
\text{Jin, Y.}, \text{Yang, Z.} and \text{Wang, Z.} (2021).
\newblock Is pessimism provably efficient for offline {RL}?
\newblock In \textit{International Conference on Machine Learning}.

\bibitem[{Kakade et~al.(2020)Kakade, Krishnamurthy, Lowrey, Ohnishi and
  Sun}]{KNRSetting}
\text{Kakade, S.}, \text{Krishnamurthy, A.}, \text{Lowrey, K.}, \text{Ohnishi,
  M.} and \text{Sun, W.} (2020).
\newblock Information theoretic regret bounds for online nonlinear control.
\newblock In \textit{Advances in Neural Information Processing Systems}.

\bibitem[{Kakade(2001)}]{NPG}
\text{Kakade, S.~M.} (2001).
\newblock A natural policy gradient.
\newblock \textit{Advances in Neural Information Processing Systems}.

\bibitem[{Kebria et~al.(2019)Kebria, Khosravi, Salaken and
  Nahavandi}]{imitation+atuonomous-driving}
\text{Kebria, P.~M.}, \text{Khosravi, A.}, \text{Salaken, S.~M.} and
  \text{Nahavandi, S.} (2019).
\newblock Deep imitation learning for autonomous vehicles based on
  convolutional neural networks.
\newblock \textit{IEEE/CAA Journal of Automatica Sinica}, \textbf{7} 82--95.

\bibitem[{Kuefler et~al.(2017)Kuefler, Morton, Wheeler and
  Kochenderfer}]{application-gail-autodriving}
\text{Kuefler, A.}, \text{Morton, J.}, \text{Wheeler, T.} and
  \text{Kochenderfer, M.} (2017).
\newblock Imitating driver behavior with generative adversarial networks.
\newblock In \textit{IEEE Intelligent Vehicles Symposium (IV)}. IEEE.

\bibitem[{Kumar et~al.(2020)Kumar, Zhou, Tucker and Levine}]{CQL}
\text{Kumar, A.}, \text{Zhou, A.}, \text{Tucker, G.} and \text{Levine, S.}
  (2020).
\newblock Conservative {Q}-learning for offline reinforcement learning.
\newblock \textit{arXiv preprint arXiv:2006.04779}.

\bibitem[{Levine and Koltun(2012)}]{IRL-Continuous-Inverse-Optimal-Control}
\text{Levine, S.} and \text{Koltun, V.} (2012).
\newblock Continuous inverse optimal control with locally optimal examples.
\newblock In \textit{International Conference on Machine Learning}.

\bibitem[{Levine et~al.(2020)Levine, Kumar, Tucker and Fu}]{uniform-covering}
\text{Levine, S.}, \text{Kumar, A.}, \text{Tucker, G.} and \text{Fu, J.}
  (2020).
\newblock Offline reinforcement learning: Tutorial, review, and perspectives on
  open problems.
\newblock \textit{arXiv preprint arXiv:2005.01643}.

\bibitem[{Liu et~al.(2020)Liu, Swaminathan, Agarwal and Brunskill}]{PessFQI}
\text{Liu, Y.}, \text{Swaminathan, A.}, \text{Agarwal, A.} and \text{Brunskill,
  E.} (2020).
\newblock Provably good batch reinforcement learning without great exploration.
\newblock \textit{arXiv preprint arXiv:2007.08202}.

\bibitem[{Merel et~al.(2017)Merel, Tassa, TB, Srinivasan, Lemmon, Wang, Wayne
  and Heess}]{application-gail-human}
\text{Merel, J.}, \text{Tassa, Y.}, \text{TB, D.}, \text{Srinivasan, S.},
  \text{Lemmon, J.}, \text{Wang, Z.}, \text{Wayne, G.} and \text{Heess, N.}
  (2017).
\newblock Learning human behaviors from motion capture by adversarial
  imitation.
\newblock \textit{arXiv preprint arXiv:1707.02201}.

\bibitem[{Munos and Szepesv{\'a}ri(2008)}]{uniform4}
\text{Munos, R.} and \text{Szepesv{\'a}ri, C.} (2008).
\newblock Finite-time bounds for fitted value iteration.
\newblock \textit{Journal of Machine Learning Research}, \textbf{9}.

\bibitem[{Nemirovskij and Yudin(1983)}]{policy_optimization}
\text{Nemirovskij, A.~S.} and \text{Yudin, D.~B.} (1983).
\newblock Problem complexity and method efficiency in optimization.

\bibitem[{Neu and Szepesv\'{a}ri(2007)}]{AL+IRL+Grandient}
\text{Neu, G.} and \text{Szepesv\'{a}ri, C.} (2007).
\newblock Apprenticeship learning using inverse reinforcement learning and
  gradient methods.
\newblock In \textit{Conference on Uncertainty in Artificial Intelligence}.
  AUAI Press.

\bibitem[{Ng and Russell(2000)}]{IRL-theory-Algorithms-for-IRL}
\text{Ng, A.~Y.} and \text{Russell, S.~J.} (2000).
\newblock Algorithms for inverse reinforcement learning.
\newblock In \textit{International Conference on Machine Learning}.

\bibitem[{Rajaraman et~al.(2021)Rajaraman, Han, Yang, Ramchandran and
  Jiao}]{Provably-Breaking-the-Quadratic-Error-Compounding-Barrier-in-Imitation-Learning}
\text{Rajaraman, N.}, \text{Han, Y.}, \text{Yang, L.~F.}, \text{Ramchandran,
  K.} and \text{Jiao, J.} (2021).
\newblock Provably breaking the quadratic error compounding barrier in
  imitation learning, optimally.
\newblock \textit{arXiv preprint arXiv:2102.12948}.

\bibitem[{Rajaraman et~al.(2020)Rajaraman, Yang, Jiao and
  Ramchandran}]{jiao-BC-Toward-the-Fundamental-Limits-of-IL}
\text{Rajaraman, N.}, \text{Yang, L.}, \text{Jiao, J.} and \text{Ramchandran,
  K.} (2020).
\newblock Toward the fundamental limits of imitation learning.
\newblock In \textit{Advances in Neural Information Processing Systems},
  vol.~33.

\bibitem[{Rashidinejad et~al.(2021)Rashidinejad, Zhu, Ma, Jiao and
  Russell}]{MACONG-2021new-Offline+IL}
\text{Rashidinejad, P.}, \text{Zhu, B.}, \text{Ma, C.}, \text{Jiao, J.} and
  \text{Russell, S.} (2021).
\newblock Bridging offline reinforcement learning and imitation learning: A
  tale of pessimism.
\newblock \textit{arXiv preprint arXiv:2103.12021}.

\bibitem[{Rosenberg and Mansour(2019)}]{advMDP2}
\text{Rosenberg, A.} and \text{Mansour, Y.} (2019).
\newblock Online convex optimization in adversarial markov decision processes.
\newblock In \textit{International Conference on Machine Learning}.

\bibitem[{Ross and Bagnell(2010)}]{BC-corvirate-shift-paper-1}
\text{Ross, S.} and \text{Bagnell, D.} (2010).
\newblock Efficient reductions for imitation learning.
\newblock In \textit{International Conference on Artificial Intelligence and
  Statistics}.

\bibitem[{Ross et~al.(2011)Ross, Gordon and
  Bagnell}]{BC-corvirate-shift-paper-2}
\text{Ross, S.}, \text{Gordon, G.} and \text{Bagnell, D.} (2011).
\newblock A reduction of imitation learning and structured prediction to
  no-regret online learning.
\newblock In \textit{International Conference on Artificial Intelligence and
  Statistics}.

\bibitem[{Schulman et~al.(2015)Schulman, Levine, Abbeel, Jordan and
  Moritz}]{TRPO}
\text{Schulman, J.}, \text{Levine, S.}, \text{Abbeel, P.}, \text{Jordan, M.}
  and \text{Moritz, P.} (2015).
\newblock Trust region policy optimization.
\newblock In \textit{International Conference on Machine Learning}.

\bibitem[{Schulman et~al.(2017)Schulman, Wolski, Dhariwal, Radford and
  Klimov}]{ppo}
\text{Schulman, J.}, \text{Wolski, F.}, \text{Dhariwal, P.}, \text{Radford, A.}
  and \text{Klimov, O.} (2017).
\newblock Proximal policy optimization algorithms.
\newblock \textit{arXiv preprint arXiv:1707.06347}.

\bibitem[{Shani et~al.(2020{\natexlab{a}})Shani, Efroni and Mannor}]{TRPO2}
\text{Shani, L.}, \text{Efroni, Y.} and \text{Mannor, S.} (2020{\natexlab{a}}).
\newblock Adaptive trust region policy optimization: Global convergence and
  faster rates for regularized mdps.
\newblock In \textit{Proceedings of the AAAI Conference on Artificial
  Intelligence}, vol.~34.

\bibitem[{Shani et~al.(2020{\natexlab{b}})Shani, Efroni, Rosenberg and
  Mannor}]{advMDPbanditfeed}
\text{Shani, L.}, \text{Efroni, Y.}, \text{Rosenberg, A.} and \text{Mannor, S.}
  (2020{\natexlab{b}}).
\newblock Optimistic policy optimization with bandit feedback.
\newblock In \textit{International Conference on Machine Learning}.

\bibitem[{Shani et~al.(2021)Shani, Zahavy and
  Mannor}]{Online-Apprenticeship-Learning}
\text{Shani, L.}, \text{Zahavy, T.} and \text{Mannor, S.} (2021).
\newblock Online apprenticeship learning.
\newblock \textit{arXiv preprint arXiv:2102.06924}.

\bibitem[{Siegel et~al.(2020)Siegel, Springenberg, Berkenkamp, Abdolmaleki,
  Neunert, Lampe, Hafner, Heess and Riedmiller}]{partial2}
\text{Siegel, N.~Y.}, \text{Springenberg, J.~T.}, \text{Berkenkamp, F.},
  \text{Abdolmaleki, A.}, \text{Neunert, M.}, \text{Lampe, T.}, \text{Hafner,
  R.}, \text{Heess, N.} and \text{Riedmiller, M.} (2020).
\newblock Keep doing what worked: Behavioral modelling priors for offline
  reinforcement learning.
\newblock \textit{arXiv preprint arXiv:2002.08396}.

\bibitem[{Syed et~al.(2008)Syed, Bowling and Schapire}]{AL-linear-program}
\text{Syed, U.}, \text{Bowling, M.} and \text{Schapire, R.} (2008).
\newblock Apprenticeship learning using linear programming.
\newblock In \textit{International Conference on Machine Learning}.

\bibitem[{Syed and Schapire(2008)}]{Game-Theoretic-Approach-AL2007}
\text{Syed, U.} and \text{Schapire, R.~E.} (2008).
\newblock A game-theoretic approach to apprenticeship learning.
\newblock In \textit{Advances in Neural Information Processing Systems}.

\bibitem[{Tsurumine et~al.(2019)Tsurumine, Cui, Yamazaki and
  Matsubara}]{application-gail-control}
\text{Tsurumine, Y.}, \text{Cui, Y.}, \text{Yamazaki, K.} and \text{Matsubara,
  T.} (2019).
\newblock Generative adversarial imitation learning with deep {P}-network for
  robotic cloth manipulation.
\newblock In \textit{IEEE-RAS International Conference on Humanoid Robots
  (Humanoids)}. IEEE.

\bibitem[{Uehara et~al.(2020)Uehara, Huang and Jiang}]{offlinepartialcoverage}
\text{Uehara, M.}, \text{Huang, J.} and \text{Jiang, N.} (2020).
\newblock Minimax weight and {Q}-function learning for off-policy evaluation.
\newblock In \textit{International Conference on Machine Learning}.

\bibitem[{Uehara and Sun(2021)}]{sunwenOffline}
\text{Uehara, M.} and \text{Sun, W.} (2021).
\newblock Pessimistic model-based offline {RL}: Pac bounds and posterior
  sampling under partial coverage.
\newblock \textit{arXiv preprint arXiv:2107.06226}.

\bibitem[{Vershynin(2010)}]{covering}
\text{Vershynin, R.} (2010).
\newblock Introduction to the non-asymptotic analysis of random matrices.
\newblock \textit{arXiv preprint arXiv:1011.3027}.

\bibitem[{Wang et~al.(2020{\natexlab{a}})Wang, Foster and
  Kakade}]{limitOffline}
\text{Wang, R.}, \text{Foster, D.~P.} and \text{Kakade, S.~M.}
  (2020{\natexlab{a}}).
\newblock What are the statistical limits of offline rl with linear function
  approximation?
\newblock \textit{arXiv preprint arXiv:2010.11895}.

\bibitem[{Wang et~al.(2020{\natexlab{b}})Wang, Novikov, Zolna, Springenberg,
  Reed, Shahriari, Siegel, Merel, Gulcehre, Heess et~al.}]{partial3}
\text{Wang, Z.}, \text{Novikov, A.}, \text{Zolna, K.}, \text{Springenberg,
  J.~T.}, \text{Reed, S.}, \text{Shahriari, B.}, \text{Siegel, N.},
  \text{Merel, J.}, \text{Gulcehre, C.}, \text{Heess, N.} \text{et~al.}
  (2020{\natexlab{b}}).
\newblock Critic regularized regression.
\newblock \textit{arXiv preprint arXiv:2006.15134}.

\bibitem[{Xie et~al.(2021{\natexlab{a}})Xie, Cheng, Jiang, Mineiro and
  Agarwal}]{jiangnanOffline2}
\text{Xie, T.}, \text{Cheng, C.-A.}, \text{Jiang, N.}, \text{Mineiro, P.} and
  \text{Agarwal, A.} (2021{\natexlab{a}}).
\newblock Bellman-consistent pessimism for offline reinforcement learning.
\newblock \textit{arXiv preprint arXiv:2106.06926}.

\bibitem[{Xie et~al.(2021{\natexlab{b}})Xie, Jiang, Wang, Xiong and
  Bai}]{jiangnanOffline3}
\text{Xie, T.}, \text{Jiang, N.}, \text{Wang, H.}, \text{Xiong, C.} and
  \text{Bai, Y.} (2021{\natexlab{b}}).
\newblock Policy finetuning: Bridging sample-efficient offline and online
  reinforcement learning.
\newblock \textit{arXiv preprint arXiv:2106.04895}.

\bibitem[{Xu et~al.(2020)Xu, Li and Yu}]{ErrorBoundINIMITATIONLEARNING}
\text{Xu, T.}, \text{Li, Z.} and \text{Yu, Y.} (2020).
\newblock Error bounds of imitating policies and environments.
\newblock In \textit{Advances in Neural Information Processing Systems}.

\bibitem[{Yang and Wang(2019{\natexlab{a}})}]{LinearsampleoptimalparaQlearning}
\text{Yang, L.} and \text{Wang, M.} (2019{\natexlab{a}}).
\newblock Sample-optimal parametric {Q}-learning using linearly additive
  features.
\newblock In \textit{International Conference on Machine Learning}.

\bibitem[{Yang and Wang(2019{\natexlab{b}})}]{linear-approximation-example-2}
\text{Yang, L.} and \text{Wang, M.} (2019{\natexlab{b}}).
\newblock Sample-optimal parametric {Q}-learning using linearly additive
  features.
\newblock In \textit{International Conference on Machine Learning}.

\bibitem[{Yang and Wang(2020)}]{RLinfeaturespace}
\text{Yang, L.} and \text{Wang, M.} (2020).
\newblock Reinforcement learning in feature space: Matrix bandit, kernels, and
  regret bound.
\newblock In \textit{International Conference on Machine Learning}.

\bibitem[{Yang et~al.(2020{\natexlab{a}})Yang, Nachum, Dai, Li and
  Schuurmans}]{upper-bound-b}
\text{Yang, M.}, \text{Nachum, O.}, \text{Dai, B.}, \text{Li, L.} and
  \text{Schuurmans, D.} (2020{\natexlab{a}}).
\newblock Off-policy evaluation via the regularized lagrangian.
\newblock \textit{arXiv preprint arXiv:2007.03438}.

\bibitem[{Yang et~al.(2020{\natexlab{b}})Yang, Jin, Wang, Wang and
  Jordan}]{upper-bound-a}
\text{Yang, Z.}, \text{Jin, C.}, \text{Wang, Z.}, \text{Wang, M.} and
  \text{Jordan, M.~I.} (2020{\natexlab{b}}).
\newblock Bridging exploration and general function approximation in
  reinforcement learning: Provably efficient kernel and neural value
  iterations.
\newblock \textit{arXiv preprint arXiv:2011.04622}.

\bibitem[{Yu et~al.(2021)Yu, Kumar, Rafailov, Rajeswaran, Levine and
  Finn}]{PessOfflinPOLICY2}
\text{Yu, T.}, \text{Kumar, A.}, \text{Rafailov, R.}, \text{Rajeswaran, A.},
  \text{Levine, S.} and \text{Finn, C.} (2021).
\newblock Combo: Conservative offline model-based policy optimization.
\newblock \textit{arXiv preprint arXiv:2102.08363}.

\bibitem[{Yu et~al.(2020)Yu, Thomas, Yu, Ermon, Zou, Levine, Finn and
  Ma}]{PessOfflinPOLICY}
\text{Yu, T.}, \text{Thomas, G.}, \text{Yu, L.}, \text{Ermon, S.}, \text{Zou,
  J.}, \text{Levine, S.}, \text{Finn, C.} and \text{Ma, T.} (2020).
\newblock Mopo: Model-based offline policy optimization.
\newblock \textit{arXiv preprint arXiv:2005.13239}.

\bibitem[{Zanette(2021)}]{LowerboundOfflineRL}
\text{Zanette, A.} (2021).
\newblock Exponential lower bounds for batch reinforcement learning: Batch rl
  can be exponentially harder than online rl.
\newblock In \textit{International Conference on Machine Learning}.

\bibitem[{Zhang and Wu(2021)}]{offlineGAILAPP}
\text{Zhang, J.} and \text{Wu, F.} (2021).
\newblock A method of offline reinforcement learning virtual reality satellite
  attitude control based on generative adversarial network.
\newblock \textit{Wireless Communications and Mobile Computing}, \textbf{2021}.

\bibitem[{Zhang et~al.(2020)Zhang, Cai, Yang and Wang}]{nerual-gail}
\text{Zhang, Y.}, \text{Cai, Q.}, \text{Yang, Z.} and \text{Wang, Z.} (2020).
\newblock {GAIL} with neural network parameterization: Global optimality and
  convergence rate.
\newblock In \textit{International Conference on Machine Learning}.

\bibitem[{Zhou et~al.(2021)Zhou, He and Gu}]{Gu-probility-measure}
\text{Zhou, D.}, \text{He, J.} and \text{Gu, Q.} (2021).
\newblock Provably efficient reinforcement learning for discounted mdps with
  feature mapping.
\newblock In \textit{International Conference on Machine Learning}.

\bibitem[{Zolna et~al.(2020)Zolna, Novikov, Konyushkova, Gulcehre, Wang, Aytar,
  Denil, de~Freitas and
  Reed}]{zOffline-Learning-from-Demonstrations-and-Unlabeled-Experience}
\text{Zolna, K.}, \text{Novikov, A.}, \text{Konyushkova, K.}, \text{Gulcehre,
  C.}, \text{Wang, Z.}, \text{Aytar, Y.}, \text{Denil, M.}, \text{de~Freitas,
  N.} and \text{Reed, S.} (2020).
\newblock Offline learning from demonstrations and unlabeled experience.
\newblock \textit{arXiv preprint arXiv:2011.13885}.

\end{thebibliography}

\newpage
\appendix{}

\section{Proofs of Supporting Lemmas: Analysis of OGAP}

\subsection{Proof of Lemma \ref{lem4.2}}\label{A.1}

\begin{proof}

For notational simplicity, we define operators $\mathbb{J}_h$ and $\mathbb{J}^k_h$ as
\begin{equation}\label{EA.1}
(\mathbb{J}_h f)(s)=\langle f(s,\cdot),\pi_{h}^{\mathrm{E}}(\cdot\given s)\rangle,~~ (\mathbb{J}^k_h f)(s)=\langle f(s,\cdot),\pi_{h}^k(\cdot\given s)\rangle,
\end{equation}
for any $s\in \cS$, $(k,h)\in [K]\times[H]$, and any function $f:\mathcal{S}\times\mathcal{A}\rightarrow\mathbb{R}$.
We define $\mathcal{F}_{k,h,1}, \mathcal{F}_{k,h,2}$ as follows,
\begin{equation}
\begin{aligned}
&\mathcal{F}_{k,h,1}=\sigma
\big(\{(s^{\tau}_i,a^{\tau}_i)\}_{(\tau,i)\in[k-1]\times[H]}\cup\{(s_i^k,a_i^k) \}_{i\in[h]}
\big) \label{EA.2}\\
&\mathcal{F}_{k,h,2}=\sigma
\big(\{(s^{\tau}_i,a^{\tau}_i)\}_{(\tau,i)\in[k-1]\times[H]}\cup\{(s_i^k,a_i^k) \}_{i\in[h]}\cup \{s^k_{h+1}\}
\big),
\end{aligned}
\end{equation}
where $s^k_{H+1}$ is defined as a null state for any $k\in[K]$. We define the time index as follows,
\begin{equation}\label{EA.4}
t(k,h,m)=(k-1)\cdot 2H+(h-1)\cdot 2+m,
\end{equation}
which imples that $\{\mathcal{F}_{k,h,m}\}_{(k,h,m)\in [K]\times[H]\times[2]}$ is a filtration with respect
to $t(k,h,m)$.

Now we are ready to prove Lemma \ref{lem4.2}. First we note that for any initial state $x\in\cS$, it holds that
\begin{equation}\label{EA.5}
V^{r^k}_{1,\pi^\rE}(x)-V^{r^k}_{1,\pi^k}(x)=
\underbrace{
\big(V^{r^k}_{1,\pi^\rE}(x)-V_1^k(x)\big)}_{\text{(i)}}
+\underbrace{
\big(\widehat{V}^k_1(x)-V^{r^k}_{1,\pi^k}(x)
\big)}_{\text{(ii)}},
\end{equation}
where $V_1^k$ is the estimated state value function in the stage of policy evaluation of OGAP (Lines  \ref{line:p3}--\ref{line:p4} of Algorithm \ref{alg:gail-online}). We calculate terms (i) and (ii) separately.

\vskip5pt
\noindent{\bf Term (i).}
By \eqref{eq:Bellman}, we have 
\begin{equation}\label{EA.6}
V_{h,\pi^\rE}^{r^k}(s)=\langle  Q_{h,\pi^\rE}^{r^k}(s,\cdot),\pi_{h}^{\mathrm{E}}(\cdot\given s)\rangle=\mathbb{J}_{h}Q_{h,\pi^\rE}^{r^k}(s),~~ \widehat{V}^k_h(s)=\langle \widehat{Q}_h^k(s,\cdot), \pi^k_h(\cdot\given s)\rangle=\mathbb{J}^k_{h}\widehat{Q}_h^k(s),
\end{equation}
for any $(k,h)\in[K]\times[H]$. We then have
\begin{equation}\label{EA.7}
\begin{aligned}
V_{h,\pi^\rE}^{r^k}-\widehat{V}^k_h&=\mathbb{J}_{h}Q_{h,\pi^\rE}^{r^k}-\mathbb{J}^k_{h}\widehat{Q}^{k}_{h}\\
&=\mathbb{J}_{h}(Q_{h,\pi^\rE}^{r^k}-\widehat{Q}_h^k)+(\mathbb{J}_{h}-\mathbb{J}^k_h)\widehat{Q}^{k}_{h},
\end{aligned}
\end{equation}
where $\mathbb{J}_h$ and $\mathbb{J}_h^k$ are defined in \eqref{EA.1}. By the property of state-value function and the definition of $\iota^k_h$, we have
\begin{equation}\label{EA.8}  Q_{h,\pi^\rE}^{r^k}=r^k_h+{\cP}_h V_{h+1,\pi^\rE}^{r^k},~
\widehat{Q}_h^k=r^k_h+\mathcal{P}_h \widehat{V}^k_{h+1}-\iota^k_h,
\end{equation}
Define $\zeta^k_h=(\mathbb{J}_h-\mathbb{J}^k_h)\widehat{Q}_h^k$ and plug \eqref{EA.8} into \eqref{EA.7}, we have 
\begin{equation}\label{EA.9}
V_{h,\pi^\rE}^{r^k}-\widehat{V}^k_h=\mathbb{J}_h\mathcal{P}_h(V_{h+1,\pi^\rE}^{r^k}-\widehat{V}^k_{h+1})+\mathbb{J}_h\iota_h^k+\zeta^k_h,
\end{equation}
for any $(k,h)\in[K]\times[H]$. Here $\iota_h^k$ is the prediction error defined in \eqref{eq:def-iota}. For any $k\in[H]$, note that $V_{h+1,\pi^\rE}^{r^k}=\widehat{V}^k_{H+1}=0$, we expand \eqref{EA.9} across $h\in[H]$ to obtain that
\begin{equation}
\begin{aligned}
V^{r^k}_{1,\pi^\rE}-\widehat{V}^k_1=\sum_{h=1}^H(\prod^{h-1}_{i=1}\mathbb{J}_{i}\mathcal{P}_{i})\mathbb{J}_h\iota_h^k+\sum_{h=1}^{H}(\prod_{i=1}^{h-1}\mathbb{J}_i\mathcal{P}_i)\zeta^k_h.
\end{aligned}
\label{EA.10}
\end{equation}
The effect of composite operator $\mathcal{P}_h\mathbb{J}_h$  on function $f$ is to calculate one-step expectation of $f$ following policy $\pi_{h}^{\mathrm{E}}$.  Hence we rewrite \eqref{EA.10} as 
\begin{equation}\label{EA.11}
\begin{aligned}
V^{r^k}_{1,\pi^\rE}(x)-\widehat{V}^k_1(x)&=
\sum_{h=1}^H
\big(\mathbb{E}_{\pi^{\mathrm{E}}}\big[\iota_h^k(s^k_h,a^k_h)|s_1=x
\big]
\big)\\ &\qquad+
\sum_{h=1}^H\mathbb{E}_{\pi^{\mathrm{E}}}\big[\langle \widehat{Q}_h^k(s_h,\cdot),\pi_{h}^{\mathrm{E}}(\cdot\given s_h)-\pi_{h}^k(\cdot\given x_h)\rangle \biggiven s_1=x\big].
\end{aligned}
\end{equation}
This characterize term (i).

\vskip5pt
\noindent{\bf Term (ii).}
By \eqref{EA.5}, we have 
\begin{equation}\label{EA.12}
\begin{aligned}
\iota^k_h&=r^k_h+\mathcal{P}_h\widehat{V}^k_{h+1}-\widehat{Q}_h^k\\
&=r^k_h+\mathcal{P}_h\widehat{V}^k_{h+1}-Q_{h,\pi^\rE}^{r^k}+(Q_{h,\pi^\rE}^{r^k}-\widehat{Q}_h^k)\\
&=\mathcal{P}_h(\widehat{V}^k_{h+1}-V_{h+1,\pi^\rE}^{r^k})+(Q_{h,\pi^\rE}^{r^k}-\widehat{Q}_h^k).
\end{aligned}
\end{equation}
By \eqref{EA.12}, we obtain that
\begin{equation}\label{EA.13}
\begin{aligned}
\widehat{V}^k_h-V^{r^k}_{h,\pi^k}&=\mathbb{J}^k_h(\widehat{Q}_h^k-Q_{h,\pi^\rE}^{r^k})+\iota^k_h-\iota^k_h\\
&=\big(\mathbb{J}^k_h(\widehat{Q}_h^k-Q_{h,\pi^\rE}^{r^k})-(\widehat{Q}_h^k-Q_{h,\pi^\rE}^{r^k})\big)+\mathcal{P}_h\big(\widehat{V}^k_{h+1}-V_{h+1,\pi^\rE}^{r^k}\big)-\iota^k_h.
\end{aligned}
\end{equation}
We define $D_{k,h,1}$ and $D_{k,h,2}$ as follows,
\begin{equation}\label{EA.14}
\begin{aligned}
&D_{k,h,1}=
\big(\mathbb{J}^k_h(\widehat{Q}_h^k-Q_{h,\pi^\rE}^{r^k})
\big)(s^k_h)-(\widehat{Q}_h^k-Q_{h,\pi^\rE}^{r^k})(s^k_h,a^k_h)\\
&D_{k,h,2}=
\big(\cP_h(\widehat{V}^k_{h+1}-V_{h+1,\pi^\rE}^{r^k})
\big)(s^k_h,a^k_h)-(\widehat{V}^k_{h+1}-V_{h+1,\pi^\rE}^{r^k})(s^k_{h+1}).
\end{aligned}
\end{equation}
By plugging \eqref{EA.14} into \eqref{EA.13}, we obtain that 
\begin{equation}\label{EA.15}
\widehat{V}^k_h(s^k_h)-V^{r^k}_{h,\pi^k}(s^k_h)=D_{k,h,1}+D_{k,h,2}+
(\widehat{V}^k_{h+1}-V_{h+1,\pi^\rE}^{r^k})(s^k_{h+1})-\iota_{h}^k(s^k_h,a^k_h).
\end{equation}
By telescoping \eqref{EA.15} with respect to $h\in[H]$, we have
\begin{equation}\label{EA.16}
\widehat{V}_1^k(x)-V^{r^k}_{h,\pi^k}(x)=\sum_{h=1}^{H}(D_{k,h,1}+D_{k,h,2})-\sum_{h=1}^H\iota^k_h(s^k_h,a^k_h).
\end{equation}
By the definition of $\mathcal{F}_{k,h,1}$ and $\mathcal{F}_{k,h,2}$ in \eqref{EA.2}, we have 
\begin{equation}\label{EA.17}
D_{k,h,1}\in \mathcal{F}_{k,h,1},~ D_{k,h,2}\in \mathcal{F}_{k,h,1},~ \mathbb{E}[D_{k,h,1}|\mathcal{F}_{k,h-1,1}]=0, ~\mathbb{E}[D_{k,h,2}|\mathcal{F}_{k,h,1}]=0.
\end{equation}
Following from \eqref{EA.17}, we define the martingale
\begin{equation}\label{EA.18}
\begin{aligned}
\mathcal{M}_{k,h,m}=\sum_{\substack{(\tau,i,l)\in [K]\times[H]\times[2]\\t(\tau,i,l)\leq t(k,h,m)}} D_{\tau,i,l},
\end{aligned}
\end{equation}
with respect to the time index $t(k, h, m)$ defined in \eqref{EA.4}. It is obvious that
\begin{equation}\label{EA.19}
\sum_{k=1}^K\sum_{h=1}^H(D_{k,h,1}+D_{k,h,2})=\mathcal{M}_{K,H,2}
\end{equation}
Combining \eqref{EA.5}, \eqref{EA.14}, and \eqref{EA.16} we obtain that
\begin{equation}\label{EA.20}
\begin{aligned}
\sum_{k=1}^K
\big(V^{r^k}_{1,\pi^\rE}(x)-V^{r^k}_{1,\pi^k}(x)
\big)&=\sum_{k=1}^K\sum_{h=1}^H\mathbb{E}_{\pi^{\mathrm{E}}}
\Big[\langle \widehat{Q}_h^k(s_h,\cdot), \pi_{h}^{\mathrm{E}}(\cdot\given s_h)-\pi_{h}^k(\cdot\given x_h)\rangle|s_1=x\Big]
\\ &\qquad+
\mathcal{M}_{K,H,2}+
\sum_{k=1}^K\sum_{h=1}^H
\Big(\mathbb{E}_{\pi^{\mathrm{E}}}
\big[\iota^k_h(s^k_h,a^k_h)|s_1=x
\big]-\iota^k_h(s^k_h,a^k_h)\Big).
\end{aligned}
\end{equation}
By this, we conclude the proof of Lemma \ref{lem4.2}.
\end{proof}

\subsection{Proof of Lemma \ref{lem:performance_improve}}\label{pf:lem:performance_improve}

\begin{proof}
By the update rule of OGAP in \eqref{eq:policy-solution} (Lines  \ref{line:p1}--\ref{line:p2} of Algorithm \ref{alg:gail-online}) and the property of the mirror descent, we have 
$$
\mathcal{L}_{k-1}   (\pi^{k}  )-\alpha^{-1} \cdot D   (\pi^{k}, \pi^{k-1}  ) \geq \mathcal{L}_{k-1}   (\pi^{\mathrm{E}} )-\alpha^{-1} \cdot D   (\pi^{\mathrm{E}}, \pi^{k-1}  )+\alpha^{-1} \cdot D   (\pi^{\mathrm{E}}, \pi^{k}  ).
$$
Recalling the definition of $\mathcal{L}_{k-1}(\pi)$ in \eqref{eq:def:mathcal-L} and rearranging the above inequality, we derive that
\begin{equation}
\begin{aligned}
\sum_{h=1}^{H}  \big \langle\widehat{Q}_h^{k-1}, \pi_{h}^{\mathrm{E}}-\pi_{h}^{k-1}  \big\rangle_{\mathcal{A}} 
&\leq \alpha^{-1} \cdot D   (\pi^{\mathrm{E}}, \pi^{k-1}  )-\alpha^{-1} \cdot D   (\pi^{\mathrm{E}}, \pi^{k}  ) \\
&\qquad +\sum_{h=1}^{H} \big \langle\widehat{Q}_h^{k-1}, \pi_{h}^{k}-\pi_{h}^{k-1}  \rangle_{\mathcal{A}}-\alpha^{-1} \cdot D   (\pi^{k}, \pi^{k-1}  ),
\end{aligned}
\label{eq:pf:lem:opt-error-1}
\end{equation}
where $D   (\pi^{k}, \pi^{k-1}  )=\sum_{h=1}^{H} D_{\mathrm{KL}}   (\pi_{h}^{k} \| \pi_{h}^{k-1}  ) .$ For the last two terms on the right-hand side of \eqref{eq:pf:lem:opt-error-1}, we have 
\[
\begin{aligned}
& \sum_{h=1}^{H}  \big \langle\widehat{Q}_h^{k-1}, \pi_{h}^{k}-\pi_{h}^{k-1}  \big\rangle_{\mathcal{A}}-\alpha^{-1} \cdot D   (\pi^{k}, \pi^{k-1}  ) \\
&\qquad \le \sum_{h=1}^{H}  \big (   \|\widehat{Q}_h^{k-1}  \|_{\mathcal{A}, \infty} \cdot   \|\pi_{h}^{k}-\pi_{h}^{k-1}  \|_{\mathcal{A}, 1}-   (2 \alpha  )^{-1} \cdot   \|\pi_{h}^{k}-\pi_{h}^{k-1}  \|_{\mathcal{A}, 1}^{2} \big ) \\
&\qquad \le \frac{\alpha}{2} \cdot \sum_{h=1}^{H}   \|\widehat{Q}_h^{k-1}  \|_{\mathcal{A}, \infty}^{2} \leq \alpha H^{3} \sqrt{d}/ 2,
\end{aligned}
\label{eq:19.1}
\]
where the first inequality follows from Holder's inequality and Pinsker's inequality, and the last inequality derives from the fact that $|r^\mu_h(\cdot, \cdot)|\le\sqrt{d}$ for any $h\in[H]$ and $\mu\in S$.
Since $\pi^{0}$ is a uniform distribution on $\mathcal{A}$, it holds that $D(\pi^{\rE}, \pi^{0}) \leq H \log (\operatorname{vol}(\mathcal{A}))$. Telescoping \eqref{eq:19.1} with respect to $k \in[K]$, we have 
\begin{equation}
\begin{aligned}
\sum_{k=1}^{K} \sum_{h=1}^{H}  \big [  \langle\widehat{Q}_{h}^{k-1}, \pi_{h}^{\mathrm{E}}-\pi_{h}^{k-1} \rangle_{\mathcal{A}}\big] &\leq \alpha KH^3\sqrt{d}/2 + \alpha^{-1} H D(\pi^\rE,\pi^0)\\
&\le \alpha KH^3\sqrt{d}/2 + \alpha^{-1} H\log (|\cA|).
\end{aligned}
\label{eq:19.2}
\end{equation}
Recalling that $\alpha = \sqrt{2\log(|\cA|)/(H^2K\sqrt{d})}$ and taking expectation on both side of \eqref{eq:19.2}, we have 
\[
\begin{aligned}
\sum_{k=1}^{K} \sum_{h=1}^{H} \mathbb{E}_{\pi^{\mathrm{E}}} \big [  \langle\widehat{Q}_{h}^{k-1}, \pi_{h}^{\mathrm{E}}-\pi_{h}^{k-1} \rangle_{\mathcal{A}} \biggiven s_1 = x\big] &\leq \alpha KH^3\sqrt{d}/2 + \alpha^{-1} H\log (|\cA|)\\
&\le\sqrt{2 H^4\sqrt{d}K\log (\operatorname{vol}(\mathcal{A}))}.
\end{aligned}
\]
Then we conclude the proof of Lemma \ref{lem:performance_improve}.
\end{proof}

\subsection{Proof of Lemma \ref{lem4.4}}\label{A.3}
\begin{proof}
Recalling that we define $D_{k,h,1}$ and $D_{k,h,2}$ in \eqref{EA.14} and the fact that $|r^\mu_h(\cdot,\cdot)|\le\sqrt{d}$ for any $\mu\in S$,
we derive that $|D_{k,h,1}|\leq 2H\sqrt{d}$ and $|D_{k,h,2}|\leq 2H\sqrt{d}$ for any $(k, h)\in [K] \times [H]$. Now by Azuma-Hoeffding inequality, we have
\begin{equation}\label{EA.25}
\mathbb{P}(|\mathcal{M}_{K,H,2}|>t)\leq 2\exp\Big(\frac{-t^2}{16H^3Kd}\Big),
\end{equation}
for any $t>0$. Setting $t=\sqrt{16H^3 dK\cdot\log(8/\xi)}$ with $\xi\in(0,1)$ in \eqref{EA.25}, we have
\$
\mathcal{M}_{K,H,2}\leq\sqrt{16H^3dK\cdot\log(8/\xi)},
\$
with probability at least $1 -\xi/4$.
\end{proof}

\subsection{Proof of Lemma \ref{lem4.5}}\label{A.4}

\begin{proof}
For notational simplicity, we write $\bar{Q}_{h}^{k}(s, a)={r}_{h}^k (s, a)+\widehat{\mathcal{P}}_{h} \widehat{V}_{h+1}^{k}(s, a)+\Gamma_{h}^k(s, a) .$ 
Then, from the policy evaluation stage in Lines \ref{line:p3}--\ref{line:p4} of Algorithm \ref{alg:gail-offline}, we have 
\begin{equation}
\widehat{Q}_{h}^{k}(s, a)=\min\Big\{\max   \big\{\bar{Q}_{h}^{ k}(s, a), 0\big\}, (H-h+1)\sqrt{d}\Big\}.\label{eq:pf:lem:pess-114}
\end{equation}
We introduce the following lemma. 
\begin{lemma}\label{lem:optim}
Let $\lambda =1$ in the construction of estimated kernels \eqref{eq:online-result-L} and $\kappa = C\sqrt{d\log(HdK/\xi)}$ in the construction of bonus \eqref{eq:bonus}. Then it holds with probability at least $1-\xi/4$ that
\[
\big|\mathcal{P}_{h} \widehat{V}_{h+1}^{ k} (s, a)-\widehat{\mathcal{P}}_{h}^k  \widehat{V}_{h+1}^{ k}(s, a)\big|\le \Gamma_h^k(s,a)
\]
for any $(s,a)\in\cS\times\cA$.
\end{lemma}
\begin{proof}
See Appendix \ref{pf:lem:con} for a detailed proof.
\end{proof}
By Lemma \ref{lem:optim}, we obtain that ${r}_{h}^k +\mathcal{P}_{h} \widehat{V}_{h+1}^{ k} \le\bar{Q}_{h}^{k} .$ 
Moreover, by the fact that $|{r}_{h}^k(s,a)| \le \sqrt{d}$ and $\widehat{V}_{h+1}^{ k}(s)= \la \widehat{Q}_h^k(s,\cdot),\pi_h^k(\cdot\given s)\ra_{\cA} \in[0,(H-h)\sqrt{d}]$  for any $(s,a)\in\cS\times\cA$, we have ${r}_{h}^k +\mathcal{P}_{h} V_{h+1,\pi^k}^{ k} \in[0,(H-h+1)\sqrt{d}]$.  Thus, we have
\[
\begin{aligned}
\widehat{Q}_{h}^{k}(s, a)&=\min\Big\{\max   \big\{\bar{Q}_{h}^{ k}(s, a), 0\big\},(H-h+1)\sqrt{d}\Big\}\\ &\ge\min\Big\{\max   \big\{{r}_{h}^k  (s,a)+\mathcal{P}_{h} \widehat{V}_{h+1}^{ k}(s,a), 0\big\}, (H-h+1)\sqrt{d}\Big\} \\ &={r}_{h}^k (s,a)+\mathcal{P}_{h} \widehat{V}_{h+1}^{ k}(s,a),
\end{aligned}
\]
which implies that $\iota_{h}^{ k} \le 0$. 

It remains to establish the lower bound of $\iota_{h}^{ k}(s, a)$. By Lemma \ref{lem:optim}, we have 
\#
\begin{aligned}
\bar{Q}_{h}^{k}(s, a) &={r}_{h}^k (s, a)+\widehat{\mathcal{P}}_{h} \widehat{V}_{h+1}^{k}(s, a)+\Gamma_{h}^k(s, a) \\
& \ge {r}_{h}^k (s, a)+\mathcal{P}_{h} \widehat{V}_{h+1}^{ k}(s, a) \ge 0,\label{eq:637}
\end{aligned}
\#
where the last inequality follows from the fact that $\widehat{V}_{h+1}^{ k}(s, a) \geq 0$ and ${r}_{h}^k (s, a) \geq 0.$ By \eqref{eq:pf:lem:pess-114} and \eqref{eq:637}, we obtain that $\widehat{Q}_{h}^{k}(s, a) \le \bar{Q}_{h}^{ k}(s, a)$, which implies that
$$
\begin{aligned}
\iota_{h}^{k}(s, a) &=   ({r}_{h}^k +\mathcal{P}_{h} \widehat{V}_{h}^{k}  )(s, a)-\widehat{Q}_{h}^{k}(s, a) \\
& \ge   (\mathcal{P}_{h}-\widehat{\mathcal{P}}_{h}  ) \widehat{V}_{h}^{k}(s, a)-\Gamma_{h}^k(s, a) \\
& \ge  -2\Gamma_h^k(s,a).
\end{aligned}
$$
Here the last inequality follows from Lemma \ref{lem:optim}. Thus, we conclude the proof of Lemma \ref{lem:pess}.
\label{pf:lem:pess}
\end{proof}

\subsection{Proof of Lemma \ref{lem4.6}}\label{A.5}
\begin{proof}
By the construction of bonus $\Gamma_h^k$ in \eqref{eq:bonus}, we have
\begin{equation}
\begin{aligned}
\sum_{h=1}^H\sum_{k=1}^K\Gamma_h^k(s_h^k,a_h^k) &= H\sqrt{d}\cdot \sum_{h=1}^H\sum_{k=1}^K\min \left\{1, \kappa\cdot\varphi_{h}^{k}\left(s_{h}^{k}, a_{h}^{k}\right)^{\top}(\Lambda_{h}^{k})^{-1} \varphi_{h}^{k}(s_{h}^{k}, a_{h}^{k})\right\}\\
&\le H\sqrt{d}\kappa\cdot\sum_{h=1}^H  \Big({K\cdot \sum_{k=1}^{K}\varphi_{h}^{k}\left(s_{h}^{k}, a_{h}^{k}\right)^{\top}(\Lambda_{h}^{k})^{-1} \varphi_{h}^{k}(s_{h}^{k}, a_{h}^{k})}\Big)^{1/2},\label{eq:73}
\end{aligned}
\end{equation}
where the last inequality comes from Cauchy-Schwarz inequality.
To upper bound the right-hand side of \eqref{eq:73}, we introduce the following lemma.

\begin{lemma}[Elliptical Potential \citep{elliptical-potential}]\label{lem:elli}
Let $\{\phi_t\}_{t=1}^\infty$ be an $\mathbb{R}^{d}$-valued sequence. Meanwhile, let $\Lambda_{0} \in \mathbb{R}^{d \times d}$ be a positive-definite matrix and $\Lambda_{t}=$
$\Lambda_{0}+\sum_{j=1}^{t-1} \phi_{j} \phi_{j}^{\top}$. It holds for any $t \in \mathbb{Z}_{+}$ that
$$
\sum_{j=1}^{t} \min \left\{1, \phi_{j}^{\top} \Lambda_{j}^{-1} \phi_{j}\right\} \leq 2 \log \left(\frac{\operatorname{det}\left(\Lambda_{t+1}\right)}{\operatorname{det}\left(\Lambda_{1}\right)}\right). 
$$
Moreover, assuming that $\left\|\phi_{j}\right\|_{2} \leq 1$ for any $j \in \mathbb{Z}_{+}$ and $\lambda_{\min }\left(\Lambda_{0}\right) \geq 1$, it holds for any $t \in \mathbb{Z}_{+}$
that
$$
\log \left(\frac{\operatorname{det}\left(\Lambda_{t+1}\right)}{\operatorname{det}\left(\Lambda_{1}\right)}\right) \leq \sum_{j=1}^{t} \phi_{j}^{\top} \Lambda_{j}^{-1} \phi_{j} \leq 2 \log \left(\frac{\operatorname{det}\left(\Lambda_{t+1}\right)}{\operatorname{det}\left(\Lambda_{1}\right)}\right). 
$$
\end{lemma}
\begin{proof}
See proof of Lemma 11 in \cite{elliptical-potential} for a detailed proof.
\end{proof}
For any fixed $h\in[H]$, by Lemma \ref{lem:elli}, we have
\#\label{eq:ffffff}
\sum_{k=1}^{K}\varphi_{h}^{k}\left(s_{h}^{k}, a_{h}^{k}\right)^{\top}(\Lambda_{h}^{k})^{-1} \varphi_{h}^{k}(s_{h}^{k}, a_{h}^{k})\leq 2 \log \left(\frac{\operatorname{det}(\Lambda_{h}^{K+1})}{\operatorname{det}(\Lambda_{h}^{1})}\right),
\#
where $\Lambda_{h}^{1}=\lambda \cdot I$ and $\Lambda_{h}^{K+1} \in \mathcal{F}_{K, H, 2}$, which is defined in \eqref{EA.2}. By Assumption \ref{def:linear}, we obtain that
\begin{equation}
\begin{aligned}
\left\|\varphi_{h}^{k}(s,a)\right\|_{2} &= \Big\|\int_\cS \phi(s,a,s^\prime)\widehat{V}_{h+1}^k(s^\prime)\mathrm{d}s^\prime\Big\|_2 \\
&\le H\sqrt{d}\cdot\Big\|\int_\cS \phi(s,a,s^\prime)\mathrm{d}s^\prime\Big\|_2\\
&\le H\sqrt{d}\cdot\text{Vol}(\cS)\cdot \sup_{s^\prime\in\cS}\|\phi(s,a,s^\prime)\|_2 \leq Hd^{3/2}R\cdot \text{Vol}(\cS).
\end{aligned}\label{b:varphi}
\end{equation}
Here the first inequality comes from the fact that $\widehat{V}_h^k\in[0,H\sqrt{d}]$ for any $(k,h)\in[K]\times[H]$, and the last inequality comes from the fact that $\|\phi(\cdot,\cdot,\cdot)\|_2\le dR$, which can be verified as follows,
\begin{equation}
\sup_{(s,a,s^\prime)\in\cS\times \cA\times \cS}\|\phi(s,a,s^\prime)\|_2= \sup_{(s,a,s^\prime)\in\cS\times \cA\times \cS}\sqrt{\sum_{i=1}^d \|\phi(s,a,s^\prime)^\top\boldsymbol{e}_i\|_2^2 }\le dR. \label{phi-2}
\end{equation}
Here $\{\boldsymbol{e}_i\}_{i=1}^d$ is a group of orthonormal basis of $\mathbb{R}^d$ and the last inequality follows from Assumption \ref{def:linear}.
By the definition of $\Lambda_h^k$ in \eqref{eq:online-result-L}, we can upper bound $\text{det}(\Lambda_h^{K+1})$ by \eqref{b:varphi} as follows,
\begin{equation}
\begin{aligned}
\det(\Lambda_h^{K+1} )&= \det \Big(\sum_{k=1}^K \varphi(s_h^k,a_h^k) \varphi(s_h^k,a_h^k)^\top +I \Big)\\
& \le \Big (\det \big  ((Hd^{3/2}R\cdot\text{Vol}(\cS) + 1)\cdot I \big)\Big)^d,
\end{aligned} \label{bound-det}
\end{equation} 
which implies that
\begin{equation}
 \log \left(\frac{\operatorname{det}\left(\Lambda_{h}^{K+1}\right)}{\operatorname{det}\left(\Lambda_{h}^{1}\right)}\right)  \leq 2 d \cdot \log (H^2d^3R^2 K\cdot\text{Vol}(\cS)^2).
 \label{eq:75}
\end{equation}
Recalling that $\kappa = C\sqrt{d\log(HdK/\xi)}$, combining \eqref{eq:73}, \eqref{eq:ffffff}, and \eqref{eq:75}, we have 
\begin{equation*}
\begin{aligned}
\sum_{h=1}^H\sum_{k=1}^K\Gamma_h^k(s_h^k,a_h^k) &\le 4H\sqrt{d}\kappa\cdot H\sqrt{dK\cdot \log (H^2d^3R^2 K\cdot\text{Vol}(\cS)^2)}\\
&\le C ^\prime \sqrt{H^4d^3K}\cdot\log(HdK/\xi),
\end{aligned}
\end{equation*}
where $C^\prime$ is an absolute constant determined by $C,R$, and $\log(\text{Vol}(\cS))$.
By this, we conclude the proof of Lemma \ref{lem4.6}.
\end{proof}

\subsection{Proof of Lemma \ref{lem4.7}}\label{A.6}
\begin{proof}
By the definition of cumulative reward in \eqref{eq:def-J}, we observe that
\begin{equation}
\begin{aligned}
J(\pi,\mu) &= \mathbb{E}_\pi \sum_{h=1}^H\big [r^\mu_h(s_h,a_h)\big] \\
&= \sum_{h=1}^H \mathbb{E}_\pi \big[r_h^\mu(s_h,a_h) \big]\\
& = \sum_{h=1}^H \int_{\cS\times \cA} \rho_h^\pi(s,a)\cdot r_h^\mu(s,a) \mathrm{d}s \mathrm{d}a,
\label{eq:visitation-measure}
\end{aligned} 
\end{equation}
where 
$\rho_h^\pi (s,a) = \PP(s_h = s, a_h = a)$ is the density of state-action visition measure on $\cS\times \cA$.
Recall that under Assumption \ref{def:linear}, we have $r^\mu_h(s,a) = \psi(s,a)^\top \mu_h$, hence we have
\#\label{EA.601}
\nabla_{\mu_{h} }J (\pi^k,\mu^k)=
\int_{\cS\times \cA} \rho_h^{\pi^k}(s,a)\cdot\psi(s,a)\mathrm{d}s \mathrm{d}a.
\#
By \eqref{eq:min-max}, we obtain that 
\begin{equation}\label{EA.48}
\begin{aligned}
L(\pi^k,\mu)-L(\pi^k,\mu^k)
&= \sum_{h=1}^H (\mu_h - \mu_h^k)^\top  \nabla_{\mu_{h} }L (\pi^k,\mu^k),\\
\text{where } \nabla_{\mu_{h} }L (\pi^k,\mu^k) &= \nabla_{\mu_{h} }J (\pi^\rE,\mu^k)-\nabla_{\mu_{h} }J (\pi^k,\mu^k).
\end{aligned}
\end{equation}
Combining \eqref{EA.601} and \eqref{EA.48}, we know that $L(\pi,\mu)$ is a linear function in $\mu$ for any $\pi$.
Recall that $\mu_h^{k+1}=\text{Proj}_{B} \{\mu_h^k+\eta\widehat{\nabla}_{\mu_{h}}L(\pi^k,\mu^k)\}$ in OGAP (Lines \ref{line:p5}--\ref{line:p6} of Algorithm \ref{alg:gail-online}), by the definition of the projection operator $\text{Proj}_{B}(\cdot)$, it holds that
\begin{equation}\label{EA.49}
\big[\mu_h^{k+1}-\mu_h^k-\eta \widehat{\nabla}_{\mu_h}L(\pi^k,\mu^k)\big]^{\top}(\mu_{h}-\mu_h^{k+1})\geq 0.
\end{equation}
Rearranging terms in \eqref{EA.49}, we obtain that
\begin{equation}\label{EA.50}
\begin{aligned}
\eta(\mu_{h}-\mu_h^{k+1})^{\top}\widehat{\nabla}_{\mu_{h}}L(\pi^k,\mu^k)
&\leq(\mu_h^{k+1}-\mu_h^k)^{\top}(\mu_{h}-\mu_h^{k+1})\\
&=\frac{1}{2}\Big(\|\mu_h^k-\mu_h\|^2_2-\|\mu_h^{k+1}-\mu_h\|^2_2-\|\mu_h^{k+1}-\mu_h^k\|^2_2\Big),
\end{aligned}
\end{equation}
which also implies that
\begin{equation}\label{EA.51}
\frac{1}{2}
\Big(\|\mu_h^k-\mu_h\|^2_2-\|\mu_h^{k+1}-\mu_h\|^2_2-\|\mu_h^{k+1}-\mu_h^k\|^2_2
\Big)-\eta(\mu_h-\mu_h^{k+1})^{\top}
\widehat{\nabla}_{\mu_h}L(\pi^k,\mu^k) \geq 0.
\end{equation}
By adding a term $\eta\nabla_{\mu_{h}}L(\pi^k,\mu^k)^{\top}(\mu_h-\mu_h^k)$
on both sides of \eqref{EA.51} and combining \eqref{EA.48}, we obtain that 
\begin{equation}
\begin{aligned}
L(\pi^k,\mu)-L(\pi^k,\mu^k)&\leq \sum_{h=1}^H\frac{1}{2\eta}(
\|\mu_h^k-\mu_h\|^2_2
-\|\mu_{h}^{k+1}-\mu_{h}\|^2_2
-\|\mu_{h}^{k+1}-\mu_{h}^k\|^2_2)\\
&\qquad+\sum_{h=1}^H\big[(\mu_{h}^{k+1}-\mu_{h}^k)^{\top}\widehat{\nabla}_{\mu_{h}}L(\pi^k,\mu^k)\big]  \\ &\qquad+\sum_{h=1}^H\big[(\mu_{h}^k-\mu_h)^{\top}(\widehat{\nabla}_{\mu_h}L(\pi^k,\mu^k)-\nabla_{\mu_{h}} L(\pi^k,\mu^k))\big],
\end{aligned} \label{eq:96}
\end{equation}
where we take the summation on $h$ from $1$ to $H$. 
By the fact that $\widehat{\nabla}_{\mu_h}L(\pi^k,\mu^k) = \nabla_{\mu_h}\tilde{J}(\pi^\rE,r^\mu)-\psi(s_h^k,a_h^k)$ and the definition of the GAIL objective function $L(\pi,\mu)$ in \eqref{eq:min-max}, we rewrite the third term on the right-hand side of  \eqref{eq:96} as
\begin{equation}
\begin{aligned}
\sum_{h=1}^H\big[(\mu_{h}^k-\mu_h)^{\top}(\widehat{\nabla}_{\mu_h}L(\pi^k,\mu^k)-\nabla_{\mu_{h}} L(\pi^k,\mu^k))\big]&= \sum_{h=1}^H\big[\mu_h^{\top}(\nabla_{\mu_{h}} \tilde{J}(\pi^k,\mu^k) - \nabla_{\mu_h} J(\pi^\rE,r^k))\big]\\ 
&+\sum_{h=1}^H\big[(\mu_{h}^k-\mu_h)^{\top}(\psi(s_{h}^k,a_h^k)-\nabla_{\mu_{h}} J(\pi^k,\mu^k))\big] \label{eq:87.1}.
\end{aligned}
\end{equation}
By \eqref{EA.601} and the definition of $\tilde{J}(\pi^\rE,r^\mu)$ in \eqref{MC:estimate}, we derive from \eqref{eq:96} and \eqref{eq:87.1} that
\$
L(\pi^k,\mu)-L(\pi^k,\mu^k)&\leq \sum_{h=1}^H\frac{1}{2\eta}\big(
\|\mu_h^k-\mu_h\|^2_2
-\|\mu_{h}^{k+1}-\mu_{h}\|^2_2
-\|\mu_{h}^{k+1}-\mu_{h}^k\|^2_2\big)\\
& \qquad +\sum_{h=1}^H\big[(\mu_{h}^{k+1}-\mu_{h}^k)^{\top}\widehat{\nabla}_{\mu_{h}}L(\pi^k,\mu^k)\big]+ \big[\tilde{J}(\pi^\rE,r^\mu)-J(\pi^\rE,r^\mu)\big] \\
& \qquad +\sum_{h=1}^H\big[(\mu_{h}^k-\mu_h)^{\top}(\psi(s_{h}^k,a_h^k)-\nabla_{\mu_{h}} J(\pi^k,\mu^k))\big].
\$
Upon telescoping sum on the below inequality for the index $k\in[K]$, we complete the proof of Lemma \ref{lem4.7}.
\end{proof}

\subsection{Proof of Lemma \ref{lem:MC estimate}}\label{pf:lem:MC estimate}
\begin{proof}
For any fixed reward parameter $\mu\in S$, we define $J^\tau (\pi^{\mathrm{E}},r^\mu)=\sum_{h=1}^H \psi(s_{h,\tau}^{\mathrm{E}},a_{h,\tau}^{\mathrm{E}})^\top\mu_h$ for any $\tau \in [N_1]$.
Since the expert demonstration $\mathbb{D}^{\mathrm{E}}=   \{   (s_{h, k}^{\mathrm{E}}, a_{h, k}^{\mathrm{E}}  )  \}_{(k, h) \in [N_1] \times[H]}$ involves $N_1$ independent trajectories induced by the expert policy $\pi^{\mathrm{E}}$, we apply Monte Carlo method to estimate $J(\pi^{\mathrm{E}},r^\mu)$ by 
$N_1$ i.i.d. samples $\{J^\tau (\pi^{\mathrm{E}},r^\mu)\}_{\tau=1}^{N_1}$.

Let $Z_n=\sum_{\tau=1}^n\big(J^\tau(\pi^{\mathrm{E}},r^\mu)-J(\pi^{\mathrm{E}},r^\mu)\big)$ and we have $|Z_n-Z_{n-1}|\le 2H\sqrt{d}$, since $|r_h(\cdot,\cdot)|\le \sqrt{d}$ for all $h\in[H]$. Note that $\{Z_n\}$ is a martingale with zero mean with respect to the filtration $\mathcal{F}_{n}=\sigma\big(\{s_h^{i},a_h^{i}\}_{(h,i)\in [H]\times [n]}\big)$, by Azuma-Hoffeding inequality, we have
$$
\mathbb{P}_{\D}(|Z_n|>t)\le2\exp\bigg(\frac{-2t^2}{4H^2dn}\bigg), 
$$
which implies that
$$
\mathbb{P}_\D\bigg(\bigg|\frac{Z_{N_1}}{N_1}\bigg|>m\bigg)\le2\exp\bigg(\frac{-2m^2N_1}{4H^2d}\bigg).
$$
Let $\delta = 2\exp\{-m^2N_1/({2H^2d})\}$, it holds with probability at least $1-\delta$ that
\#
\big|\tilde{J}(\pi^{\mathrm{E}},r^\mu)-J(\pi^{\mathrm{E}},r^\mu)\big|=\bigg|\frac{Z_{N_1}}{N_1}\bigg|\le H\sqrt{2d\log(2/\delta)/N_1}. 
\label{eq:pre-union-bound}
\#

We union bound $ \big|\tilde{J}(\pi^{\mathrm{E}},r^\mu)-J(\pi^{\mathrm{E}},r^\mu)\big|$ for any $\mu\in S$ as follows. Since the reward parameter domain $S$ defined in \eqref{eq:def:parameter domain} is not a finite set, we apply discretization on $S$ to derive a union bound on  $ \big|\tilde{J}(\pi^{\mathrm{E}},r^\mu)-J(\pi^{\mathrm{E}},r^\mu)\big|$. 
If we define a normed space $(\mathbb{R}^{Hd},\|\cdot\|_\star)$, where $\|\cdot\|_\star $ is defined as
\#
\|\mu\|_\star = \sup_{h\in[H]}\|\mu_h\|_2,\label{sup-norm}
\#
then parameter domain $S$ belongs to this normed space. Before we continue, we first introduce the definitions of $\epsilon$-covering and covering number as follows. 
\begin{definition}[$\epsilon$-covering] \label{def:covering}
Let $(V,\|\cdot\|)$ be a normed space, and $\Theta \subset V$. We say that $\left\{V_{1}, \ldots, V_{N}\right\}$ is an $\epsilon$-covering of $\Theta$ if $\Theta \subset \cup_{i=1}^{N} B\left(V_{i}, \epsilon\right)$, or equivalently, $\forall \theta \in \Theta$, $\exists i$ such that $\left\|\theta-V_{i}\right\| \leq \epsilon$. Here $B(V_i,\epsilon)$ denotes a ball centering $V_i$ with radius $\epsilon$.
\end{definition}
We define the covering number as follows,
$$
\cN(\Theta,\|\cdot\|, \epsilon):=\min \big\{n: \exists \epsilon \text { -covering over } \Theta \text { of size } n,\Theta\in(V,\|\cdot\|)\big\}. 
$$

With Definition \ref{def:covering}, we  introduce the following lemma to upper bound the covering number.
\begin{lemma}\label{lem:covering-number}
If $(V,\|\cdot\|)$ is a normed space, and
(i) $\Theta\subset V=\mathbb{R}^d$,
(ii) $\Theta $ is convex, 
(iii) $\epsilon B_{\rm unit}\in\Theta$, where $\epsilon >0$ and $B_{unit}$ is the unit ball in $\mathbb{R}^d$, then it holds that
$$
\cN(\Theta,\|\cdot\|, \epsilon) \le\left(\frac{3}{\epsilon}\right)^{d} \frac{\operatorname{vol}(\Theta)}{\operatorname{vol}(B_{\rm unit})}.
$$
\end{lemma}
\begin{proof}
See Lemma 5.2 of \cite{covering} for proof.
\end{proof}
Note that $S$ is convex as a subset of $\mathbb{R}^{Hd}$, we apply Lemma \ref{lem:covering-number} with $V = \mathbb{R}^{Hd}$, $\Theta = S$, $\|\cdot\| = \|\cdot\|_\star$, and an appropriate $\epsilon >0$ satisfing condition (iii) in Lemma \ref{lem:covering-number}, which implies that
$$
\cN(S,\|\cdot\|_\star,\epsilon) \le \left(\frac{3}{\epsilon}\right)^{Hd} d^{Hd/2}.
$$

By the definition of covering number, there exists an $\epsilon$-covering $\cV_\epsilon =\{\mu^1,...,\mu^{\cN(S,\|\cdot\|_\star,\epsilon)}\}\subset S$. For each $\mu\in \cV_\epsilon$, by \eqref{eq:pre-union-bound}, it holds that 
$$
\big|\tilde{J}(\pi^{\mathrm{E}},r^\mu)-J(\pi^{\mathrm{E}},r^\mu)\big|\le H\sqrt{2d\log(2\cN(S,\|\cdot\|_\star,\epsilon)/\xi)/N_1},
$$
with probability at least $1-\xi/\cN(S,\|\cdot\|_\star,\epsilon)$. By the union bound, it yields that 
\#
\begin{aligned}
\sup_{\mu \in \cV_\epsilon}\big|\tilde{J}(\pi^{\mathrm{E}},r^\mu)-J(\pi^{\mathrm{E}},r^\mu)\big|&\le H\sqrt{2d\log(2\cN(S,\|\cdot\|_\star,\epsilon)/\xi)/N_1}\\
&\le H\sqrt{\big(Hd^2\log(d) + 2Hd^2\log(3/\epsilon)+2d\log(2/\xi)\big)/N_1},
\end{aligned}
\label{eq:union-tra-1}
\#
with probability at least $1-\xi$.
Note that for any $\mu^\prime,\mu^{\prime\prime}\in S$ satisfying $\|\mu^\prime-\mu^{\prime\prime}\|_\star \le \epsilon$, it holds that
\#
\Big|\big[\tilde{J}(\pi^{\mathrm{E}},r^{\mu^{\prime}})-J(\pi^{\mathrm{E}},r^{\mu^{\prime}})\big]-\big[\tilde{J}(\pi^{\mathrm{E}},r^{\mu^{\prime\prime}})-J(\pi^{\mathrm{E}},r^{\mu^{\prime\prime}})\big]\Big|\le 4H\epsilon.
\label{eq:union-tra-2}
\#
Combining \eqref{eq:union-tra-1} and \eqref{eq:union-tra-2} and applying triangle inequality, we derive that
\#
\sup_{\mu \in S}\big|\tilde{J}(\pi^{\mathrm{E}},r^\mu)-J(\pi^{\mathrm{E}},r^\mu)\big|\le H\sqrt{\big(Hd^2\log(d)+2Hd^2\log(3/\epsilon)+ 2d\log(2/\xi)\big)/N_1}+4H\epsilon,
\label{eq:union-tra-3}
\#
with probability at least $1-\xi$.
By taking $\epsilon = \sqrt{6d/N_1}$ in \eqref{eq:union-tra-3}, which satisfies the conditions in Lemma \ref{lem:covering-number}, it holds  with probability at least $1-\xi$ that 
$$
\begin{aligned}
\sup_{\mu \in S}\big|\tilde{J}(\pi^{\mathrm{E}},r^\mu)-J(\pi^{\mathrm{E}},r^\mu)\big|&\le H\sqrt{\big(Hd^2\log(d)+Hd^2\log(\frac{2N_1}{d})+ 2d\log(2/\xi)\big)/N_1}+4H\sqrt{\frac{d}{N_1}},\\
&\le 4\sqrt{H^3d^2/N_1}\log(6N_1/\xi).
\end{aligned}
$$
We conclude the proof of Lemma \ref{lem:MC estimate}.
\end{proof}

\subsection{Proof of Lemma \ref{lem4.8}}\label{A.7}
\begin{proof} 
First, we fix $\mu\in S$. We define for $(k, h) \in [K]\times [H]$ that
\begin{equation}\label{EA.57}
\begin{aligned}
&X^k_h=(\mu^k_h-\mu_h)^{\top}(-\psi(s_h^k,a_h^k)+\nabla_{\mu_{h} }J(\pi^k,\mu^k)),\\
&Y^k=\sum_{i=1}^{k}\sum_{h=1}^HX^i_h,\\
&S^k_h=\sigma
\big((s^1_h,a^1_h) , (s^2_h,a^2_h),\dots,(s^h_k,a^h_k)\big)
\big),\\
&E_h = \sigma\big((s_{h,1}^\rE,a_{h,1}^\rE),(s_{h,2}^\rE,a_{h,2}^\rE),\cdots,(s_{h,N_1}^\rE,a_{h,N_1}^\rE)\big),\\
&G_h^k =\sigma(S^k_h,E_h),\\ 
&G^k=\sigma(G^k_1,G^k_2,\dots,G^k_H),
\end{aligned}
\end{equation}
where $\sigma(\cdot)$ denotes the generated $\sigma$-algebra. It holds that $\{G^k\}_{k\in [K]}$ is a filtration with respect to the time index $k$, since 
$G^{k_1}\subseteq G^{k_2}$ for $k_1 \leq k_2$. 

We first show that $X^k_h\in G^k$ holds for any
$(k, h) \in [K] \times [H]$. 
By the definition of $X^k_h$ in \eqref{EA.57}, it only suffices to prove that $\mu_{h}^k\in G^k$ for any $(k, h) \in [K] \times [H]$. Here we show this by induction with index $k$. 
Since $\mu_h^1=\boldsymbol{0}$ for any
$h \in [H]$, the base case where $k = 1$ is trival. We assume that $\mu_h^k\in G^k$ where $k \geq 1$ is a given integer, then we consider the case $k + 1$. Recall that  the update process of reward parameter in OGAP (Lines  \ref{line:p5}--\ref{line:p6} of Algorithm \ref{alg:gail-online}) takes the following form, 
\$
\mu_h^{k+1}&=\text{Proj}_{B} \{\mu_h^k+\eta\widehat{\nabla}_{\mu_{h}}L(\pi^k,\mu^k)\}\\
&=\text{Proj}_{B} \Big\{\mu_h^k+\eta\cdot\Big[\frac{1}{N_1}\sum_{\tau =1}^{N_1}\psi(s^{\mathrm{E}}_{h,\tau}
,a^{\mathrm{E}}_{h,\tau}) -\psi(s^k_h,a^k_h)\Big]\Big\}.
\$
First, according to the induction hypothesis, we have $\mu_h^k\in G^{k}\subseteq G^{k+2}$, which implies that $(s^k_h,a^k_h)\in G^{k+1}$. Then it holds that
\$
\mu_h^k + \frac{1}{N_1}\sum_{\tau =1}^{N_1}\psi(s^{\mathrm{E}}_{h,\tau}
,a^{\mathrm{E}}_{h,\tau}) -\psi(s^k_h,a^k_h)\in G^{k+1},
\$
for any $h \in [H]$. As 
$\text{Proj}_{B}$ is a continous operator, we obtain that $\mu_h^{k+1}\in G^{k+1}$. Thus we complete the induction. 

Now we construct a martingale to upper bound $Y^K$. Note that conditioning on the filtration $G^{k-1}$, the term $\mu_h^k-\mu_h$ is a constant. 
Recall that in \eqref{EA.601} we show that 
\$
\nabla_{\mu_{h} }J (\pi^k,\mu^k)=
\int_{\cS\times\cA} \psi(s,a)\cdot \rho_h^{\pi^k}(s,a)\mathrm{d}s \mathrm{d}a,
\$
which implies that
\begin{equation}\label{EA.61}
\mathbb{E}_k(X^k_h|G^{k-1})=(\mu_h^k-\mu_h)^{\top}\mathbb{E}_k(\psi(s^k_{h},a^k_{h})-\nabla_{\mu_{h} }J(\pi^k,\mu^k)|G^{k-1})=0
\end{equation}
for any $(k, h) \in [K] \times [H]$. Here the expectation $\mathbb{E}_k$ is taken with respect to $a_i^k\sim \pi^k_i(\cdot\given s_i^k)$ and $s_{i+1}^k\sim\cP_i(\cdot\given s_i^k,a_i^k)$, corresponding to the expectation taken with respect to the state-action visitation measure $\rho_h^{\pi^k}$ defined in \eqref{eq:visitation-measure}. 
 Also, we have $Y^k\in G^k$, since $X^k_h\in G^k$ for any $(k, h) \times [K] \times [H]$. Moreover, we obtain for any $k \in [K]$ that
\$
\mathbb{E}_k(Y^k|G^{k-1})&=\mathbb{E}_k \big(\sum_{i=1}^k\sum_{h=1}^H X^i_h\biggiven G^{k-1}\big)\\
&=\mathbb{E}_k \big(\sum_{h=1}^H X^k_h+\sum_{i=1}^{k-1}\sum_{h=1}^H X^i_h \biggiven G^{k-1}\big)\\
&=\mathbb{E}_k \big(\sum_{h=1}^k X^k_h\biggiven G^{k-1}\big)+Y^{k-1}=Y^{k-1},
\$
where  the last equality follows from \eqref{EA.61}. Thus $\{Y^k\}_{k=1}^K$  is a martingale. Furthermore, by the definition of $X^k_h$ in \eqref{EA.57}, it holds that
\$
|Y^k-Y^{k-1}|=\big|\sum_{h=1}^H X^k_h\big|
&\leq \sum_{h=1}^H \|\mu_h^k-\mu_h\|_2\|\psi(s_h^k,a_h^k)-\nabla_{\mu_{h} }J(\pi^k,\mu^k)\|_2\le 8\sqrt{d}H,
\$
where the last inequality follows from the facts that $\|\psi(\cdot,\cdot)\|_2\leq 1$ and $\|\mu_{h}^k\|_2\le \sqrt{d}$. Therefore, by Azuma-Hoeffding inequality, we obtain that
\$
\mathbb{P}(|Y^{K}|\geq t)\leq \exp \bigg(\frac{-t^2}{2\sum_{k=1}^K(8\sqrt{d}H)^2 }\bigg)=\exp
\Big(\frac{-t^2}{128KdH^2} \Big)
\$
for any $t > 0$. Setting $t =\sqrt{128H^2dK\log(2/\xi)}$ with $\xi\in (0,1)$ and by the definition of $Y^k$ in \eqref{EA.57}, it holds with probability at least $1 -\xi$ that
\begin{equation}\label{EA.65}
\Big|\sum_{k=1}^K\sum_{h=1}^H (\mu^k_h-\mu_h)^{\top}(-\psi(s_h^k,a_h^k)+\nabla_{\mu_{h} }J(\pi^k,\mu^k))\Big|\leq 8\sqrt{2H^2dK \log (2/\xi)}.
\end{equation} 

Now we union bound \eqref{EA.65}. This is similar to the proof of Lemma \ref{lem:MC estimate} in Appendix \ref{pf:lem:MC estimate}.
By applying Lemma \ref{lem:covering-number} in the same normed space $(\mathbb{R}^{Hd},\|\cdot\|_\star)$, it holds with probability at least $1-\xi$ that
\#
\begin{aligned}
\sup_{\mu\in\mathcal{V}_{\epsilon}}|M(\mu)|&=\Big|\sum_{k=1}^K\sum_{h=1}^H (\mu^k_h-\mu_h)^{\top}(-\psi(s_h^k,a_h^k)+\nabla_{\mu_{h} }J(\pi^k,\mu^k))\Big|\\
&\le 8\sqrt{\big(H^3d^2\log(d) + 2Hd^2\log(3/\epsilon)+2d\log(2/\xi)\big)K},
\end{aligned}
\label{eq:union-1}
\#
where $\mathcal{V}_\epsilon$ is the $\epsilon$-covering for $S$ in Definition \ref{def:covering} and $\|\cdot\|_\star$ is defined in \eqref{sup-norm}. Here for notational convenience, we denote by $M(\mu) = \sum_{k=1}^K\sum_{h=1}^H (\mu^k_h-\mu_h)^{\top}(-\psi(s_h^k,a_h^k)+\nabla_{\mu_{h} }J(\pi^k,\mu^k))$.
 For any $\mu^\prime,\mu^{\prime\prime}\in S$ satisfying $\|\mu^\prime-\mu^{\prime\prime}\|_\star \le \epsilon$, it holds that
\#
|M(\mu^{\prime})-M(\mu^{\prime\prime})|\le 4HK\epsilon.
\label{eq:union-2}
\#
Combining \eqref{eq:union-1} and \eqref{eq:union-2} and applying triangle inequality, we have that
\#
\sup_{\mu \in S}|M(\mu)|\le  8\sqrt{\big(H^3d^2\log(d) + 2Hd^2\log(3/\epsilon)+2d\log(2/\xi)\big)K} +4HK\epsilon,
\label{eq:union-3}
\#
with probability at least $1-\xi$.
By taking $\epsilon = \sqrt{d/K}$ in \eqref{eq:union-3}, which satisfies that $\epsilon B_{\rm unit}\subset S$, we derive that 
$$
\begin{aligned}
\sup_{\mu \in S}|M(\mu)|&\le 8\sqrt{\big(H^3d^2\log(d) + 2Hd^2\log(9K/d)+2d\log(2/\xi)\big)K} +4H\sqrt{dK},\\
&\le 32\sqrt{H^3d^2K}\log(9K/\xi),
\end{aligned}
$$
with probability at least $1-\xi$. Hence we conclude the proof of Lemma \ref{lem4.8}.
\end{proof}

\subsection{Proof of Lemma \ref{lem:optim}}
\label{pf:lem:con}
\begin{proof}
Under Assumption \ref{def:linear} and by the definition of $\Lambda_{h}^k$ in \eqref{eq:online-result-L}, we have 
\begin{equation}
\begin{aligned}
(\cP_{h} \widehat{V}_{h+1}^{k})(x, a) &=\varphi_{h}^{k}(s, a)^{\top}(\Lambda_{h}^{k})^{-1}\big(\sum_{\tau=1}^{k-1} \varphi_{h}^{\tau}(s_{h}^{\tau}, a_{h}^{\tau}) \varphi_{h}^{\tau}(s_{h}^{\tau}, a_{h}^{\tau})^{\top} \theta_{h}+\lambda \cdot \theta_{h}\big) \\
&=\varphi_{h}^{k}(s, a)^{\top}(\Lambda_{h}^{k})^{-1}\big(\sum_{\tau=1}^{k-1} \varphi_{h}^{\tau}(s_{h}^{\tau}, a_{h}^{\tau}) \cdot(\cP_{h} \widehat{V}_{h+1}^{\tau})(s_{h}^{\tau}, a_{h}^{\tau})+\lambda \cdot \theta_{h}\big).
\end{aligned}\label{eq:lemm:optim:1}
\end{equation}
Note that $ \widehat{\cP}_{h}\widehat{V}_h^k(s,a) = \varphi(s,a)^\top \widehat{\theta}_h^k$ by the closed form of $\widehat{\theta}_{h}^k$ in \eqref{eq:lemm:pess:1}, we obtain that
\begin{equation}
\label{eq:lemm:on:2-2terms}
\begin{aligned}
\varphi_{h}^{k}(s, a)^{\top} \widehat{\theta}_h^k-&(\mathcal{P}_{h} \widehat{V}_{h+1}^{k})(s, a) \\
=&\underbrace{\varphi_{h}^{k}(s, a)^{\top}(\Lambda_{h}^{k})^{-1}\Big(\sum_{\tau=1}^{k-1} \varphi_{h}^{\tau}(s_{h}^{\tau}, a_{h}^{\tau}) \cdot\big(\widehat{V}_{h+1}^{\tau}(s_{h+1}^{\tau})-({\cP}_{h} \widehat{V}_{h+1}^{\tau})(s_{h}^{\tau}, a_{h}^{\tau}\big)\Big)}_{(\mathrm{i})}\\
& \qquad -\underbrace{\lambda \cdot \varphi_{h}^{k}(s, a)^{\top}(\Lambda_{h}^{k})^{-1} \theta_{h}}_{(\mathrm{ii})},
\end{aligned} 
\end{equation}
for any $(s,a)\in\cS\times \cA$.
To upper bound the norm of term (i) in \eqref{eq:lemm:on:2-2terms}, we introduce the following lemma.
\begin{lemma} \label{lem:online-concen}
Let $\lambda =1$ in the construction of estimated kernels \eqref{eq:online-result-L}.  It holds with probability at least $1-\delta / 4$ that
\[
\Big\|\sum_{\tau=1}^{k-1} \varphi_{h}^{\tau}\left(s_{h}^{\tau}, a_{h}^{\tau}\right) \cdot\left(\widehat{V}_{h+1}^{\tau}(s_{h+1}^{\tau})-\left(\mathcal{P}_{h} \widehat{V}_{h+1}^{\tau}\right)\left(s_{h}^{\tau}, a_{h}^{\tau}\right)\right)\Big\|_{\left(\Lambda_{h}^{k}\right)^{-1}} \leq C \sqrt{ H^{2}d^2 \cdot \log (HdK/ \delta)}
\]
for any $(k, h) \in[K] \times[H]$, where $C>0$ is an absolute constant and $\delta \in(0,1)$. 
\end{lemma}
\begin{proof}
See Appendix \ref{pf:on-con} for a detailed proof. 
\end{proof}

By applying Cauchy-Schwarz inequality and Lemma \ref{lem:online-concen} on term (i) in \eqref{eq:lemm:on:2-2terms}, we have
   \begin{equation}
\begin{aligned}
|(\mathrm{i})| & \leq   \|\varphi   (s, a)   \|_{(\Lambda_{h}^k)^{-1}} \cdot \Big\|\sum_{\tau=1}^{k-1} \varphi_{h}^{\tau}\left(s_{h}^{\tau}, a_{h}^{\tau}\right) \cdot\left(\widehat{V}_{h+1}^{\tau}(s_{h+1}^{\tau})-\left(\mathcal{P}_{h} \widehat{V}_{h+1}^{\tau}\right)\left(s_{h}^{\tau}, a_{h}^{\tau}\right)\right)\Big\|_{\left(\Lambda_{h}^{k}\right)^{-1}} \\
& \leq H\sqrt{d}\cdot C\sqrt{d \log ( HdK/\xi)}\cdot  \|\varphi   (s, a)   \|_{(\Lambda_{h}^k)^{-1}}
\end{aligned}\label{eq:pf:lem:on-(i)}
   \end{equation}
with probability at least $1-\xi /4 .$ 

For term (ii) in \eqref{eq:lemm:on:2-2terms}, by setting $\lambda=1$, we obtain that
\begin{equation}
\begin{aligned}
|(\mathrm{ii})| & \leq   \|\varphi   (s, a  )  \|_{{(\Lambda_h^k)}^{-1}}\cdot   \|\theta_{h}  \|_{(\Lambda_{h}^k)^{-1}} \\
& \leq \sqrt{d} \cdot   \|\varphi   (s, a)  \|_{{(\Lambda_h^k)}^{-1}},
\end{aligned}
\label{eq:pf:lem:on:(ii)}
\end{equation}
where the last inequality follows from the fact that $   \|(\Lambda_h^k)^{-1}  \|_{2} \leq 1 $ and $\|\theta_h\|_2\le\sqrt{d}$ for any $h\in[H]$. 
Combining \eqref{eq:lemm:optim:1}, \eqref{eq:lemm:on:2-2terms}, \eqref{eq:pf:lem:on-(i)}, and \eqref{eq:pf:lem:on:(ii)}, it holds with probability at least $1-\xi /4$ that 
$$
\begin{aligned}
\big|\mathcal{P}_{h} \widehat{V}_{h+1}^{ k} (s, a)-\widehat{\mathcal{P}}_{h}^k  \widehat{V}_{h+1}^{ k}(s, a)\big|
& \leq C \sqrt{d \log (HdK/\xi)} \cdot   \|\varphi   (s, a)  \|_{(\Lambda_{h}^k)^{-1}} \\
& \leq H\sqrt{d}\kappa \cdot   \|\varphi   (s, a)  \|_{(\Lambda_{h}^k)^{-1}}\le \Gamma_h^k(s,a)
\end{aligned}
$$
for any $h \in[H]$ and $  (s, a ) \in \mathcal{S} \times \mathcal{A}$. Here $\kappa=C \sqrt{d \log (HdK/\xi)}$ is the scaling parameter in \eqref{eq:bonus} with an absolute constant $C>0$. 
Then we conclude the proof of Lemma \ref{lem:optim}.
\end{proof}

\subsection{Proof of Lemma \ref{lem:online-concen}}\label{pf:on-con}

\begin{proof}
The proof of Lemma \ref{lem:online-concen} is adapted from that of Lemma D.1 in \cite{OPPO}.
 \begin{lemma}[Concentration of Self-Normalized Process]\label{lem:self-norm}
Let $\big\{\widetilde{\mathcal{F}}_{t}\big\}_{t=0}^{\infty}$ be a filtration and $\left\{\eta_{t}\right\}_{t=1}^{\infty}$ be an $\mathbb{R}$ -valued stochastic process such that $\eta_{t}$ is $\widetilde{\mathcal{F}}_{t}$ -measurable for any $t \geq 0$. Moreover, we assume that, for any $t \geq 0$, conditioning on $\widetilde{\mathcal{F}}_{t}, \eta_{t}$
is a zero-mean and $\sigma$-sub-Gaussian random variable with the variance proxy $\sigma^{2}>0$, that is,
$$
\mathbb{E}\big[\exp({\lambda \eta_{t}}) \given \widetilde{\mathcal{F}}_{t}\big] \leq e^{\lambda^{2} \sigma^{2} / 2}
$$
for any $\lambda \in \mathbb{R}$. Let $\left\{X_{t}\right\}_{t=1}^{\infty}$ be an $\mathbb{R}^{d}$ -valued stochastic process such that $X_{t}$ is $\widetilde{\mathcal{F}}_{t}$ -measurable
for any $t \geq 0$. Also, let $Y \in \mathbb{R}^{d \times d}$ be a deterministic and positive-definite matrix. For any
$t \geq 0$, we define
$$
\bar{Y}_{t}=Y+\sum_{s=1}^{t} X_{s} X_{s}^{\top}, \quad S_{t}=\sum_{s=1}^{t} \eta_{s} \cdot X_{s}.
$$
For any $\delta>0$, it holds with probability at least $1-\delta$ that
$$
\left\|S_{t}\right\|_{\bar{Y}_{t}^{-1}}^{2} \leq 2 \sigma^{2} \cdot \log \left(\frac{\operatorname{det}\left(\bar{Y}_{t}\right)^{1 / 2} \operatorname{det}(Y)^{-1 / 2}}{\delta}\right)
$$
for any $t \geq 0$.
 \end{lemma}
\begin{proof}
See Theorem 1 of \cite{elliptical-potential} for a detailed proof.
\end{proof}
By the definition of filtration $\left\{\mathcal{F}_{k, h, m}\right\}_{(k, h, m) \in[K] \times[H] \times[2]}$ in \eqref{EA.2} and Markov property, we have 
\begin{equation}
\mathbb{E}\big[\widehat{V}_{h+1}^{\tau}(s_{h+1}^{\tau}) \biggiven\mathcal{F}_{\tau, h, 1}\big]=\big(\mathcal{P}_{h} \widehat{V}_{h+1}^{\tau}\big)(s_{h}^{\tau}, a_{h}^{\tau}). \label{D.2}
\end{equation}
Conditioning on $\mathcal{F}_{\tau, h, 1}$, the only randomness comes from $s_{h+1}^{\tau}$, while $\widehat{V}_{h+1}^{\tau}$ is a deterministic
function determined by $\widehat{Q}_{h+1}^{\tau}$ and $\pi_{h+1}^{\tau}$, which are further
determined by the historical data in $\mathcal{F}_{\tau, h, 1}$. For simiplicity of notations, we define
$
\eta_{\tau, h}=\widehat{V}_{h+1}^{\tau}\left(s_{h+1}^{\tau}\right)-\big(\mathcal{P}_{h} \widehat{V}_{h+1}^{\tau}\big)\left(s_{h}^{\tau}, a_{h}^{\tau}\right).
$
By \eqref{D.2}, conditioning on $\mathcal{F}_{\tau, h, 1}, \eta_{\tau, h}$ is a zero-mean random variable. Moreover, as $\widehat{V}_{h+1}^{\tau} \in$
$[0, H\sqrt{d}]$, conditioning on $\mathcal{F}_{\tau, h, 1}$,   $\eta_{\tau, h}$ is an $(H\sqrt{d}/ 2)$-sub-Gaussian random variable defined
in Lemma \ref{lem:self-norm}. Meanwhile, $\eta_{\tau, h}$ is $\mathcal{F}_{k, h, 2}$ -measurable, since $\mathcal{F}_{\tau, h, 1} \subseteq \mathcal{F}_{k, h, 2}$ for any $\tau \in[k-1]$.
Hence, for any fixed $h \in[H]$, by Lemma \ref{lem:self-norm}, it holds with probability at least $1-\delta /(4 H)$
that
\#
\begin{aligned}
\Big\|\sum_{\tau=1}^{k-1} \varphi_{h}^{\tau}\left(s_{h}^{\tau}, a_{h}^{\tau}\right) &\cdot\left(\widehat{V}_{h+1}^{\tau}(s_{h+1}^{\tau})-\big(\mathcal{P}_{h} \widehat{V}_{h+1}^{\tau}\big)(s_{h}^{\tau}, a_{h}^{\tau})\right)\Big\|^2_{\left(\Lambda_{h}^{k}\right)^{-1}} 
\\ &\le \frac{H^2 d}{2} \big(\frac{1}{2}\log (\operatorname{det}(\Lambda_{h}^{k}) ) - \frac{1}{2}\log(\operatorname{det}( I)) + \log(4H/\delta) \big).\label{eq2.20.1}
\end{aligned}
\#
Recall that in \eqref{bound-det} we derive that
\begin{equation}
\det(\Lambda_h^{K+1} )\le \Big (\det \big  ((Hd^{3/2}R\cdot\text{Vol}(\cS) + 1)\cdot I \big)\Big)^d. \label{eq:104}
\end{equation}
By plugging \eqref{eq:104} into \eqref{eq2.20.1} and a union bound argument, we obtain with probability at least $1-\delta/2$ that
\[
\begin{aligned}
\Big\|\sum_{\tau=1}^{k-1} \varphi_{h}^{\tau}\left(s_{h}^{\tau}, a_{h}^{\tau}\right) &\cdot\left(\widehat{V}_{h+1}^{\tau}(s_{h+1}^{\tau})-\big(\mathcal{P}_{h} \widehat{V}_{h+1}^{\tau}\big)(s_{h}^{\tau}, a_{h}^{\tau})\right)\Big\|^2_{\left(\Lambda_{h}^{k}\right)^{-1}} \\
&\le \frac{H^2d}{2} \big(d\cdot\log((Hd^{3/2}R\cdot\text{Vol}(\cS) + 1) + \log(4H/\delta)\big),
\end{aligned}
\]
which implies that
\[
\Big\|\sum_{\tau=1}^{k-1} \varphi_{h}^{\tau}\left(s_{h}^{\tau}, a_{h}^{\tau}\right) \cdot\left(\widehat{V}_{h+1}^{\tau}(s_{h+1}^{\tau})-\big(\mathcal{P}_{h} \widehat{V}_{h+1}^{\tau}\big)(s_{h}^{\tau}, a_{h}^{\tau})\right)\Big\|^2_{\left(\Lambda_{h}^{k}\right)^{-1}} 
\leq C^{\prime \prime} \sqrt{ H^{2}d^2 \cdot \log (HdK/ \delta)},
\]
for any $(k, h) \in[K] \times[H]$ with probability at least $1-\delta/4$. Here $C^{\prime \prime}>0$ is an absolute constant. By this, we conclude the proof of Lemma \ref{lem:online-concen}.
\end{proof}

\section{Proofs of Supporting Lemmas: PGAP}

\subsection{Proof of Lemma \ref{lem:xi-uncertainty}} \label{pf:lem:xi-uncertainty}
\begin{proof}
We show that $\{\Gamma_{h}\}_{h=1}^H$ constructed in Lemma \ref{lem:xi-uncertainty} are the $\xi$-uncertainty qualifiers for the initally estimated transition kernels $\{\tilde{\cP}_h\}_{h=1}^H$ constructed in \eqref{eq:result-regression of thetaand Lambda}. By the definition of $\Lambda_{h}$ in \eqref{eq:result-regression  of thetaand Lambda}, we have 
\begin{equation}
\begin{aligned}
\mathcal{P}_{h}(s^{\prime} \mid s, a  ) &=\phi(s, a, s^{\prime}  )^{\top} \theta_{h} \\
&=\phi(s, a, s^{\prime}  )^{\top} \Lambda_{h}^{-1}(\sum_{\tau=1}^{N_2} \int_{\mathcal{S}} \phi(s_{h}^{\tau}, a_{h}^{\tau}, s^{\prime}  ) \mathcal{P}_{h}(s^{\prime} \mid s_{h}^{\tau}, a_{h}^{\tau}  ) \mathrm{d} s^{\prime}+\lambda \cdot \theta_{h}  ).
\end{aligned}\label{eq:lemm:pess:1}
\end{equation}
By \eqref{eq:lemm:pess:1}, we have
\begin{equation}\label{eq:lemm:pess:2-2terms}
\begin{aligned}
&\mathcal{P}_{h}(s^{\prime} \mid s, a  )-\tilde{\mathcal{P}}_{h}(s^{\prime} \mid s, a  ) \\
&\qquad  =\mathcal{P}_{h}(s^{\prime} \mid s, a  )-\phi(s, a, s^{\prime}  )^{\top} \tilde{\theta}_{h}\\
&\qquad =\underbrace{\phi(s, a, s^{\prime}  )^{\top} \Lambda_{h}^{-1}  \Big (\sum_{\tau=1}^{N_2}\big(\int_{\mathcal{S}} \phi(s_{h}^{\tau}, a_{h}^{\tau}, s^{\prime}  ) \mathcal{P}_{h}(s^{\prime} \mid s_{h}^{\tau}, a_{h}^{\tau}  ) \mathrm{d} s^{\prime}-\phi(s_{h}^{\tau}, a_{h}^{\tau}, s_{h+1}^{\tau}  ) \big )  \Big)}_{(\mathrm{i})}\\
&\qquad \quad +\underbrace{\lambda \cdot \phi(s, a, s^{\prime}  )^{\top} \Lambda_{h}^{-1} \theta_{h}}_{\text {(ii) }} . 
\end{aligned} 
\end{equation}

We introduce the following lemma to upper bound term (i) on the RHS of \eqref{eq:lemm:pess:2-2terms}. 

\begin{lemma}\label{lem:appendix-1}
Let $\lambda=1$ in the construction of $\tilde{\mathcal{P}}_{h}$ and $\Gamma_{h}$ in \eqref{eq:result-regression  of thetaand Lambda} and \eqref{eq:xi-uncertainty}. By Assumption \ref{def:linear} , the event that
$$
\Big\|\sum_{\tau=1}^{N_2}\big(\int_{\mathcal{S}} \phi(s_{h}^{\tau}, a_{h}^{\tau}, s^{\prime}  ) \mathcal{P}_{h}(s^{\prime} \mid s_{h}^{\tau}, a_{h}^{\tau}  ) \mathrm{d} s^{\prime}-\phi(s_{h}^{\tau}, a_{h}^{\tau}, s_{h+1}^{\tau}  ) \big )  \Big\|_{\Lambda_{h}^{-1}} \leq \sqrt{c_{1} R^{2} \cdot(d \log (Hd N_2/\delta))}
$$
holds for all $(s, a,h) \in \mathcal{S} \times \mathcal{A}\times[H]$ with probability at least $1-\delta .$ Here $c_{1}$ is an absolute constant.
\end{lemma}
\begin{proof}
See proof in Appendix \ref{pf:lem:appendix-1}. 
\end{proof}

For term (i) on the RHS of \eqref{eq:lemm:pess:2-2terms}, by Cauchy-Schwarz inequality, it holds with probability at least $1-\xi /2$ that 
\begin{equation}
\begin{aligned}
|(\mathrm{i})| & \leq\|\phi(s, a, s^{\prime}  )  \|_{\Lambda_{h}^{-1}} \cdot\Big\|\sum_{\tau=1}^{N_2}\big(\int_{\mathcal{S}} \phi(s_{h}^{\tau}, a_{h}^{\tau}, s^{\prime}  ) \mathcal{P}_{h}(s^{\prime} \mid s_{h}^{\tau}, a_{h}^{\tau}  ) \mathrm{d} s^{\prime}-\phi(s_{h}^{\tau}, a_{h}^{\tau}, s_{h+1}^{\tau}  ) \big ) \Big \|_{\Lambda_{h}^{-1}} \\
& \leq c_{1} R \cdot \sqrt{d \log (d HN_2/\xi)} \cdot\|\phi(s, a, s^{\prime}  )\|_{\Lambda_{h}^{-1}},
\end{aligned}\label{eq:pf:lem:uncertainty-(i)}
\end{equation}
where the last inequality follows from Lemma \ref{lem:appendix-1}.

For term (ii) in \eqref{eq:lemm:pess:2-2terms}, setting $\lambda = 1$, we have
\begin{equation}
\begin{aligned}
|(\mathrm{ii})| & \leq\|\phi(s, a, s^{\prime}  )  \|_{\Lambda_{h}^{-1}} \cdot\|\theta_{h}  \|_{\Lambda_{h}^{-1}} \leq \sqrt{d} \cdot\|\phi(s, a, s^{\prime}  )  \|_{\Lambda_{h}^{-1}},
\end{aligned}
\label{eq:pf:lem:uncertanity:(ii)}
\end{equation}
where the last inequality follows from the facts that $\|\Lambda_h^{-1}  \|_{2} \leq 1 $ and $\|\theta_h\|_2\le\sqrt{d}$ for all $h\in[H]$. Plugging \eqref{eq:pf:lem:uncertainty-(i)} and \eqref{eq:pf:lem:uncertanity:(ii)} into \eqref{eq:lemm:pess:2-2terms}, it holds  with probability at least $1-\xi /2$ that
$$
\begin{aligned}
|\mathcal{P}_{h}(s^{\prime} \mid s, a  )-\tilde{\mathcal{P}}_{h}(s^{\prime} \mid s, a  )  | & \leq c R \sqrt{d \log (Hd N_2/\xi)} \cdot\|\phi(s, a, s^{\prime}  )  \|_{\Lambda_{h}^{-1}} \\
& \leq \kappa \cdot\|\phi(s, a, s^{\prime}  )  \|_{\Lambda_{h}^{-1}}
\end{aligned}
$$
for any $h \in[H]$ and $ (s, a, s^{\prime}  ) \in \mathcal{S} \times \mathcal{A}\times \cS$.  Here $\kappa=c R \sqrt{d \log (d N_2)}$ is the scaling parameter with an absolute constant $c>0$. We conclude the proof of Lemma \ref{lem:xi-uncertainty}.
\end{proof}

\subsection{Proof of Lemma \ref{lem:pess}} \label{pf:lem:pess}
\begin{proof}
We establish the lower and upper bounds for $\iota_h^{k, r^\mu}$ as follows, respectively. 

\vskip5pt
\noindent\textbf{Lower Bound.} 
First we prove by backward induction that $\widehat{V}_{h}^{k,r^\mu}\in[0,(H-h+1)\sqrt{d}]$ for any $h\in[H]$. The base case $h=H$ holds, since $\widehat{V}_{H+1}^{k,r^\mu}= 0$ and $r_h^\mu\in[0,\sqrt{d}]$. 
We assume that $\widehat{V}_{h+1}^{k,r^\mu}\in[0,H-h]$. For the case for $h$, recall that $\{\widehat{\cP}_h(\ \cdot\mid s^\prime,a^\prime)\}_{h=1}^H$ is a set of probability measures on the state space $\cS$ for $(s^\prime,a^\prime)\in \cS\times\cA$, which implies that $\widehat{\cP}_h\widehat{V}_{h+1}^{k,r^\mu}\in[0,(H-h)\sqrt{d}\ ]$.
Note that $\Gamma_h\ge 0$, hence we have that 
$ 
\widehat{Q}_{h}^{k,r^\mu}(s, a) \in [0,(H-h+1)\sqrt{d}].
$
Then it holds that $\widehat{V}_{h}^{k,r^\mu}\in[0 , H-h]$, since $\widehat{V}_{h}^{k,r^\mu}(s) = \la\widehat{Q}_{h}^{k,r^\mu}(s, \cdot),\pi_h^k(\cdot\mid s)\ra_{\cA}$ for any $s\in\cS$. By induction, it holds that $\widehat{V}_{h}^{k,r^\mu}\in[0,(H-h+1)\sqrt{d}]$ for any $h\in[H]$.

For notational simplicity, we write $\bar{Q}_{h}^{k,r^\mu}(s, a)={r}_{h}^\mu(s, a)+\widehat{\mathcal{P}}_{h} \widehat{V}_{h+1}^{k,r^\mu}(s, a)-\Gamma_{h}(s, a) .$ 
From the policy evaluation stage in Algorithm \ref{alg:gail-offline}, we have
\#\label{eq:pf:lem:pess-1}
\widehat{Q}_{h}^{k,r^\mu}(s, a)=\max\big\{\bar{Q}_{h}^{ k,r^\mu}(s, a), 0\big\}.
\#
Meanwhile, by the definition of the $\xi$-uncertainty qualifiers in Definition \ref{def:xi-uncertainty}, we have ${r}_{h}^\mu+\mathcal{P}_{h} \widehat{V}_{h+1}^{ k} \geq \bar{Q}_{h}^{k,r^\mu}$.  
Moreover, by the fact that ${r}_{h}^\mu\in[0,\sqrt{d}]$ and $\widehat{V}_{h+1}^{ k,r^\mu} \in[0,(H-h)\sqrt{d}]$,  we have ${r}_{h}^\mu+\mathcal{P}_{h} V_{h+1,\pi^k}^{ k,r^\mu} \in[0,(H-h+1)\sqrt{d}]$.  Thus, we derive that
$$
\begin{aligned}
\widehat{Q}_{h}^{k}(s, a)&=\max\{\bar{Q}_{h}^{ k,r^\mu}(s, a),0  \}\\ &\leq \max\{{r}_{h}^\mu (s,a)+\mathcal{P}_{h} \widehat{V}_{h+1}^{ k,r^\mu}(s,a), 0\}\\ &={r}_{h}^\mu(s,a)+\mathcal{P}_{h} \widehat{V}_{h+1}^{ k,r^\mu}(s,a),
\end{aligned}
$$
which implies that $\iota_{h}^{ k,r^\mu} \geq 0$.

\vskip5pt
\noindent\textbf{Upper Bound.} 
Since we condition on the event $\mathcal{E}$ defined in Definition \ref{def:xi-uncertainty}, we have
$$
\begin{aligned}
\bar{Q}_{h}^{k,r^\mu}(s, a) &={r}_{h}^\mu(s, a)+\widehat{\mathcal{P}}_{h} \widehat{V}_{h+1}^{k,r^\mu}(s, a)-\Gamma_{h}(s, a)) \\
& \leq {r}_{h}^\mu(s, a)+\mathcal{P}_{h} \widehat{V}_{h+1}^{ k,r^\mu}(s, a) \leq H-h+1,
\end{aligned}
$$
where the last inequality follows from the facts that $\widehat{V}_{h+1}^{ k,r^\mu}(s, a) \leq (H-h)\sqrt{d}$ and ${r}_{h}^\mu(s, a) \leq \sqrt{d}.$ 
By \eqref{eq:pf:lem:pess-1} we have that $\widehat{Q}_{h}^{k,r^\mu}(s, a) \geq \bar{Q}_{h}^{ k,r^\mu}(s, a) .$ Thus, we obtain that
$$
\begin{aligned}
\iota_{h}^{k,r^\mu}(s, a) &=({r}_{h}^\mu+\mathcal{P}_{h} \widehat{V}_{h}^{k}  )(s, a)-\widehat{Q}_{h}^{k,r^\mu}(s, a) \\
& \leq {r}_{h}^\mu(s, a)+(\mathcal{P}_{h}-\widehat{\mathcal{P}}_{h}  ) \widehat{V}_{h}^{k,r^\mu}(s, a)+\Gamma_{h}(s, a) \\
& \leq  2 \Gamma_{h}(s, a),
\end{aligned}
$$
where the last inequality follows from the definition of $\mathcal{E}$. Then we complete the proof of Lemma \ref{lem:pess}.
\end{proof}

\subsection{Proof of Lemma \ref{lem:concave} }
\label{pf:lem:concave} 
\begin{proof}
Recall that $\widehat{L}(\pi,\mu)= \Tilde{J}(\pi^\rE,\mu)-\widehat{J}(\pi^k,\mu)$. By Assumption \ref{def:linear}, we know that the function $\Tilde{J}(\pi^\rE,\mu) = \frac{1}{N_1}\sum_{\tau=1}^{N_1}\sum_{h=1}^H\psi(s_{h,\tau}^\rE,a_{h,\tau}^{\rE}) ^\top {\mu_h}$ is a linear combination of $\{\mu_h\}_{h=1}^H$ and concave. Therefore, to prove that $\widehat{L}(\pi,\mu)$ is concave, it suffices to prove that $\widehat{J}(\pi^k,r^\mu)$ is convex for any $\mu_h$ with $\mu =\{\mu_h\}_{h=1}^H\in S$.

Recall that $\widehat{J}(\pi^k,r^\mu)= \widehat{V}_1^{k,r^\mu}(x)$, where $x$ is the fixed inital state and $\widehat{V}_1^{k,r^\mu}$ defined in \eqref{eq:def:iota-on} is solved by
\begin{equation}
\begin{aligned}
\widehat{V}_{H+1}^{k,r^\mu}(\cdot) &= 0\\
\widehat{Q}_{h}^{k,r^\mu}(\cdot, \cdot)
&=\max\big\{({r}^\mu_{h}+\widehat{\mathcal{P}}_{h} \widehat{V}_{h+1}^{k,r^\mu}-\Gamma_{h})(\cdot, \cdot), 0\big\}\\
\widehat{V}_{h}^{k,r^\mu}(\cdot, \cdot)
&=\big\langle\widehat{Q}_{h}^{k,r^\mu}(\cdot, \cdot) ,\pi_{h}^{k}(\cdot \mid \cdot)\big\rangle_{\mathcal{A}},\text{ for } h\in[H].
\label{eq:grad-}
\end{aligned}
\end{equation}

Our proof relies on the following three basic properties of convex functions: 
\begin{itemize}
\item[(i)]  If $f(u)$ and $g(u)$ are both convex function for $u$, then $f(u)+g(u)$
is also convex. 
\item[(ii)] If $f(u)$ and $g(u)$ are both convex function for $u$, then $\max\big(f(u),g(u)\big)=\big(|f(u)+g(u)|+|f(u)-g(u)|\big)/2$
is also convex. 
\item[(iii)] If $f(u,s)$ is a convex function for $u$, then $\mathbb{E}_{s\sim p} {f(u,s) } $ is also convex function for $u$, where $p$ is a distribution.
\end{itemize}

Now we are ready to prove that $\widehat{J}(\pi^k,\mu)$ is convex for $\{\mu_h\}_{h=1}^H$.
For the base case where $h=1$, observing that $\widehat{J}(\pi^k,\mu)= \big\la \widehat{Q}_{1}^{k,r^\mu}(s_1, \cdot) ,\pi_{1}^{k}(\cdot \mid s_1) \big\ra_{\cA}$, by property (ii) and (iii) and \eqref{eq:grad-}, it suffices to prove that $r_h^\mu + \widehat{\mathcal{P}}_{1} \widehat{V}_{2}^{k,r^\mu}-\Gamma_{1}$ is convex for $\mu_1$. 
Note that $\{\mu_h\}_{h=1}^H$ are seperate reward parameters and $r_h^\mu(\cdot,\cdot)$ is only determined by $\mu_h$, it shows that $\widehat{\mathcal{P}}_{1} \widehat{V}_{2}^{k,r^\mu}-\Gamma_1$ is a constant regardless of $\mu_h$, which implies that $\widehat{\mathcal{P}}_{1} \widehat{V}_{2}^{k,r^\mu}-\Gamma_{1}$ is convex for $\mu_1$. Meanwhile, since $\cR$ is linear to $\psi$ as shown in \eqref{eq:B-mu}, we know that $r_h^\mu=\psi ^\top \mu_h$ is also convex for $\mu_h$.
By property (i), we know that $\widehat{J}(\pi^k,\mu)=\widehat{V}_1^{k,r^\mu}(s_1)$ is convex for $\mu_1$.

For the case when $h=H^\prime$, where $2\le H^\prime \le H$, similar to the analysis in the case when $h=1$, we can prove that $\widehat{V}_{H^\prime}^{k,r^\mu}(s_h)$ is convex for $\mu_{H^\prime}$. Note that $\{\Gamma_{h}\}_{h=1}^H$ defined in \eqref{def:xi-uncertainty} is independent of $\{\mu_h\}_{h=1}^H$, we know that  $r_{H^\prime-1}^\mu+\widehat{\mathcal{P}}_{H^\prime -1} \widehat{V}_{H^\prime}^{k,r^\mu}-\Gamma_{H^\prime -1}$ is convex for $\mu_{H^\prime}$. By property (ii) and (iii) and \eqref{eq:grad-}, we know that $\widehat{V}_{H^\prime-1}^{k,r^\mu}$ is also convex for $\mu_{H^\prime}$.  By repeating the anlysis,  we know that $\widehat{J}(\pi^k,\mu)$ is convex for $\mu_{H^\prime}$. 
Therefore, we conclude the proof of Lemma \ref{lem:concave}.
\end{proof}

\subsection{Proof of Lemma \ref{lem:telescope-reward-update]}}\label{pf:telescope-reward-update]}

\begin{proof} 
Since Lemma \ref{lem:concave} shows that $\widehat{L}(\pi^k,\mu)$ is concave for $\mu_h$ for any $h\in[H]$,  by the property of concave function, we have 
\begin{equation}
\label{eq:tele-2}
\widehat{L}(\pi^k,\mu)-\widehat{L}(\pi^k,\mu^k)\le \sum_{h=1}^H \big[\nabla_{\mu_h}\widehat{L}(\pi^k,\mu^k)^\top(\mu_h-\mu_h^k)\big].
\end{equation}
Recall that we apply projected gradient ascent method to update $\{\mu_h^k\}_{h=1}^H$ in PGAP (Line \ref{line:pp-proj} of Algorithm \ref{alg:gail-offline}) as
\begin{equation}
\mu_{h}^{k+1}  = \operatorname{Proj}_{S}\big[\mu_{h}^{k}+\eta {\nabla}_{\mu_{h}} \widehat{L}(\pi^{k}, \mu^{k})\big],\label{eq:pf:reward-update}
\end{equation}
we obtain that
\begin{equation}
\big[\mu_{h}^{k+1}-\mu_{h}^{k}-\eta \nabla_{\mu_{h}}\widehat{L} (\pi^{k}, \mu^{k})\big]^{\top}(\mu_{h}-\mu_{h}^{k+1}) \geq 0.
\label{eq:proj-grad}
\end{equation}
Rearranging terms in \eqref{eq:proj-grad}, we have 
\begin{equation}
\begin{aligned}
\label{eq:tele-3}
\nabla_{\mu_h}\widehat{L}(\pi^k,\mu^k)^\top(\mu_h-\mu_h^{k+1}) &\le-\frac{1}{2\eta}\big((\mu_{h}^{k+1}-\mu_{h}^{k})^{\top}(\mu_{h}-\mu_{h}^{k+1})\big) \\
&=\frac{1}{2\eta}\big(\|\mu_{h}^{k}-\mu_{h}\|_{2}^{2}-\|\mu_{h}^{k+1}-\mu_{h}\|_{2}^{2}-\|\mu_{h}^{k+1}-\mu_{h}^{k}\|_{2}^{2}\big).
\end{aligned}
\end{equation}
By adding the term $ \nabla_{\mu_{h}} \widehat{L}(\pi^{k+1}, \mu^{k})^{\top}(\mu_{h}^{k+1}-\mu_{h}^{k})$ on both sides of \eqref{eq:tele-3}, we obtain that 
\begin{equation}
\begin{aligned}
\label{eq:tele-33}
\nabla_{\mu_h}\widehat{L}(\pi^k,\mu^k)^\top(\mu_h-\mu_h^{k}) 
&=\frac{1}{2\eta}\big(\|\mu_{h}^{k}-\mu_{h}\|_{2}^{2}-\|\mu_{h}^{k+1}-\mu_{h}\|_{2}^{2}-\|\mu_{h}^{k+1}-\mu_{h}^{k}\|_{2}^{2}\big)\\
&\qquad +  \nabla_{\mu_h}\widehat{L}(\pi^k,\mu^k)^\top(\mu_{h}^{k+1}-\mu_{h}^{k}).
\end{aligned}
\end{equation}
Note that $\eta$ is positive and by applying Cauchy-Schwarz inequality on  the second term of the right-hand side of \eqref{eq:tele-33}, we derive that
\#\label{eq:tele-44}
\nabla_{\mu_h}\widehat{L}(\pi^k,\mu^k)^\top(\mu_{h}^{k+1}-\mu_{h}^{k})\le \|{\nabla_{\mu_h}}\widehat{L}(\pi^k,\mu^k)\|_2\|\mu_h^{k+1}-\mu_h^k\|_2.
\#
From the reward update process in \eqref{eq:pf:reward-update}, we observe that 
\#
\|\mu_h^{k+1}-\mu_h^{k}\|_2\le\|\nabla_{\mu_h}\widehat{L}(\pi^k,\mu^k)\|_2. 
\label{eq:tele 55}
\#
By plugging \eqref{eq:tele-33}, \eqref{eq:tele-44}, and \eqref{eq:tele 55} into \eqref{eq:tele-2}, we have 
$$
\begin{aligned}
\sum_{k=1}^K  \big[\widehat{L}(\pi^k,\mu)-\widehat{L}(\pi^k,\mu^k)\big] &\le\sum_{k=1}^K \sum_{h=1}^H \big[\frac{1}{2\eta}||\mu_h^{k+1}-\mu_h||_2^2+\frac{1}{2\eta}||\mu_h^{k+1}-\mu_h||_2^2 \\
&\quad - \frac{1}{2\eta}||\mu_h^{k+1}-\mu^{k}_h||_2^2+\eta \|\nabla_{\mu_h}\widehat{L}(\pi^k,\mu^k)\|_2^2\big],
\end{aligned}
$$
which concludes the proof of Lemma \ref{lem:telescope-reward-update]}.
\end{proof}

\subsection{Proof of Lemma \ref{lem:appendix-1}}
\label{pf:lem:appendix-1}
\begin{proof}
Before we prove Lemma \ref{lem:appendix-1}, we introduce the following lemma  to generalize the concentration of self-normalized vector-valued process in \cite{elliptical-potential} to function-valued process.

\begin{lemma}\label{lem:appendix:Concentration of Self-Normalized Function-Valued Process}
Let $\Omega$ be a probability space and $\{\eta_{t}  \}_{t=1}^{\infty}$ be a function-valued stochastic process with a filtration $\{\mathcal{M}_{t}  \}_{t=0}^{\infty}$, i.e. $\eta_{t}: \mathcal{S} \times \Omega\rightarrow \mathbb{R}.$ We assume that $\eta_{t} \mid \mathcal{G}_{t-1}$ is zero-mean and $\sigma$-sub-Gaussian, that is,
$$
\begin{aligned}
\mathbb{E}[\eta_{t}(s) \mid \mathcal{G}_{t-1}  ] =0, \qquad 
\log \Big(\mathbb{E}\big[\exp\big(\int_{\mathcal{S}} g(s) \eta_{t}(s) \mathrm{d} s  \big) \biggiven \mathcal{G}_{t-1} \big]\Big) \leq \frac{1}{2}\|g\|_{\infty}^{2} \cdot \sigma^{2}, 
\end{aligned}
$$
for any $s\in\cS$ and function $g:\cS\rightarrow\mathbb{R}$.
Let $\{X_{t}  \}_{t=0}^{\infty}$ be an vector-function-valued stochastic process where $X_{t}: \mathcal{S} \times \Omega\rightarrow \mathbb{R}^{d}, X_{t} \in \mathcal{M}_{t-1} .$ We also assume that $\|\lambda^{\top} X_{t}  \|_{\infty, \mathcal{S}} \leq R \cdot\|\lambda^{\top} X_{t}  \|_{2, \mathcal{S}}$ a.s. for all $\lambda \in \mathbb{R}^{d}$. Let $V \in \mathbb{R}^{d \times d}$ be a positive definite matrix and $\bar{V}_{t}=$
$\sum_{\tau=1}^{t} \int X_{\tau}(s) X_{\tau}(s)^{\top} \mathrm{d} s .$ We also define
$$
S_{t}=\sum_{\tau=1}^{t} \int_{\mathcal{S}} X_{\tau}(s) \eta_{\tau}(s) \mathrm{d} s.
$$
Then for any $\delta>0$ and $t>0$, it holds with probability at least $1-\delta$ that
$$
  \|S_{t}  \|_{\bar{V}_{t}^{-1}}^{2} \leq 2(\sigma R)^{2} \cdot \log\bigg(\frac{\operatorname{det}(\bar{V}_{t}  )^{1 / 2}}{\delta \operatorname{det}(V)^{1 / 2}}  \bigg).
$$
\end{lemma}
\begin{proof}
See Appendix \ref{pf:lem:appendix:Concentration} for a detailed proof.
\end{proof}

We consider the filtration $\{\mathcal{F}_{h, \tau}  \}_{h\in[H],\tau\in[N_2]}$ defined in \S \ref{sec:ORIL}.
For any function $f: \mathcal{S}\rightarrow \mathbb{R},$ by Holder inequality, it holds that 
\begin{equation}
\Big|  \int_\cS f(s^{\prime}) \big(\mathcal{P}_{h}(s^{\prime} \mid s_{h}^{\tau}, a_{h}^{\tau}  )-\delta_{s_{h+1}^{\tau}} (s^{\prime})\big) \mathrm{d}s^\prime \Big|\leq 2\|f\|_{\infty}. \label{con-11}
\end{equation}
By the property of Dirac function, we have
\begin{equation}
\mathbb{E}[\delta_{s_{h+1}^{\tau}}(s^{\prime}  ) \mid \mathcal{F}_{h, \tau}  ]=\mathcal{P}_{h}(s \mid s_{h}^{\tau} , a_{h}^{\tau}  ). \label{con-12}
\end{equation}
Combining \eqref{con-11} and \eqref{con-12}, we verify the condition of Lemma \ref{lem:appendix:Concentration of Self-Normalized Function-Valued Process} as follows,
$$
\log \bigg(\mathbb{E} \Big  [\exp\Big(  \int_\cS f(s^\prime)\big( \mathcal{P}_{h}(s^\prime\mid s_{h}^{\tau}, a_{h}^{\tau}  )-\delta_{s_{h+1}^{\tau}} (s^\prime) \big) \mathrm{d}s^\prime \Big) \Big |\,\mathcal{F}_{h, \tau}  \Big]\bigg) \leq 2\|f\|_{\infty}^{2} .
$$
Note that $\phi(s_{h}^{\tau}, a_{h}^{\tau}, s^{\prime}  )$ is $\mathcal{F}_{h, \tau}$-measurable and $\delta_{s_{h}^{\tau}+1}$ is $\mathcal{F}_{h+1, \tau}$-measurable, we apply Lemma \ref{lem:appendix:Concentration of Self-Normalized Function-Valued Process} with $X_{\tau}=\phi(s_{h}^{\tau}, a_{h}^{\tau}, \cdot  )$ and $\eta_{\tau}=\mathcal{P}_h(\cdot \mid s_{h}^{\tau}, a_{h}^{\tau}  )-\delta_{s_{h+1}^{\tau}}$, which implies that
\begin{equation}
\begin{aligned}
  & \Big\|\sum_{\tau=1}^{N_2}\Big[\int_{\mathcal{S}} \phi(s_{h}^{\tau}, a_{h}^{\tau}, s^{\prime}  ) \mathcal{P}_{h}(s^{\prime} \mid s_{h}^{\tau}, a_{h}^{\tau}  ) \mathrm{d} s^{\prime}-\phi(s_{h}^{\tau}, a_{h}^{\tau}, s_{h+1}^{\tau}  )  \Big] \Big\|_{\Lambda_{h}^{-1}}^{2} \\ 
  & \qquad \leq 8 R^{2} \cdot \log\big(H / \delta \cdot \operatorname{det}(\Lambda_{h}  )^{1 / 2} \operatorname{det}(\lambda I)^{-1 / 2}  \big),
\end{aligned}\label{eq:pf:lem:appendcx-1-1}
\end{equation}
with probability at least $1-\delta / H$. 

We now upper bound the term $\operatorname{det}(\Lambda_{h}  ) $. 
By the definition of $\Lambda_{h}$ in \eqref{eq:result-regression  of thetaand Lambda}, it holds for any $y \in \mathbb{R}^{d}$ that
$$
y^{\top} \Lambda_{h} y=\lambda\|y\|_{2}^{2}+\sum_{\tau=1}^{N_2}  \int_\cS |y^{\top} \phi(s_{h}^{\tau}, a_{h}^{\tau}, s^\prime )|^2\mathrm{d} s^\prime\leq (\lambda+d N_2 )\|y\|_2^2,
$$
where the last inequality follows from Assumption \ref{def:linear}. Hence we derive that $\|\Lambda_{h}  \|_{2} \leq \lambda+d N_2,$ which implies that
\begin{equation}
\operatorname{det}(\Lambda_{h}  ) \leq\|\Lambda_{h}  \|_{2}^{d} \leq(\lambda+d N_2)^{d}.
\label{eq:pf:lem:appendix-1-2}
\end{equation}
Setting $\lambda=1$, combining \eqref{eq:pf:lem:appendcx-1-1} and \eqref{eq:pf:lem:appendix-1-2}, it holds with probability at least $1-p / H$ that
\begin{equation}
\begin{aligned}
&\Big\|\sum_{\tau=1}^{N_2}   \big(\int_{\mathcal{S}} \phi(s_{h}^{\tau}, a_{h}^{\tau}, s^{\prime}  ) \mathcal{P}_{h}(s^{\prime} \mid s_{h}^{\tau}, a_{h}^{\tau}  ) \mathrm{d} s^{\prime}-\phi(s_{h}^{\tau}, a_{h}^{\tau}, s_{h+1}^{\tau}  )  \big)  \Big\|_{\Lambda_{h}^{-1}}^{2} \\
&\qquad\leq 8 R^{2} \cdot\big(1 / 2 \cdot d \log (1+d N_2)+\log (H / p)\big) \leq c_{1} R^{2} \cdot\big(d \log (Hd N_2/p)\big).
\end{aligned}
\label{eq:pf:lem:appendix-1-3}
\end{equation}
Here $c_{1} > 0$ is an absolute constant. By the union bound for $h \in[H],$ we know that \eqref{eq:pf:lem:appendix-1-3} holds for all $h \in[H]$ with probability at least $1-p$. Thus, we complete the proof of Lemma \ref{lem:appendix-1}.
\end{proof}

\subsubsection{Proof of Lemma \ref{lem:appendix:Concentration of Self-Normalized Function-Valued Process}}
\label{pf:lem:appendix:Concentration}
\begin{proof}
The proof is a generalization of that in \cite{elliptical-potential}. For notational simplicity, we denote by $\langle f,  g\rangle =\int_{\mathcal{S}} f(s) g(s) \mathrm{d} s$ the inner product of any functions $f$ and $g$. We use the same definitions and notations as Lemma \ref{lem:appendix-1}. First, we introduce the following lemmas.
\begin{lemma}\label{lem:appendix-concentration-p1}
       Let $\beta \in \mathbb{R}^{d}$ be a vector and
   $$
   M_{t}^{\beta}=\exp    \bigg\{\sum_{\tau=1}^{t}   \bigg(\frac{   \langle\beta^{\top} X_{\tau}, \eta_{\tau}  \rangle}{\sigma R}-\frac{   \langle\beta^{\top} X_{\tau}, \beta^{\top} X_{\tau}  \rangle}{2}  \bigg) \bigg\}.
   $$
   Let $T$ be a stopping time with respect to the filtration $   \{\mathcal{M}_{t}  \}_{t=1}^{\infty}$. Then $M_{T}^{\beta}$ is almost surely well-defined and $\mathbb{E}   [M_{T}^{\beta}  ] \leq 1$.
   \end{lemma}
   \begin{proof}
       We first show that $   \{M_{t}^{\beta}  \}_{t=0}^{\infty}$ is a supermartingale. Let
   $$
   G_{\tau}^{\beta}=\exp    \bigg(\frac{   \langle\beta^{\top} X_{\tau}, \eta_{\tau}  \rangle}{\sigma R}-\frac{   \|\beta^{\top} X_{\tau}  \|_{2, \mathcal{S}}^{2}}{2} \bigg). 
   $$
   By the conditional sub-Gaussian property of $\eta_{\tau}$ and the fact that $   \|\beta^{\top} X_{t}  \|_{\infty, \mathcal{S}} \leq R \cdot   \|\beta^{\top} X_{t}  \|_{2, \mathcal{S}}$,  we have
   $$
   \mathbb{E}   [G_{\tau}^{\beta} \mid \mathcal{M}_{t-1}  ] \leq \exp   \bigg (\frac{   \|\beta^{\top} X_{\tau}  \|_{\infty, \mathcal{S}}^{2}}{2 R}-\frac{   \|\beta^{\top} X_{\tau}  \|_{\infty, \mathcal{S}}^{2}}{2 R} \bigg)=1.
   $$
   Thus, we have $\mathbb{E}   [M_{t}^{\beta} \mid \mathcal{M}_{t-1}  ]=M_{t-1}^{\beta} \cdot \mathbb{E}   [G_{\tau}^{\beta} \mid \mathcal{M}_{t-1}  ] \leq M_{t-1}^{\beta},$ which implies that $   \{M_{t}^{\beta}  \}_{t=0}^{\infty}$ is a supermartingale and
   $\mathbb{E}   [M_{t}^{\beta}  ] \leq 1$.
   We then show that $M_{T}^{\beta}$ is well-defined, where $T$ is a stopping time. By the convergence theorem of nonnegative supermartingales, it holds that $M_{\infty}^{\beta}=\lim _{t   \rightarrow \infty} M_{t}^{\beta}$. Thus, $M_{T}^{\beta}$ is well-defined whether $T<\infty$ or not.
   Finally, to show that $\mathbb{E}   [M_{T}^{\beta}  ] \leq 1$, we apply Fatou's lemma and obtain that
   $$
   \mathbb{E}   [M_{\tau}^{\beta}  ]=\mathbb{E}   [\lim _{t   \rightarrow \infty} M_{T \wedge t}^{\beta}  ] \leq \liminf _{t   \rightarrow \infty} \mathbb{E}   [M_{T \wedge t}^{\beta}  ] \leq 1.
   $$
   Thus, we conclude the proof of Lemma \ref{lem:appendix-concentration-p1}.
   \end{proof}
   
   \begin{lemma}
       \label{lem:appendix-concentration-p2}
       Let $T$ be a stopping time with respect to $   \{\mathcal{M}_{t}  \}_{t=0}^{\infty},$ then it holds with probability at least $1-\delta$ that
   $$
      \|S_{T}  \|_{\bar{V}_{T}^{-1}}^{2}>2(\sigma R)^{2} \cdot \log    \bigg(\frac{\operatorname{det}   (\bar{V}_{T}  )^{1 / 2}}{\delta \operatorname{det}(V)^{1 / 2}}  \bigg).
   $$
   \end{lemma}
   \begin{proof}
       Without loss of generality, we assume that $\sigma \cdot R=1$. We define
   $$
   V_{t}=\sum_{\tau=1}^{t} \int X_{\tau}(s) X_{\tau}(s)^{\top}.
   $$
   Then, we have
   $$
   M_{t}^{\beta}=\exp    (\beta^{\top} S_{t}-\|\beta\|_{V_{t}}^{2} / 2  ).
   $$
   
   By Lemma \ref{lem:appendix-concentration-p1}, we have that $\mathbb{E}   [M_{t}^{\beta}  ] \leq 1$. Let $\Lambda$ be an $\mathbb{R}^{d}$-valued Gaussian random variable with covariance matrix $V^{-1}$. Moreover, we assume that $\Lambda$ is independent of $   \{\mathcal{M}_{t}  \}_{t=0}^{\infty} .$ Let
   $
   M_{t}=\mathbb{E}   [M_{t}^{\Lambda} \mid \mathcal{M}_{\infty}  ],
   $
   where $\mathcal{M}_{\infty}=\sigma   (\cup_{\tau=0}^{\infty} \mathcal{M}_{\tau}  )$. Notice that $\mathbb{E}   [M_{T}  ]=\mathbb{E}   [\mathbb{E}   [M_{T}^{\Lambda} \mid \Lambda  ]  ] \leq 1$. We denote by $p$ the density of $\Lambda$ and by $v(A)=$
   $\int \exp    (-x^{\top} A x  ) \mathrm{d} x=\sqrt{(2 \pi)^{d} / \operatorname{det}(A)}$ for positive definite matrix $A \in \mathbb{R}^{d \times d}$. Then we obtain that 
   \begin{equation}
       \begin{aligned}
   M_{t} &=\int \exp    (\beta^{\top} S_{t}-\|\beta\|_{V_{t}}^{2} / 2  ) p(\beta) \mathrm{d} \beta \\
   &=\int \exp    (-   \|\beta-V_{t}^{-1} S_{t}  \|_{V_{t}}^{2} / 2+   \|S_{t}  \|_{V_{t}^{-1}}^{2}  ) p(\beta) \mathrm{d} \beta \\
   &=v(V)^{-1} \cdot \exp    (   \|S_{t}  \|_{V_{t}^{-1}}^{2} / 2  ) \cdot \int \exp    (-   \|\beta-V_{t}^{-1} S_{t}  \|_{V_{t}}^{2} / 2-\|\beta\|_{V}^{2} / 2  ) \mathrm{d} \beta .
   \end{aligned}
   \label{eq:pf:lem:concentra---1}
   \end{equation}
   Note that
   \begin{equation}
       \begin{aligned}
      \|\beta-V_{t}^{-1} S_{t}  \|_{V_{t}}^{2}+\|\beta\|_{V}^{2} / 2 &=   \|\beta-\bar{V}_{t}^{-1} S_{t}  \|_{V_{t}}^{2}+   \|V_{t}^{-1} S_{t}  \|_{V_{t}}^{2}-   \|S_{t}  \|_{V_{t}}^{2} \\
   &=   \|\beta-\bar{V}_{t}^{-1} S_{t}  \|_{V_{t}}^{2}+   \|S_{t}  \|_{V_{t}^{-1}}^{2}-   \|S_{t}  \|_{V_{t}}^{2}.
   \end{aligned}
   \label{eq:pf:lem:concentra---2}
   \end{equation}
   Plugging \eqref{eq:pf:lem:concentra---2} into \eqref{eq:pf:lem:concentra---1}, we have that
   $$
   \begin{aligned}
   M_{t} &=v(V)^{-1} \cdot \exp    (   \|S_{t}  \|_{\bar{V}_{t}^{-1}}^{2} / 2  ) \cdot \int \exp    (-   \|\beta-\bar{V}_{t}^{-1} S_{t}  \|_{\bar{V}_{t}}^{2} / 2  ) \mathrm{d} \beta \\
   &=\frac{v   (\bar{V}_{t}  )}{v   (V_{t}  )} \cdot \exp    (   \|S_{t}  \|_{\bar{V}_{t}^{-1}}^{2} / 2  ) \\
   &=\sqrt{\operatorname{det}(V) / \operatorname{det}   (\bar{V}_{t}  )} \cdot \exp    (   \|S_{t}  \|_{\bar{V}_{t}^{-1}}^{2} / 2  ).
   \end{aligned}
   $$
   Thus, we have 
   $$
   \begin{aligned}
     \mathbb{P}   \bigg\{   \|S_{T}  \|_{\bar{V}_{T}^{-1}}^{2}>2 \log   \bigg(\frac{\operatorname{det}   (\bar{V}_{T}  )^{1 / 2}}{\delta \operatorname{det}(V)^{1 / 2}}  \bigg)  \bigg\} &=  \mathbb{P}   (\delta  M_{T}>1  ) \leq \mathbb{E}   [\delta  M_{T}  ] \leq \delta,
   \end{aligned}
   $$
   which completes the proof of Lemma \ref{lem:appendix-concentration-p2}.
   \end{proof}
   
   We now prove Lemma \ref{lem:appendix:Concentration of Self-Normalized Function-Valued Process} as follows. Define
   $$
   T=\inf    \bigg\{t \geq 0 \colon   2 \log    \bigg(\frac{\operatorname{det}   (\bar{V}_{t}  )^{1 / 2}}{\delta \operatorname{det}(V)^{1 / 2}} \bigg) < \|S_{t}  \|_{\bar{V}_{t}^{-1}}^{2}  \bigg\}
   $$
   for a fixed $\delta>0$. Then it holds that
   $$
   \begin{aligned}
     \mathbb{P}   \bigg\{\exists t \geq 0,   \|S_{t}  \|_{\bar{V}_{t}^{-1}}^{2}>2 \log    \bigg(\frac{\operatorname{det}   (\bar{V}_{t}  )^{1 / 2}}{\delta \operatorname{det}(V)^{1 / 2}} \bigg)  \bigg\} &=  \mathbb{P}(T<\infty) \\
   &=  \mathbb{P}   \bigg\{\|S_{T}  \|_{\bar{V}_{T}^{-1}}^{2}>2 \log    \bigg(\frac{\operatorname{det}   (\bar{V}_{T}  )^{1 / 2}}{\delta \operatorname{det}(V)^{1 / 2}} \bigg),T<\infty \bigg\}\\
   & \leq   \mathbb{P}   \bigg\{\|S_{T}  \|_{\bar{V}_{T}^{-1}}^{2}>2 \log    \bigg(\frac{\operatorname{det}   (\bar{V}_{T}  )^{1 / 2}}{\delta \operatorname{det}(V)^{1 / 2}} \bigg)\bigg\} \leq \delta,
   \end{aligned}
   $$
   which completes the proof of Lemma \ref{lem:appendix:Concentration of Self-Normalized Function-Valued Process}.
\end{proof}

\end{document}